\documentclass[11pt]{elsarticle}

\usepackage[margin=1in,papersize={8.5in,11in}]{geometry}
\usepackage{times,xcolor,hyperref}
\usepackage{amssymb,amsfonts,amsmath,mathrsfs,mathtools}
\usepackage{graphicx,epsfig,subfigure}
\usepackage{bm}

\usepackage{xcolor}
\definecolor{r}{rgb}{0.0, 0.0, 0.0}

\def\A{\mathcal{A}} 
\def\L{\mathcal{L}} 
\def\M{\mathcal{M}} 
\def\N{\mathcal{N}} 
\def\G{\mathcal{G}} 
\def\R{\mathcal{R}} 
\def\P{\mathcal{P}}

\def\Q{\mathcal{Q}}
\def\I{\mathcal{I}}

\def\vs{\vspace{0.2cm}} 
\usepackage{amsthm}

\newtheorem{theorem}{Theorem}[]

\newtheorem{lemma}[theorem]{Lemma}
\newtheorem{proposition}[theorem]{Proposition}

\DeclareMathOperator*{\argmin}{\arg\!\min}

\journal{Arxiv}

\begin{document}
\begin{frontmatter}

\title{The Mori-Zwanzig formulation of deep learning}

\author[ucsc]{Daniele Venturi\corref{correspondingAuthor}}
\cortext[correspondingAuthor]{Corresponding author}
\ead{venturi@ucsc.edu}

\author[penn]{Xiantao Li} 
\ead{Xiantao.Li@psu.edu}
\address[ucsc]{Department of Applied Mathematics, UC Santa Cruz, Santa Cruz, CA 95064}
\address[penn]{Department of Mathematics, Pennsylvania State University, State College, PA 16801}

\begin{abstract}
We develop a new formulation of deep learning based on the Mori-Zwanzig (MZ) formalism of irreversible statistical mechanics. The new formulation is built upon the well-known duality between deep neural networks and discrete dynamical systems, and it allows us to directly propagate quantities of interest (conditional expectations and probability density functions) forward and backward through the network by means of exact linear operator equations. Such new equations can be used as a starting point to develop new effective parameterizations of deep neural networks, and provide a new framework to study deep-learning via operator theoretic methods. The proposed MZ formulation of deep learning naturally introduces a new concept, i.e., the memory of the neural network, which plays a fundamental role in low-dimensional modeling and parameterization. By using the theory of contraction mappings, we develop sufficient conditions for the memory of the neural network to decay with the number of layers. This allows us to rigorously transform deep networks into shallow ones, e.g., by reducing the number of neurons per layer (using projection operators), or by reducing the total number of layers (using the decay property of the memory operator).
\end{abstract}

\end{frontmatter}

\section{Introduction}
It has been recently shown that new insights on deep learning 
can be obtained by regarding the process of training a 
deep neural network as a discretization of an optimal control  
problem involving  nonlinear differential equations 
\cite{Weinan2019,Weinan2017,he2016deep}.
One attractive feature of this formulation is that it allows 
us to use tools from dynamical system theory such as 
the Pontryagin maximum principle or the Hamilton-Jacobi-Bellman 
equation to study deep learning from a rigorous mathematical 
perspective \cite{Li2018,he2016identity,lu2017beyond}. 
For instance, it has been recently shown that 
by idealizing deep residual networks as continuous-time
dynamical systems it is possible to derive sufficient conditions 
for universal approximation in $L^p$, 
which can also be understood as an approximation theory 
that leverages flow maps generated by dynamical systems \cite{Li2022}.

In the spirit of modeling a deep neural network as a 
flow of a discrete dynamical system, in this paper we 
develop a new formulation of deep learning based on the 
Mori-Zwanzig (MZ) formalism. The MZ formalism was 
originally developed in statistical mechanics  
\cite{Mori,zwanzig1961memory} to formally integrate 
under-resolved phase variables in 
nonlinear dynamical systems by means of a 
projection operator. One of the main features 
of such formulation is that it allows us to 
systematically derive exact evolution equations for 
quantities of interest, e.g., macroscopic observables, 
based on microscopic equations of motion \cite{ciccotti1981derivation,hijon2010mori,IzVo06,chorin2000optimal,
chen2014computation,VenturiBook,dominy2017duality,
zhu2018faber,zhu2019generalized}. 

In the context of deep learning, the MZ formalism 
can be used to reduce the total number of degrees 
of freedom of the neural network, e.g., by reducing 
the number of neurons per layer (using projection operators), 
or by transforming deep networks into shallows networks, 
e.g., by approximating the MZ memory operator.
%
Computing the solution of the MZ equation for 
deep learning is not an easy task. One of the 
main challenges is the approximation of the memory 
term and the fluctuation (noise) term, 
which encode the interaction between 
the so-called orthogonal dynamics and the dynamics 
of the quantity of interest. In the context of neural networks, 
the orthogonal dynamics is essentially a discrete high-dimensional 
flow governed by a difference equation that is hard to solve. 
Despite these difficulties, the MZ equation of deep 
learning is formally exact, and can be used as a 
starting point to build useful 
approximations and parameterizations that target the 
output function directly. Moreover, it provides a new 
framework to study deep-learning via operator theoretic 
approaches. For example, the analysis of the 
memory term in the MZ formulation may shed light on the 
behaviour of recent neural network architectures such as 
the long short-term memory (LSTM) network
\cite{sherstinsky2020fundamentals,harlim2021machine}.  

This paper is organized as follows. In section \ref{sec:modeling} 
we briefly review the formulation of deep learning as a control 
problem involving a discrete stochastic dynamical system. 
In section \ref{sec:compositiontransfer} we introduce the 
composition and transfer operators associated with the neural 
network. Such operators are the discrete analogues of the 
stochastic Koopman \cite{Mezic2019,zhu2020} and 
Frobenius-Perron operators in classical continuous-time 
nonlinear dynamics. In the neural network setting the 
composition and transfer operators are integral operators 
with kernel given by the conditional transition density 
between one layer and the next.
In section \ref{sec:trainingParadigms} we discuss {\color{r}different 
training paradigms for stochastic neural networks, i.e., the 
classical ``training over weights'' paradigm, 
and a novel ``training over noise'' paradigm. Training over noise 
can be seen as an instance of transfer learning in which 
we optimize for the PDF of the noise to re-purpose a previously 
trained neural network to another task, without changing 
the neural network weights and biases.}
In section \ref{sec:MZequations} we present
the MZ formulation of deep learning and derive the operator equations 
at the basis of our theory. In section \ref{sec:Projections_1} 
we introduce a particular class of projection operators, 
i.e., Mori's projections \cite{zhu2019generalized} and study 
their properties. {\color{r} In section \ref{sec:fadingMemory} 
we develop the analysis of the MZ equation}, and derive sufficient 
conditions under which the MZ memory term decays with the 
number of layers. This allows us to approximate the MZ 
memory term with just a few terms and re-parameterize 
the network accordingly. The main findings are 
summarized in section \ref{sec:summary}. We also include two appendices 
in which we establish theoretical results concerning the composition and 
transfer operators for neural networks with additive random 
perturbations, {\color{r}and prove the Markovian property  
of neural networks driven by discrete random processes 
characterized by statistically independent random 
vectors.}

\section{Modeling neural networks as discrete stochastic dynamical systems}
\label{sec:modeling}
We model a neural network with $L$ layers as a discrete 
stochastic dynamical system of the form 
{\color{r}
\begin{equation}
\bm X_{n+1} =\bm H_n(\bm X_n,\bm w_n,\bm \xi_n), \qquad {n=0,1, \cdots, L-1}.
\label{discrete_dyn}
\end{equation}
Here, the index $n$ labels a specific layer in the network,
$\bm H_n$ is the transition function of the $(n+1)$ layer,
$\bm X_0\in \mathbb{R}^d$ is the network input, 
$\bm X_{n}\in \mathbb{R}^{d_{n}}$ is 
the output of the $n$-th layer\footnote{\color{r}The 
dimension of the vectors $\bm X_n$ 
and $\bm X_{n+1}$ can vary from layer 
to layer, e.g., in encoding or decoding neural networks \cite{kingma2013auto}.},
$\{\bm \xi_0,\ldots,\bm\xi_{L-1}\}$ are random 
vectors, and $\bm w_n\in\mathbb{R}^{q_n}$ are 
parameters characterizing  the $(n+1)$ layer. We allow the 
input $\bm X_0$ to be random.
Furthermore, we assume that the random 
vectors $\{\bm \xi_{0},\ldots,\bm \xi_{L-1}\}$ are 
statistically independent, and that $\bm \xi_n$ is 
independent of past and current states, i.e., 
$\{\bm X_0,\ldots,\bm X_{n}\}$. In this 
assumption, the neural network model \eqref{discrete_dyn} 
defines a Markov process $\{\bm X_n\}$ 
(see \ref{app:NN_characterization}). Further assumptions about
the mapping $\bm H_n$ and its relation to the noise process 
will be stated in subsequent sections. 
\begin{figure}
\centerline{\includegraphics[scale=0.42]{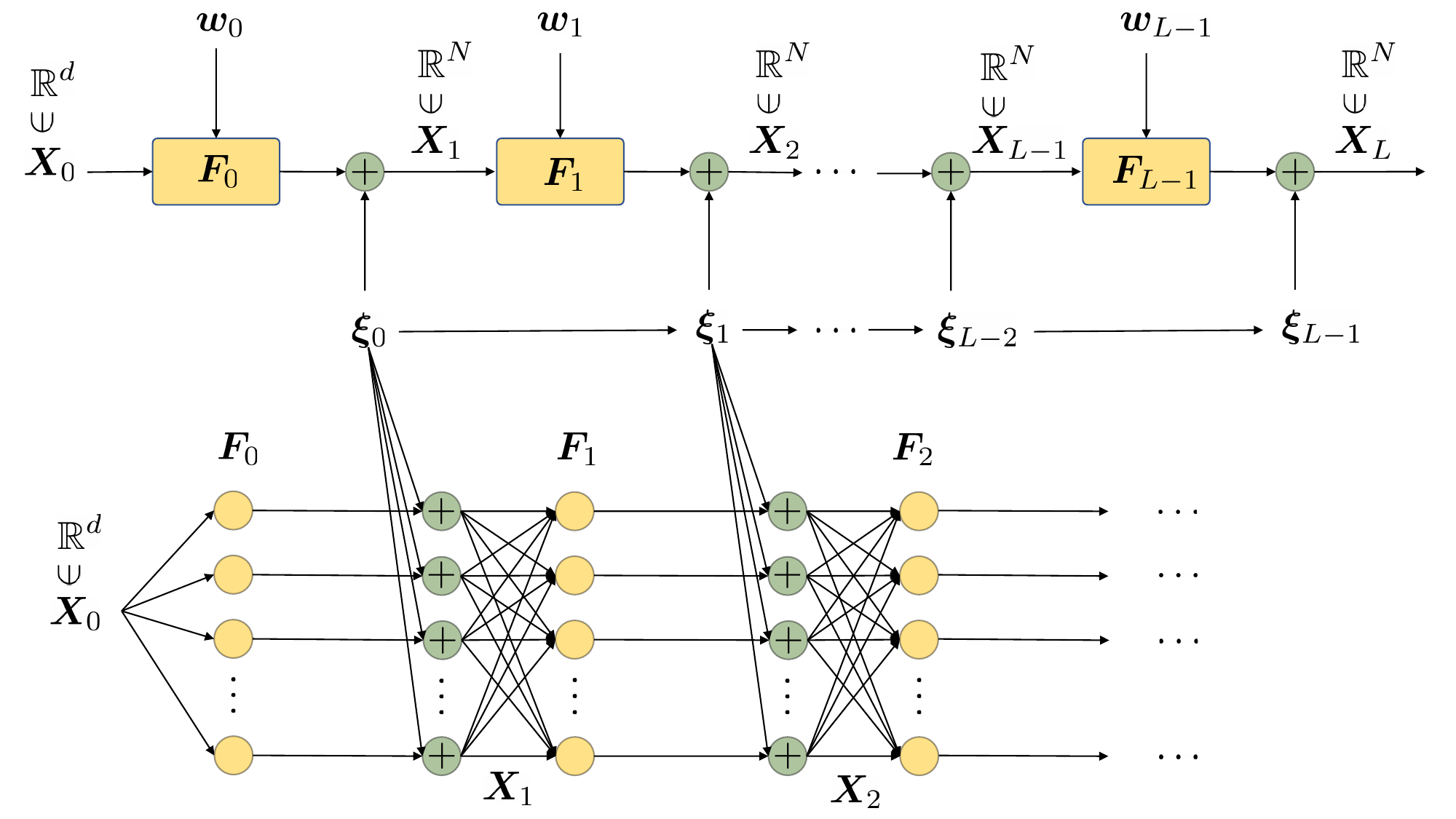}}
\caption{\color{r}Sketch of a stochastic neural network model of the 
form \eqref{discrete_sys_10}, with $L$ layers and $N$ neurons 
per layer. We assume that the random vectors 
$\{\bm \xi_{0},\ldots,\bm \xi_{L-1}\}$ are 
statistically independent, and that $\bm \xi_n$ is 
independent of past and current states, i.e., 
$\{\bm X_0,\ldots,\bm X_{n}\}$. With these  
assumptions,  $\{\bm X_0,\ldots,\bm X_L\}$ is a 
Markov process (see \ref{app:NN_characterization}).}
\label{fig:net2}
\end{figure}
The general formulation \eqref{discrete_dyn} includes 
the following important classes of neural networks: 
 
\begin{enumerate}
\item {\em Neural networks perturbed by additive 
random noise} (Figure \ref{fig:net2}). 
These models are of the form
\begin{equation}
\bm X_{n+1} = \bm F_n(\bm X_n,\bm w_n)+\bm \xi_n, \qquad 
n=0,\ldots, L-1.
\label{discrete_sys_10}
\end{equation} 
The mapping $\bm F_n$ is often defined as a composition of a 
layer-dependent affine transformation with 
an {activation function} $\varphi$, i.e., 
\begin{equation}
\bm F_n(\bm X_n,\bm w_n) = \varphi(\bm W_n \bm X_n+\bm b_n) \qquad 
\bm w_n = \{\bm W_n,\bm b_n\},
\label{activationFunction}
\end{equation} 
where $\bm W_n$ is a $d_{n+1}\times d_n$ weight 
matrix, and $\bm b_n\in\mathbb{R}^{d_{n+1}}$ is a bias 
vector.\\

\item {\em Neural networks perturbed by multiplicative random noise} (Figure \ref{fig:net1}). These models are of the form
\begin{equation}
\bm X_{n+1} =\bm F_n(\bm X_n,\bm w_n)+
\bm M_n(\bm X_n)\bm \xi_n, \qquad n=0,\ldots, L-1,
\label{discrete_sys}
\end{equation} 
where $\bm M_n(\bm X_n)$ is a matrix depending on $\bm X_n$.\\

\begin{figure}[t]
\centerline{\includegraphics[scale=0.42]{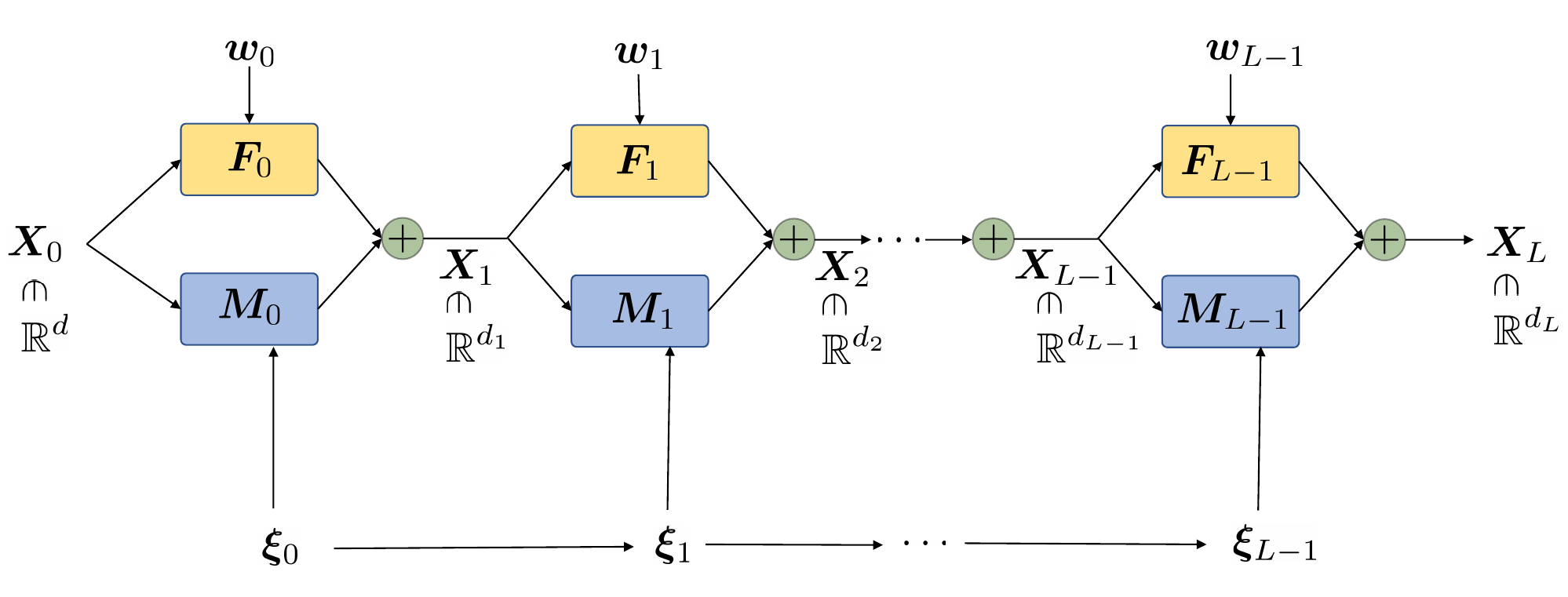}}
\caption{\color{r}Sketch of the stochastic neural network 
model \eqref{discrete_sys}. We assume that the random vectors 
$\{\bm \xi_{0},\ldots,\bm \xi_{L-1}\}$ are 
statistically independent, and that $\bm \xi_n$ is 
independent of past and current states, i.e., 
$\{\bm X_0,\ldots,\bm X_{n}\}$. In this 
assumption, the neural network model \eqref{discrete_sys} 
defines a Markov process $\{\bm X_n\}$. Note that the dimension 
of the vectors $\bm X_n$ can vary from layer to layer, 
e.g., in encoding or decoding neural networks.}
\label{fig:net1}
\end{figure}

\item {\em Neural networks with random weights and biases} \cite{gonon2020risk,yu2021simple}. These models are 
of the form 
\begin{equation}
\bm X_{n+1}= \varphi\left({\bm {Z}}_n \bm X_n+{\bm {z}}_n\right),
 \qquad n=0,\ldots, L-1
\label{NNrandomweights}
\end{equation}
where ${\bm {Z}}_n $ are random weight matrices, 
and ${\bm {z}}_n$ are random bias vectors.
The pairs $\{{\bm {Z}}_n, {\bm {z}}_n\}$ 
and $\{{\bm {Z}}_j, {\bm {z}}_j\}$ are assumed to be 
statistically independent for $n\neq j$. Moreover, 
$\{{\bm {Z}}_j, {\bm {z}}_j\}$ are independent of the 
neural network states $\{\bm X_0,\ldots,\bm X_{j}\}$ 
for all $j=0,\ldots, L-1$.
\end{enumerate}
}
\noindent
{\color{r}In this article, we will focus our attention 
primarily on neural network models with additive random 
noise, i.e., models of the form \eqref{discrete_sys_10}. 
The functional setting for these models is extensively 
discussed in \ref{sec:functionalsetting}.} 
{\color{r} The neural network output is usually written as 
\begin{equation}
q_L(\bm x)=\bm \alpha\cdot 
\mathbb{E}\left[\bm X_L|\bm X_0=\bm x\right],
\label{out}
\end{equation}
where $\bm \alpha$ is a vector of output weights, and 
$\mathbb{E}\left[\bm X_L|\bm X_0=\bm x\right]$ is the 
expectation of the random vector $\bm X_L$ 
conditional to $\bm X_0=\bm x$. In the absence 
of noise, \eqref{out} reduces to the well-known 
function composition rule 
\begin{equation}
q_L(\bm x) = \bm \alpha\cdot \bm F_{L-1}(\bm F_{L-2}(\cdots \bm F_1(\bm F_0(\bm x, \bm w_0),\bm w_1),\cdots,\bm w_{L-2}),\bm w_{L-1}).
\label{compR}
\end{equation}
The neural network parameters 
$\{\bm \alpha,\bm w_0,\ldots,\bm w_{L-1}\}$ 
appearing in \eqref{out} or \eqref{compR} are usually  
determined by minimizing a dissimilarity 
measure between $q_L(\bm x)$ and a given target 
function $f(\bm x)$ (supervised learning). 
By adding random noise to the neural network, e.g., 
in the form of additive noise or by randomizing 
weights and biases, we are essentially 
adding an infinite number of degrees of freedom 
to the system, which can be leveraged for training
and transfer learning (see section \ref{sec:trainingParadigms}).}

\section{Composition and transfer operators for neural networks}
\label{sec:compositiontransfer}

In this section we derive the composition and transfer 
operators associated with the neural network model \eqref{discrete_dyn}, 
which map, respectively, the conditional expectation
$\mathbb{E}\left\{\bm u(\bm X_L)|\bm X_n=\bm x\right\}$ 
(where $\bm u(\cdot)$ is a user-defined measurable function) 
and $p_n(\bm x)$ (the probability density of $\bm X_n$)  
forward and backward across the network. To this end, we 
assume that {\color{r} the random vectors 
$\{\bm \xi_0,\ldots,\bm \xi_{L-1}\}$ in \eqref{discrete_dyn}
are statistically independent, and that $\bm \xi_n$ is 
independent of past and current states, i.e., 
$\{\bm X_0,\ldots,\bm X_{n}\}$,
With these assumptions, $\{\bm X_n\}$ in \eqref{discrete_dyn} is a 
discrete Markov process (see \ref{app:NN_characterization}). 
Hence, the joint probability density function (PDF) 
of the random vectors $\{\bm X_0,\ldots,\bm X_L\}$, 
i.e., joint PDF of the state of the entire neural network, 
can be factored\footnote{\label{fot3}In equation \eqref{jointPDFtransition}
we used the shorthand notation $p_{i|j}(\bm x_i|\bm x_j)$ to denote the conditional probability density function 
of the random vector $\bm X_i$ given $\bm X_j=\bm x_j$. 
With this notation we have that the conditional 
probability density of $\bm X_i$ given $\bm X_i=\bm y$ is 
$p_{i|i}(\bm x|\bm y)=\delta(\bm x_i-\bm x_j)$, where $\delta(\cdot)$ 
is the Dirac delta function.} as}
\begin{equation}
p(\bm x_0,\ldots,\bm x_L)=p_{L|L-1}(\bm x_L|\bm x_{L-1})
p_{L-1|L-2}(\bm x_{L-1}|\bm x_{L-2})\cdots p_{1|0}(\bm x_1|\bm x_0) 
p_0(\bm x_0).
\label{jointPDFtransition}
\end{equation}
By using the identity {\color{r} (Bayes' theorem)}
\begin{equation}
p(\bm x_{k+1},\bm x_k)=p_{k+1|k}(\bm x_{k+1}|\bm x_k)
p_{k}(\bm x_k)= p_{k|k+1}(\bm x_{k}|\bm x_{k+1})
p_{k+1}(\bm x_{k+1}) 
\end{equation}
we see that the chain of transition 
probabilities \eqref{jointPDFtransition} can be 
reverted, yielding
\begin{equation}
p(\bm x_0,\ldots,\bm x_L)=p_{0|1}(\bm x_0|\bm x_{1})
p_{1|2}(\bm x_{1}|\bm x_{2})\cdots
p_{L-1|L}(\bm x_{L-1}|\bm x_L) 
p_L(\bm x_L).
\label{jointPDFtransitionInv}
\end{equation}
From these expressions, it follows that 
\begin{equation}
p_{n|q}(\bm x|\bm y)=\int p_{n|j}(\bm x|\bm z) 
p_{j|q}(\bm z|\bm y)d\bm z,
\label{trasition}
\end{equation}
for all indices $n$, $j$ and $q$ in $\{0,\ldots,L\}$, 
excluding $n=j=q$. 
The transition probability equation \eqref{trasition} 
is known as {\em discrete Chapman-Kolmogorov equation} and it
allows us to define the transfer operator 
mapping the PDF $p_n(\bm x_n)$ into 
$p_{n+1}(\bm x_{n+1})$, together 
with the composition operator for the 
conditional expectation 
$\mathbb{E}\{\bm u(\bm x_L)|\bm X_n=\bm x_{n}\}$. 
As we shall see hereafter, the discrete composition 
and transfer operators are adjoint to one another.

\subsection{Transfer operator}

Let us denote by $p_q(\bm x)$ the PDF of 
$\bm X_q$, i.e., the output of the $q$-th 
neural network layer. We first define the operator that maps 
$p_q(\bm x)$ into $p_n(\bm x)$. 
By integrating the joint probability density of $\bm X_n$ 
and $\bm X_q$, i.e., $p_{n|q}(\bm x|\bm y)p_q(\bm y)$
with respect to $\bm y$ we immediately obtain 
\begin{equation}
p_n(\bm x)=\int p_{n|q}(\bm x|\bm y)
p_q(\bm y)d\bm y.
\end{equation}
At this point, it is convenient to define the linear operator
\begin{equation}
\N(n,q) f(\bm x) = \int p_{n|q}(\bm x|\bm y)
f(\bm y)d\bm y.
\label{N}
\end{equation}
$\N(n,q)$ is known as {\em transfer} (or Frobenius-Perron) operator \cite{dominy2017duality}. From a 
mathematical viewpoint $\N(n,q)$ is a integral operator with 
kernel $p_{n|q}(\bm x,\bm y)$,  i.e., the transition 
density integrated ``from the right''. It follows from 
the Chapman-Kolmogorov identity \eqref{trasition} 
that the set of integral operators $\{\N(n,q)\}$ 
satisfies
\begin{equation}
\N(n,q) = \N(n,j)\N(j,q),\qquad \N(j,j)=\I,\qquad 
\forall n,j,q \in \{0,\ldots,L\},
\end{equation}
{\color{r} where $\I$ is the identity operator.}
The operator $\N$ allows us to map the one-layer PDF, 
e.g., the PDF of $\bm X_q$, either forward or backward 
across the neural network (see Figure \ref{fig:net3}). 
As an example, consider a network with four layers 
and states $\bm X_0$ (input), $\bm X_1$, $\bm X_2$, $\bm X_3$, 
and $\bm X_4$ (output). Then Eq. \eqref{N} implies that,
\begin{align}
p_2(\bm x)
=\underbrace{\N(2,1)\N(1,0)}_{\N(2,0)} p_0(\bm x)=
\underbrace{\N(2,3)\N(3,4)}_{\N(2,4)} p_4(\bm x).\nonumber
\end{align}
In summary, we have 
\begin{equation}
p_n(\bm x) = \N(n,q)p_q(\bm x)\qquad 
\forall n,q\in\{0,\ldots,L\}, 
\end{equation}
where 
\begin{equation}
\N(n,q)p_q(\bm x) = 
\int p_{n|q}(\bm x|\bm y)p_{q}(\bm y)d\bm y.
\label{Ndef}
\end{equation}
We emphasize that modeling the PDF dynamics via neural 
networks has been studied extensively in machine 
learning, e.g., in the theory of normalizing 
flows for density estimation or variational inference \cite{rezende2015variational,kobyzev2020normalizing,tabak2010density}.

\subsection{Composition operator}
\label{sec:compositionOperator}
For any  measurable deterministic function 
$\bm u(\bm x)$, the expectation of 
$\bm u(\bm X_j)$ conditional to $\bm X_n=\bm x$ is 
defined as   
\begin{equation}
\mathbb{E}\left\{\bm u(\bm X_j)|\bm X_n=\bm x\right\} = 
\int  \bm u (\bm y) p_{j|n}(\bm y|\bm x)d\bm y.
\label{conditionalMoment1}
\end{equation}
A substitution of \eqref{trasition} into 
\eqref{conditionalMoment1} yields
\begin{equation}
\mathbb{E}\left\{\bm u(\bm X_j)|\bm X_n=\bm x\right\} = \int  \mathbb{E}\left\{\bm u(\bm X_j)|\bm X_q=\bm y\right\} p_{q|n}(\bm y|\bm x)d\bm y, 
\label{M0}
\end{equation}
which holds for all $j,n,q\in \{0,\ldots, L-1\}$. At this point, 
it is convenient to define the integral operator
\begin{equation}
\M(n,q)f(\bm x)= \int  f(\bm y) p_{q|n}(\bm y|\bm x)d\bm y,
\label{M}
\end{equation}
which is known as {\em composition} \cite{dominy2017duality} or ``stochastic 
Koopman'' \cite{Mezic2019,zhu2020} operator. {\color{r}The operator \eqref{M} is also related to the Kolmogorov backward equation \cite{pavliotis2014stochastic} }.
Thanks to the Chapman-Kolmogorov identity \eqref{trasition}, the
operators $\M(q,j)$ satisfy
\begin{equation}
\M(n,q) = \M(n,j)\M(j,q),\qquad \M(j,j)=\I,\qquad \forall n,j,q \in \{0,\ldots,L\},
\label{groupM}
\end{equation}
{\color{r} where $\I$ is the identity operator}.
Equation \eqref{groupM} allows us to map the 
conditional expectation  \eqref{conditionalMoment1}  
of any measurable phase space function $\bm u(\bm X_j)$ 
forward or backward through the network. 
As an example, consider again a neural network with four 
layers and states $\{\bm X_0,\dots,\bm X_4\}$. 
We have
\begin{align}
\mathbb{E}\{\bm u(\bm X_j)|\bm X_2=\bm x\} 
=&\M(2,3)\M(3,4) 
\mathbb{E}\{\bm u(\bm X_j)|\bm X_4=\bm x\} 
\nonumber\\
=& \M(2,1)\M(1,0)  
\mathbb{E}\{\bm u(\bm X_j)|\bm X_0=\bm x\}.
\label{p99}
\end{align}
Equation \eqref{p99} holds for every  $j\in\{0,..,4\}$. Of particular 
interest in the machine-learning context is the 
conditional expectation of $\bm u(\bm X_L)$ 
(network output) given $\bm X_0=\bm x$ (network input), 
which can be computed as 
\begin{align}
\mathbb{E}\{\bm u(\bm X_L)|\bm X_0=\bm x\} 
&=\M(0,L) \bm u(\bm x), \nonumber\\
&=\M(0,1)\M(1,2)\cdots \M(L-1,L) \bm u(\bm x), 
\label{composition1} 
\end{align}
i.e., by propagating 
$\bm u(\bm x)=\mathbb{E}\{\bm u(\bm X_L)|
\bm X_L=\bm x\}$ 
{\em backward} through the neural network using 
single layer operators $\M(i-1,i)$. 
Similarly, we can compute, e.g., 
$\mathbb{E}\{\bm u(\bm X_0)|\bm X_L=\bm x\}$ as 
\begin{equation}
\mathbb{E}\{\bm u(\bm X_0)|\bm X_L=\bm x\} =
\M(L,0) \bm u(\bm x).
\end{equation}
For subsequent analysis, it is convenient to define 
\begin{equation}
\bm q_n(\bm x) = 
\mathbb{E}\{\bm u(\bm X_L)|\bm X_{L-n}=\bm x\}.
\label{qdef}
\end{equation}
In this way, if  
$\mathbb{E}\{\bm u(\bm X_L)|\bm X_n=\bm x\}$
is propagated {\em backward} through the network 
by $\M(n-1,n)$, then $\bm q_n(x)$ is propagated 
{\em forward} by the operator 
\begin{align}
\G(n,q)=&\M(L-n,L-q).
\label{G}
\end{align} 
In fact, equations \eqref{qdef}-\eqref{G} allow 
us to write  \eqref{composition1} in the equivalent form 
\begin{align}
\bm q_L(\bm x) 
= &\G(L,L-1)\bm q_{L-1}(\bm x)\nonumber\\
=& \G(L,L-1)\cdots \G(1,0)\bm q_{0}(\bm x),
\label{38}
\end{align}
i.e., as a forward propagation problem (see Figure \ref{fig:net3}). 
Note that we can write \eqref{38} 
(or \eqref{composition1}) explicitly 
in terms of iterated integrals involving single-layer transition 
densities as 
\begin{align}
\bm q_L(\bm x) = &\int \bm u(\bm y) 
p_{0|L}(\bm y|\bm x)d\bm y\nonumber\\ 
= &\int \bm u(\bm y) 
\left(\int\cdots \int p_{L|L-1}(\bm y|\bm x_{L-1})
\cdots p_{2|1}(\bm x_2|\bm x_1)  p_{1|0}(\bm x_1|\bm x) d\bm x_{L-1}\cdots d\bm x_{1}
\right)d\bm y.
\label{39}
\end{align}

%

\begin{figure}
\centerline{\includegraphics[scale=0.45]{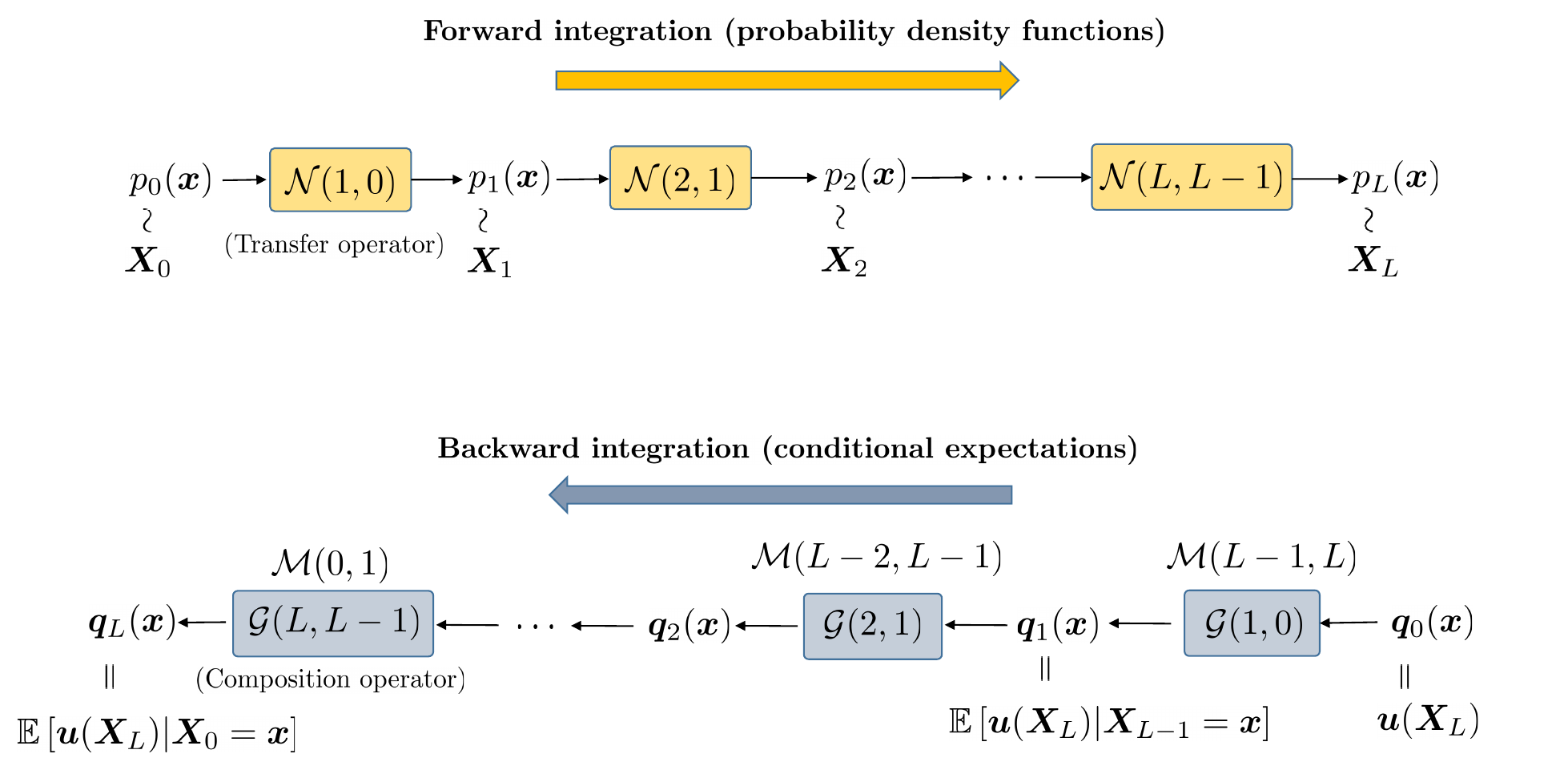}}
\caption{Sketch of the forward/backward integration process for probability density functions (PDFs) and conditional expectations. The transfer operator $\N(n+1,n)$ maps {\color{r}the PDF of $\bm X_n$ into the PDF of $\bm X_{n+1}$ }forward through the neural network. On the other hand, the composition operator $\M$ maps the conditional expectation $\mathbb{E}\left[\bm u(\bm X_L)|\bm X_{n+1}=\bm x\right]$ backwards to $\mathbb{E}\left[\bm u(\bm X_L)|\bm X_{n}=\bm x\right]$. By defining the operator $\G(n,m)=\M(L-n,L-m)$ we can transform the backward propagation problem for $\mathbb{E}\left[\bm u(\bm X_L)|\bm X_{n}=\bm x\right]$ into a forward propagation 
problem for $\bm q_n(\bm x)=\mathbb{E}\left[\bm u(\bm X_L)|\bm X_{L-n}=\bm x\right]$.}
\label{fig:net3}
\end{figure}

\subsection{Relation between composition and transfer operators}
\label{sec:MandN}
The integral operators $\M$ and $\N$ defined 
in \eqref{M} and \eqref{N} involve 
the same kernel function, i.e., the multi-layer 
transition density $p_{q|n}(\bm x,\bm y)$. 
In particular, $\M(n,q)$ integrates $p_{q|n}$  ``from the left'', 
while $\N(q,n)$ integrates it ``from the right''.
It is easy to show that $\M(n,q)$ and $\N(q,n)$ 
are adjoint to each other relative to the standard 
inner product in $L^2$ (see \cite{dominy2017duality} 
for the continuous-time case).  In fact, 
\begin{align}
\mathbb{E}\{\bm u(\bm X_k)\}
=&\int \mathbb{E}\{\bm u(\bm X_k)|\bm X_q=\bm x\} 
p_q(\bm x)d\bm x\nonumber\\
=&\int \left[\M(q,j) \mathbb{E}\{\bm u(\bm X_k)|\bm X_j=\bm x\}\right] p_q(\bm x)d \bm x\nonumber\\
=&\int \mathbb{E}\{\bm u(\bm X_k)|\bm X_j=\bm x\} 
\N(j,q) p_q(\bm x)d\bm x.
\end{align}  
Therefore 
\begin{equation}
\M(q,j)^*=\N(j,q) \qquad \forall q,j\in\{0,\ldots,L\}, 
\label{adjointrelation}
\end{equation}
where $\M(q,j)^*$ denotes the operator adjoint of $\M(q,j)$
with respect to the $L^2$ inner product. By 
invoking the definition \eqref{G}, we can also 
write \eqref{adjointrelation} as 
\begin{equation}\label{eq: duality-gn}
\G(L-q,L-j)^*=\N(j,q), \qquad \forall j,q\in\{0,\ldots,L\}.
\end{equation}
In \ref{sec:functionalsetting} we show that if the cumulative distribution 
function of each random vector $\bm \xi_n$ in the noise process
has partial derivatives that are Lipschitz  continuous 
in $\mathscr{R}(\bm \xi_n)$ (range of $\bm \xi_n$), then the composition 
and transfer operators defined in Eqs. \eqref{M} and \ref{N} 
are { bounded} in $L^2$ (see Proposition \ref{prop:L2boundeness} and Proposition \ref{lemma:operatorbounds}).
Moreover, is possible to choose the probability density 
of $\bm \xi_n$ such that the single layer composition 
and transfer operators become strict {contractions}. 
This property will be used {\color{r}in section \ref{sec:fadingMemory}} to 
prove that the memory of a stochastic neural network 
driven by particular types of noise decays with the number of layers.\\

{\subsection{Multi-layer conditional transition density}}
\label{sec:transition_density}
We have seen that the composition and the 
transfer operators $\M$ and $\N$ defined in Eqs. 
\eqref{M} and  \eqref{N}, allow us 
to push forward and backward conditional expectations and 
probability densities across the neural 
network. Moreover, such operators are adjoint
to one another {\color{r}(section \ref{sec:MandN})}, 
and also have the same kernel, i.e., the transition density 
$p_{n|q}(\bm x_n|\bm x_q)$. 
In this section, we derive analytical formulas for 
the one-layer transition density 
$p_{n+1|n}(\bm x_{n+1}|\bm x_n)$ 
{\color{r} corresponding to the neural network models 
we discussed in section \ref{sec:modeling}.}
The multi-layer transition density 
$p_{n|q}(\bm x_n|\bm x_q)$ 
is then obtained by composing one-layer 
transition densities as follows
\begin{equation}
p_{n|q}(\bm x_n|\bm x_q) =\int \cdots \int
p_{n|n-1}(\bm x_n|\bm x_{n-1})\cdots 
p_{q+1|q}(\bm x_{q+1}|\bm x_{q})
d \bm x_{n-1}\cdots d \bm x_{q+1}.
\end{equation} 
{\color{r}
We first consider the general class of stochastic 
neural network models defined by equation \eqref{discrete_dyn}. 
By the definition of conditional probability density,  we have 
\begin{align}
p_{n+1|n}(\bm x_{n+1}|\bm x_{n}) = \int 
p_{\bm X_{n+1}|\bm X_n,\bm \xi_n}(\bm x_{n+1}|\bm x_n,\bm \xi_n)p_{\bm \xi_n|\bm X_n}(\bm \xi_{n}|\bm x_{n})d\bm \xi_n. 
\end{align}
By assumption, $p_{\bm \xi_n|\bm X_n}(\bm \xi_{n}|\bm x_{n})=\rho_n(\bm \xi_n)$ (the random vector $\bm \xi_n$ is independent of $\bm X_n$) and therefore
\begin{align}
p_{n+1|n}(\bm x_{n+1}|\bm x_{n}) = 
\int \delta \left( \bm x_{n+1}- \bm H_n\left(\bm x_n,\bm w_n,\bm \xi_n\right)\right)\rho_n(\bm \xi_{n})d\bm \xi_n, 
\label{transitiongeneral}
\end{align}
where we denoted by $\delta(\cdot)$ the Dirac delta function, and 
set $\rho_n(\bm \xi_{n})=p_{\bm \xi_n}(\bm \xi_n)$. The delta 
function arises because if $\bm x_n$ and $\bm \xi_n$ are known 
then $\bm x_{n+1}$ is obtained by a purely deterministic 
relationship, i.e, Eq. \eqref{discrete_dyn}.

The general expression \eqref{transitiongeneral} can be simplified 
for particular classes of stochastic neural network models. For example, 
if the neural network has purely additive noise as 
in equation \eqref{discrete_sys_10}, 
then by using elementary properties of the delta function we obtain
\begin{align}
p_{n+1|n}(\bm x_{n+1}|\bm x_{n}) = &
\int \delta \left( \bm x_{n+1}- \bm F_n\left(\bm x_n,\bm w_n\right)-\bm \xi_n\right)\rho_n(\bm \xi_{n})d\bm \xi_n\nonumber\\
=&\rho_n\left(\bm x_{n+1}- \bm F_n\left(\bm x_n,\bm w_n\right)\right).
\label{TDadditive}
\end{align}
Note that such transition density depends on the PDF 
of random vector $\bm \xi_n$ (i.e., $\rho_n$), the 
one-layer transition function $\bm F_n$, and the parameters  $\bm w_n$. 
Similarly, one-layer transition density associated with the 
stochastic neural network model \eqref{discrete_sys} can 
be computed by substituting $\bm H_n(\bm x_n,\bm w_n,\bm \xi_n)=
\bm F_n(\bm x_n,\bm w_n) - \bm M_n(\bm x_n)\bm \xi_n$ into \eqref{transitiongeneral}. This yields  
\begin{align}
p_{n+1|n}(\bm x_{n+1}|\bm x_n) = 
\int \delta\left(\bm x_{n+1} - 
\bm F_n(\bm x_n,\bm w_n) - \bm M_n(\bm x_n)\bm \xi_n\right) 
\rho_n(\bm \xi_n)d\bm \xi_n. 
\label{transitionDensity2}
\end{align}
By using well-known properties of the multivariate delta 
function \cite{Khuri} it is possible to re-write  
the integrand in \eqref{transitionDensity2} in a more 
convenient way. For instance, if the matrix $\bm M_n(\bm x_n)$ has
full rank then  
\begin{equation}
\delta\left(\bm x_{n+1} - \bm F_n(\bm x_n,\bm w_n) - 
\bm M_n(\bm x_n)\bm \xi_n\right)= 
\frac{1}{\left|\det(\bm M_n(\bm x_n))\right|}
\delta \left(\bm \xi_n - \bm M_n(\bm x_n)^{-1}\left[\bm x_{n+1} - 
\bm F_n(\bm x_n,\bm w_n)\right]\right),
\end{equation}
which yields 
\begin{equation}
p_{n+1|n}(\bm x_{n+1}|\bm x_n) = \frac{1}{\left|\det(\bm M_n(\bm x_n))\right|} 
\rho_n\left(\bm  M(\bm x_n)^{-1}\left[\bm x_{n+1} - 
\bm F_n(\bm x_n,\bm w_n)\right]\right).
\label{trandenM}
\end{equation}}
Other cases where  $\bm M_n(\bm x_n)$ is not a square 
matrix can be handled similarly \cite{Papoulis,Khuri}. 
{\color{r}Finally, consider the neural network model with 
random weights and biases \eqref{NNrandomweights}. The one-layer 
transition density in this case can be expressed as  
\begin{align}
p_{n+1|n}(\bm x_{n+1}|\bm x_n) = 
\int \delta\left( 
\bm x_{n+1} - \varphi\left(\bm Z_n \bm x_n + \bm z_n\right)\right) 
p(\bm Z_n,\bm z_n)d\bm Z_n d\bm z_n,  
\label{transitionDensity3}
\end{align}
where $p(\bm Z_n,\bm z_n)$ is the joint PDF of the 
weight matrix and bias vector in the $n$-th layer. 
}

\vs
\noindent
{\bf Remark:}
The transition density \eqref{TDadditive} associated with the 
neural network model \eqref{discrete_sys_10} can be computed explicitly 
once we choose a probability model for $\bm \xi_n\in \mathbb{R}^N$. 
For instance, if we assume that 
$\{\bm \xi_0,\bm \xi_1,\ldots,\bm \xi_{L-1}\}$ 
are i.i.d. Gaussian random vectors with PDF,
\begin{equation}
\rho_n(\bm \xi)=\frac{1}{(2\pi)^{N/2}}
e^{-\bm \xi\cdot\bm \xi/2}, \qquad n=0,\ldots,L-1, 
\label{PDFXI}
\end{equation}
then we can explicitly write the one-layer 
transition density \eqref{TDadditive} as 
\begin{equation}
p_{n+1|n}(\bm x_{n+1}|\bm x_n) = 
\frac{1}{(2\pi)^{N/2}}
\exp\left[-\frac{\left[\bm x_{n+1}
-\bm F_n(\bm x_{\bm n},\bm w_n)\right]\cdot \left[\bm x_{n+1}
-\bm F_n(\bm x_{\bm n},\bm w_n)\right]}{2}\right].
\label{TDGaussadditive}
\end{equation}
In \ref{sec:functionalsetting} we provide an analytical 
example of transition density for a neural network with two 
layers, one neuron per layer, $\tanh(\cdot)$ activation 
function, and uniformly distributed random noise.

\subsection{The zero noise limit}
An important question is what happens to the neural 
network as we send the amplitude of the noise to zero. 
To answer this question consider the neural network model  
\eqref{discrete_sys_10} with $N$ neurons per layer, 
and introduce the parameter 
$\epsilon\geq 0$, i.e., 
\begin{equation}
\bm X_{n+1} = \bm F_n(\bm X_n,\bm w_n) + \epsilon \bm \xi_n,
\label{discreteDS}
\end{equation}  
We are interested in studying the orbits of the discrete dynamical 
system \eqref{discreteDS} as $\epsilon \rightarrow 0$. To this end, we 
assume $\{\bm \xi_n\}$ independent random 
vectors with density $\rho_n(\bm x)$. 
This implies that the PDF of $\epsilon \bm \xi_n$ is 
\begin{equation}
\epsilon \bm \xi_n \sim \frac{1}{\epsilon^N} \rho_n
\left(\frac{\bm x}{\epsilon}\right).
\end{equation}
It is shown in \cite[Proposition 10.6.1]{Lasota} 
that the transfer operator $\N(n+1,n)$ associated 
with \eqref{discreteDS}, i.e.,
\begin{align}
p_{n+1}(\bm x)=&\N(n+1,n) p_n(\bm x)\nonumber\\
=&\int \frac{1}{\epsilon^N} \rho_n
\left(\frac{\bm x-\bm F_n(\bm z,\bm w_n)}{\epsilon}\right) 
p_n(\bm z)d\bm z
\end{align}
converges in norm to the Frobenius-Perron operator 
corresponding to  $\bm F_n(\bm X_n,\bm w_n)$ as 
$\epsilon \rightarrow 0$.
Indeed, in the limit $\epsilon\rightarrow 0$ we have, formally 
\begin{equation}
\lim_{\epsilon \rightarrow 0} 
p_{n|n+1}\left(\bm x_{n+1}|\bm x_n\right)=\lim_{\epsilon \rightarrow 0} \int \frac{1}{\epsilon^N} \rho_n
\left(\frac{\bm x_{n+1}-\bm F_n(\bm x_n,\bm w_n)}{\epsilon}\right)=\delta\left(\bm x_{n+1}-\bm F_n(\bm x_n,\bm w_n)\right).
\end{equation}
Substituting this expression into \eqref{N}, one gets,
\begin{equation}
p_{n+1}(\bm x)=\N(n+1,n) p_n(\bm x)= \int \delta\big(\bm x-\bm F_n(\bm z,\bm w_n) \big) p_n(\bm z)d\bm z.
\end{equation}
Similarly, a substitution into equation \eqref{38} yields
\begin{equation}
\bm q_{n}(\bm x)=\G(n,n-1) \bm q_{n-1}(\bm x) = \bm q_{n-1}\left(\bm F_{L-n}(\bm x, \bm w_{L-n}) \right).
\end{equation}
{\color{r} Iterating this expression all the way back to $n=1$ yields the familiar function composition rule for neural networks, i.e.,}
\begin{align}
\bm q_{L}=& \bm q_0 \Big(\bm F_{L-1} \big(\bm F_{L-2}(\cdots \bm F_0(\bm x, \bm w_{0}), \cdots,  \bm w_{L-2}), \bm w_{L-1}\big) \Big).
\label{qqg}
\end{align}
{\color{r}
Recalling that  $\bm q_0(\bm x)=\bm u(\bm x)$ 
and assuming that $\bm u(\bm x)= \bm A \bm x$ (linear output layer),  
where $\bm A$ is a matrix of output weights and $\bm x$ 
is a column vector, we can write \eqref{qqg} as
\begin{align}
\bm q_{L}\bm (x)=& \bm A 
\bm F_{L-1}(\bm F_{L-2}(\cdots \bm F_1(\bm F_0(\bm x, \bm w_0),\bm w_1),\cdots,\bm w_{L-2}),\bm w_{L-1}).
\label{qqg1}
\end{align}
If $\bm u(\bm x)$ is a linear scalar function, i.e., 
$u(\bm x)=\bm \alpha\cdot \bm x$ then \eqref{qqg1} 
coincides with equation \eqref{compR}.}

\section{\color{r}Training paradigms}
\label{sec:trainingParadigms}

By adding random noise to a neural 
network we are essentially adding an infinite number 
of degrees of freedom to our system. This allows 
us to rethink the process of training the neural network 
from a probabilistic perspective. In particular, instead 
of optimizing a performance metric\footnote{In a supervised 
learning setting the neural network weights are usually 
determined by minimizing a dissimilarity measure between 
the output of the network and a target function. Such measure
may be an entropy measure, the Wasserstein distance, the 
Kullback--Leibler divergence, or other measures defined by classical 
$L^p$ norms.} relative to the neural network weights 
$\bm w=\{\bm w_0,\bm w_1,\ldots,\bm w_{L-1}\}$ 
{\color{r}(classical ``training over weights'' paradigm)}, 
we can now optimize the transition density\footnote{
The transition density for a deterministic neural 
network model of the form $\bm X_{n+1}=\bm F_n(\bm X_n,\bm w_n)$ 
is 
\begin{equation}
p_{n+1|n}(\bm x_{n+1}|\bm x_n)= \delta \left(\bm x_{n+1}-
\bm F_n(\bm x_n,\bm w_n)\right),
\end{equation} 
where $\delta(\cdot)$ is the Dirac delta function. 
Such density does not have any degree of freedom other than 
$\bm w_n$. On the other hand, in a stochastic setting we may 
be allowed to {\em choose} the PDF of $\bm \xi_n$. For a neural network model 
of the form $\bm X_{n+1}=\bm F_n(\bm X_n,\bm w_n)+\bm \xi_n$
the transition density has the form  
\begin{equation}
p_{n+1|n}(\bm x_{n+1}|\bm x_n)= \rho_n \left(\bm x_{n+1}-
\bm F_n(\bm x_n,\bm w_n)\right),
\end{equation} 
where $\rho_n(\bm \xi)$ is the PDF of $\bm \xi_n$.
This allows us to rethink the process of training the 
neural network from a probabilistic perspective, e.g., 
by optimizing over $\rho_n$.} 
$p_{n+1|n}(\bm x_{n+1}|\bm x_n)$. 
Clearly, such transition {\color{r}density depends on the 
neural network weights and on the functional form of 
the one-layer transition function, e.g., as in 
equation \eqref{TDadditive}}. 
Hence, if we prescribe the PDF of $\bm \xi_n$ (e.g., $\rho_n$ 
in \eqref{TDadditive}), then the transition density $p_{n+1|n}$ is 
uniquely determined by the functional form of 
function $\bm F_n$, and by 
the weights $\bm w_n$. On the other hand, {\color{r}if 
we are allowed to choose the PDF of the random vector 
$\bm \xi_n$}, then we can optimize it during training. 
{\color{r}This can be done while keeping the neural network 
weights $\bm w_n$ fixed, or by including them in the 
optimization process.}

The interaction between random noise 
and the nonlinear dynamics modeled by the 
network can yield surprising results. 
For example, in stochastic resonance 
\cite{Nozaki,Venturi_MZ} it is well known that 
random noise added to a properly tuned a bi-stable 
system can induce a peak in the Fourier power spectrum 
of the output - hence effectively amplifying 
the signal. Similarly, the random noise added to 
a neural network can be leveraged to achieve 
specific goals. 
For example, {\color{r} noise allows us to  
re-purpose a previously trained network on a different 
task without changing the weights of network. This can be
seen as an instance of stochastic transfer learning.} 
%
{\color{r}To describe the method, consider the two-layer 
neural network model
\begin{equation}
\bm X_1 = \bm F_0(\bm X_0,\bm w_0)+\bm \xi_0, \qquad 
\bm X_2 = \bm F_1(\bm X_1,\bm w_1),
\label{Modadd}
\end{equation}
with $N$ neurons per layer, input 
$\bm X_0\in \Omega_0\subseteq \mathbb{R}^d$, linear output
$u(\bm x) =\bm \alpha\cdot \bm x$, hyperbolic tangent activation function, 
and intra-layer random perturbation $\bm \xi_0$.
We are interested in training the input-output map represented 
by the conditional expectation (see Eq. \eqref{out})
\begin{equation}
q_2(\bm x)= \bm \alpha\cdot \mathbb{E}\left[\bm X_2|\bm X_0=\bm x\right], \qquad \bm x\in \Omega_0.
\label{io2}
\end{equation}
Let us first re-write \eqref{io2} in a more explicit form. To this end, 
we recall that 
\begin{equation}
q_0(\bm x)= \bm \alpha \cdot \mathbb{E}\left[\bm X_2|\bm X_2=\bm x\right] = \bm \alpha \cdot \bm x\qquad \bm x\in \mathscr{R}(\bm X_2)=[-1,1]^N,
\end{equation}
where $\mathscr{R}(\bm X_2)$ denotes the range of the mapping 
$\bm X_2=\bm F_1(\bm F_0(\bm X_0,\bm w_0)+\bm \xi_0,\bm w_1)$ for $\bm X_0\in \Omega_0$ and arbitrary weights $\bm w_0$ and $\bm w_1$. 
By using the definition of the operator $\G(i+1,i)$ in \eqref{G} and 
the composition rule $q_{i+1}=\G(i+1,i)q_i$ ($i=0,1$) we easily obtain 
\begin{align}
q_1(\bm x)= &\G(1,0)q_0\nonumber\\
= &\int_{\mathscr{R}(\bm X_2)} q_0(\bm y) 
p_{2|1} \left(\bm y|\bm x\right) d\bm y\nonumber\\
= &\int_{[-1,1]^N} \bm \alpha\cdot \bm y 
\delta \left(\bm y-\bm F_1(\bm x,\bm w_1)\right)d\bm y\nonumber\\
= &\bm \alpha\cdot \bm F_1(\bm x,\bm w_1)\qquad \bm x\in \mathscr{R}(\bm X_1),
\label{q111}
\end{align}
and
\begin{align}
q_2(\bm x)= &\G(2,1)q_1\nonumber\\
= &\int_{\mathscr{R}(\bm X_1)} q_1(\bm y) 
p_{1|0} \left(\bm y|\bm x\right) d\bm y\nonumber\\
= &\int_{\mathscr{R}(\bm X_1)} q_1(\bm y) 
\rho_0\left(\bm y-\bm F_0(\bm x,\bm w_0)\right) d\bm y\nonumber\\
= &\int_{\mathscr{R}(\bm \xi_0)} q_1\left(\bm z+\bm F_0(\bm x,
\bm w_0)\right)\rho_0(\bm z)d\bm z, \qquad 
\bm x\in \Omega_0\subseteq \mathbb{R}^d.
\label{q222}
\end{align}
where $\mathscr{R}(\bm\xi_0)$ is the range of the random 
variable $\bm \xi_0$, i.e., the support of $\rho_0$. Hence, we 
can equivalently write input-output map \eqref{io2} as
\begin{equation}
q_2(\bm x)= \bm \alpha\cdot\int_{\mathscr{R}(\bm \xi_0)} 
\bm F_1\left(\bm z+\bm F_0(\bm x,\bm w_0)\right)\rho_0(\bm z)d\bm z,\qquad
\bm x\in \Omega_0\subseteq \mathbb{R}^d. 
\label{equazione}
\end{equation}

}

\subsection{Training over weights}

In the absence of noise, {\color{r} the PDF of $\bm \xi_0$ appearing 
in \eqref{equazione}, i.e, $\rho_0(\bm z)$, reduces to the 
delta function $\delta(\bm z)$. Hence, the output of 
the neural network \eqref{equazione} can be written as
\begin{equation}
q_2(\bm x)=\bm \alpha\cdot \underbrace{
\bm F_1(\bm F_0(\bm x,\bm w_0),\bm w_1)}_{\mathbb{E}[\bm X_2|\bm X_0=\bm x]}, \quad \bm x\in \Omega_0.
\label{detN}
\end{equation}
This is consistent with the well-known composition rule for deterministic networks. 
The parameters $\{\bm \alpha, \bm w_0, \bm w_1\}$ appearing 
in \eqref{detN} can be optimized to minimize a dissimilarity 
measure between $q_2(\bm x)$ and a given target function 
$f(\bm x)$, e.g., relative to the $L^2(\Omega_0)$ norm
\begin{equation}
\left\|q_2(\bm x)-f(\bm x) \right\|^2_{L^2(\Omega_0)} =\int_{\Omega_0} \left[q_2(\bm x)-f(\bm x)\right]^2d\bm x,
\end{equation}
or a discrete $L^2(\Omega_0)$ norm computed on point set $\left\{\bm x{[1]},\ldots,\bm x{[S]}\right\}\in \Omega_0$
\begin{equation}
\left\|q_2(\bm x)-f(\bm x) \right\|_2^2=\sum_{k=1}^S \left[q_2\left(\bm x{[k]}\right)-f\left(\bm x{[k]}\right)\right]^2.
\end{equation}
The brackets $[\cdot]$ here are used to label the data points.}

\subsection{Training over noise} 
By adding noise $\bm \xi_0\in \mathbb{R}^N$ to 
the output of the first layer we obtain the input-output 
map \eqref{equazione}, hereafter rewritten for convenience
{\color{r}
\begin{equation}
{q}_2(\bm x)=\bm \alpha\cdot
\underbrace{\int_{\mathscr{R}(\bm \xi_0)} \bm F_1(\bm \xi+\bm F_0(\bm x,\bm w_0),\bm w_1) 
\rho_0\left(\bm \xi\right)d\bm \xi}_{\mathbb{E}[\bm X_2|\bm X_0=\bm x]},
\label{Volterra}
\end{equation}}
where $\rho_0$ denotes the PDF of $\bm \xi_0$.
Equation \eqref{Volterra} looks like a Fredholm integral 
equation of the first kind. In fact, it 
can be written as 
\begin{equation}
{q}_2(\bm x)=\int_{\mathscr{R}(\bm \xi_0)}\kappa_2(\bm x,\bm \xi) 
\rho_0\left(\bm \xi\right)d\bm \xi,
\label{FH}
\end{equation}
where 
\begin{equation}
\kappa_2(\bm x,\bm \xi)=\bm \alpha\cdot\bm 
F_1(\bm \xi+\bm F_0(\bm x,\bm w_0),\bm w_1).
\label{K2}
\end{equation}
However, differently from standard Fredholm equations 
of the first kind, in \eqref{FH} we have that 
$\bm x\in \Omega_0\subseteq \mathbb{R}^d$ 
while $\bm \xi\in \mathbb{R}^N$, i.e., the integral operator 
with kernel $\kappa_2$ maps functions with $N$ variables 
into functions with $d$ variables. 
We are interested in finding a PDF $\rho_0(\bm y)$ that solves 
\eqref{Volterra} for a given function $h(\bm x)$, i.e., 
{\color{r}find $\rho_0$ such that 
\begin{equation}
h(\bm x)=\int\kappa_2(\bm x,\bm \xi) 
\rho_0\left(\bm \xi\right)d\bm \xi.
\label{FH1}
\end{equation}
If such PDF $\rho_0$ exists, then we can re-purpose 
the neural network \eqref{detN} 
with output $q_2(\bm x)\simeq f(\bm x)$ to 
approximate a different function 
$h(\bm x)$, without modifying the weights  
$\{\bm w_1,\bm w_0\}$ but rather simply 
adding noise $\bm \xi_0$ between the first and the second layer, 
and then averaging the output over the PDF $\rho_0$.
Equation \eqref{FH1} is unfortunately ill-posed 
in the space of probability distributions. 
In other words, for a given kernel $\kappa_2$, and 
a given target function $h(\bm x)$, there is 
(in general) no PDF $\rho_0$ that satisfies \eqref{FH1} exactly.}
However, one can proceed by optimization. 
For instance, $\rho_0$ can be determined by solving the 
constrained least squares problem\footnote{The optimization problem 
\eqref{leastsquares1} is a quadratic program with linear constraints if we represent $\rho_0$ in the span of a basis made of positive functions, e.g., Gaussian kernels \cite{Botev}.}
{\color{r} 
\begin{equation}
\{\rho_0,\bm \alpha\}=\argmin_{(\rho,\bm \alpha)}\left\|h(\bm x) - 
\bm \alpha\cdot\int_{\mathscr{R}(\bm \xi_0)} \bm 
F_1(\bm \xi+\bm F_0(\bm x,\bm w_0),\bm w_1)\rho(\bm \xi)d\bm \xi\right\|_{L^2(\Omega)}, \quad \|\rho\|_{L^1(\mathbb{R}^N)}=1, 
\quad \rho\geq 0. 
\label{leastsquares1}
\end{equation}
Note that the training-over-noise paradigm 
can be seen as an instance of {\em transfer learning} \cite{pan2010survey}, in 
which we turn the knobs on the PDF of the noise $\rho_0$ 
(changing it from a Dirac delta function to a proper PDF), 
and eventually the coefficients $\bm \alpha$, 
to approximate a different function while 
keeping the neural network weights and 
biases fixed. Training over noise can also be performed in 
conjunction with training over weights, to improve the 
overall optimization process of the neural network.}

{\color{r} 
\vs
\noindent
{\em An example:} Let us demonstrate the 
``training over noise'' and the ``training over weights' 
paradigms with a simple numerical example. 
Consider the following one-dimensional function 
\begin{equation}
f(x) = \sin(7\pi x)e^{-\cos^3(x)} \qquad x\in \Omega_0=[0,1], 
\label{f1d}
\end{equation}
We are interested in approximating $f(x)$ with the 
two-layer neural network depicted 
in Figure \ref{fig:1DexampleNET} ($N=5$ neurons per layer).
\begin{figure}
\centerline{\includegraphics[scale=0.45]{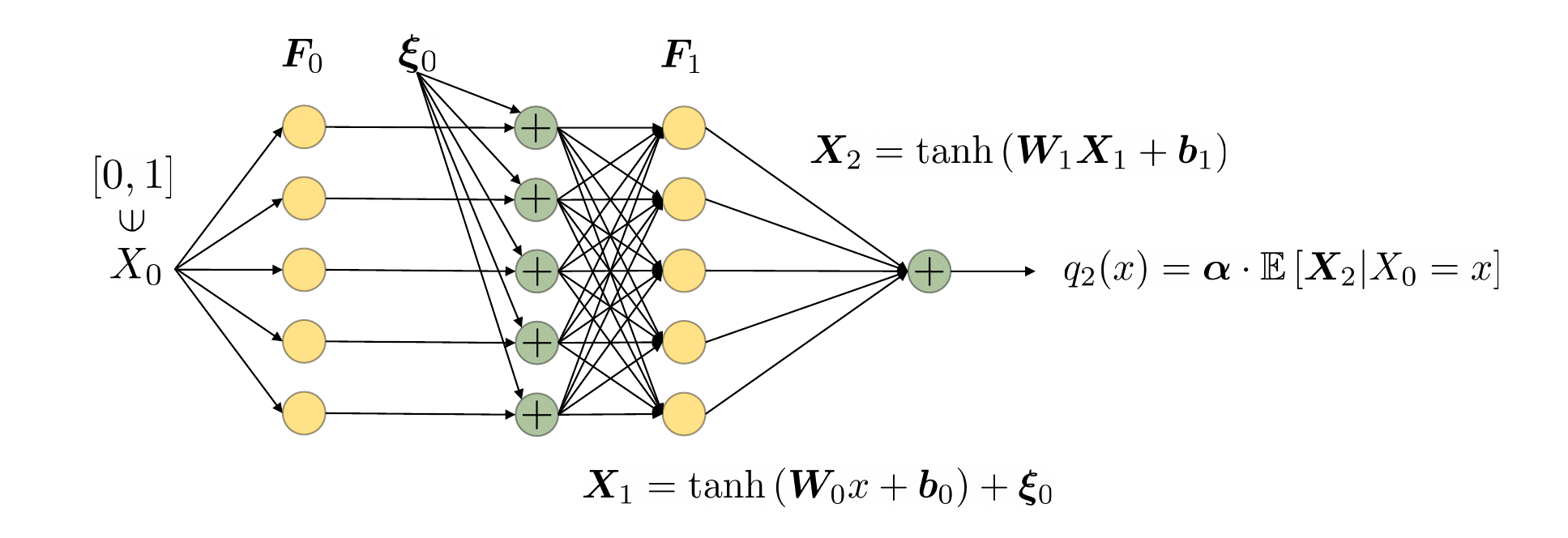}}
\caption{\color{r}Sketch of the stochastic neural network model used approximate the functions \eqref{f1d} (training over weight paradigm) and \eqref{h1d} (training over noise paradigm). The five-dimensional random vector $\bm \xi_0$ is assumed to have statistically independent components. We proceed by first training the neural network with no noise on the target function $f(x)$ defined in \eqref{f1d}. Subsequently, we perturb the network with the random vector $\bm \xi_0$, and optimize the PDF of $\bm \xi_0$ so that  the conditional expectation of the neural network output, i.e.,  \eqref{FHex1}, approximates a second target function $h(x)$ for the same weights and biases.}
\label{fig:1DexampleNET}
\end{figure}
In the absence of noise, the output of the network is given by equation \eqref{detN}, hereafter rewritten 
in full form for $\tanh(\cdot)$ activation functions \cite{Mishra}
\begin{align}
q_2(x) = &\bm \alpha\cdot \tanh\left[\bm W_1 \tanh\left(\bm W_0 x+\bm b_0\right)+\bm b_1\right].
\label{NN1dout}
\end{align}
Here, $\bm W_0$, $\bm b_0$, $\bm b_1$ and $\bm \alpha$ are five-dimensional column vectors, while $\bm W_1$ is a $5\times 5$ matrix.
Hence, the input-output map \eqref{NN1dout} has $45$ 
free parameters 
$\{\bm W_0,\bm W_1,\bm b_0,\bm b_1,\bm \alpha\}$ which
are determined by minimizing the discrete 2-norm
\begin{equation}
\left\| q_2(x)-f(x)\right\|^2_2 = \sum_{i=1}^{30}\left[q_2\left(x{[i]}\right)-f\left(x{[i]}\right)\right]^2,
\label{points}
\end{equation} 
where $\left\{x{[1]},\ldots,x{[30]}\right\}$ is an
evenly-spaced set of points in $[0,1]$ 
\begin{equation}
x[j]=\frac{(j-1)}{29}\quad j=1,\ldots,30.
\end{equation} 
In Figure \ref{fig:trainingf} we show the neural network 
output \eqref{NN1dout} we obtained by minimizing the 
cost \eqref{points} relative to the weights 
$\{\bm W_0,\bm W_1,\bm b_0,\bm b_1,\bm \alpha\}$ 
(training over weights paradigm).

Next, we add noise to our fully trained deterministic 
neural network. Specifically, we perturb the output of the first layer 
by an additive random vector $\bm \xi_0$ with independent 
components supported in $[-0.4,0.4]$. 
Since the random vector $\bm \xi_0$ is assumed to have 
independent components, we can write its PDF $\rho_0$ as   
\begin{equation}
\rho_0(\bm \xi) = \rho^1_0\left(\xi_1\right)\cdots \rho^N_0\left(\xi_N\right)
\label{1DPDF_a}
\end{equation}
where $\{\rho^1_0,\ldots,\rho_1^N\}$ are 
one-dimensional PDFs, each one of which 
is supported in $[-0.4,0.4]$. 
In the training-over-noise paradigm, we are interested 
in finding the PDF of the random vector $\bm \xi_0$, i.e., the 
one-dimensional PDFs $\{\rho^1_0,\ldots,\rho_1^N\}$ 
appearing in \eqref{1DPDF_a}, and a new vector 
of coefficients $\bm \alpha$ such that the output of 
the neural network (with the same weights and biases) 
averaged over all realizations of 
the noise $\bm \xi_0$, approximates a new one-dimensional map 
$h(x)$, different from \eqref{f1d}. For this 
example, we choose
\begin{equation}
h(x) = 4\tanh\left(10x-\frac{7}{2}\right)+3.
\label{h1d}
\end{equation}
In the presence of noise, the neural network output 
takes the form (see Eq. \eqref{FH}) 
\begin{align}
\widehat{q}_2(x)=\int_{\mathscr{R}(\bm \xi_0)} \kappa_2(x,\bm \xi) \rho_0^1(\xi_1)\cdots
\rho_0^5(\xi_5) d\xi_1\cdots  d\xi_5,
\label{FHex1}
\end{align}
where $\mathscr{R}(\bm \xi_0)=[-0.4,0.4]^5$ is the range 
of $\bm \xi_0$, and
\begin{align}
\kappa_2( x,\bm \xi)=&\bm \alpha\cdot\tanh\left[\bm W_1 \left(\bm \xi+\tanh\left(\bm W_0 x+\bm b_0\right)\right)+\bm b_1\right].
\end{align}
We approximate the $5$-dimensional integral
in \eqref{FHex1} with a Gauss-Legendre-Lobatto (GLL)  
quadrature formula \cite{Hesthaven} on a tensor-product 
grid with $6$ quadrature points per dimension. 
To this end, let $\{z[1],\ldots, z[6]\}$
be the GLL quadrature points in $[-0.4,0.4]$. The 
tensor product quadrature approximation of \eqref{FHex1} 
takes the form
\begin{align}
\widehat{q}_2(x)\approx \sum_{j=1}^{H} \theta_j
\kappa_2\left(x,\bm \xi [j]\right) \rho_0^1(z[i_1(j)])\cdots
\rho_0^5(z[i_5(j)]),
\label{FHex2}
\end{align}
where $H=6^5=7776$ is the total number of quadrature points\footnote{\color{r}As is well known, the curse of 
dimensionality in the tensor product 
quadrature rule \eqref{FHex2}, i.e., the exponential growth in the number of nodes with the dimension can be mitigated by using, e.g., sparse grids \cite{Novak,Bungartz} or 
quasi-Monte Carlo (qMC) quadrature \cite{qmc}.} in the domain 
$[-0.4,0.4]^5$, $\theta_k$ are 
tensor product GLL quadrature weights, and 
\begin{equation}
\bm \xi[j] = (z[i_1(j)],\ldots, z[i_5(j)])
\end{equation}
represents a grid in $[-0.4,0.4]^5$ indexed 
by $\{i_1(j),\ldots,i_5(j)\}$, where $i_k(j)\in\{1,\ldots,6\}$ for 
each $j$ and each $k$. 
Such indices are obtained by an appropriate 
ordering of the nodes in the tensor product grid.
We represent each one-dimensional PDFs 
$\rho_0^k(z)$ using a polynomial 
interpolant through the GLL points, i.e., 
\begin{equation}
\rho_0^k(z) \simeq \sum_{j=1}^6 \rho_0^k\left(z[j]\right) l_j(z),
\label{lagrange}
\end{equation}
where $l_j(z)$ are Lagrange characteristic 
polynomials associated with the one-dimensional GLL grid. 
Thus, the degrees of freedom of each PDF are represented 
by the following vector of PDF values at the GLL nodes
\begin{equation}
\bm \rho_0^k=\{\rho_0^k\left(z[1]\right),\ldots,\rho_0^k\left(z[6]\right)\}, \quad k=1,\ldots,5.
\label{DOF-pdf}
\end{equation}
Note that in this setting we are approximating the PDF of $\bm \xi_0$ 
using a non-parametric method, i.e., a polynomial interpolant 
through a tensor product GLL grid. For non-separable PDFs, 
or for PDFs in higher dimensions, it may be more practical to 
consider a tensor representation \cite{Dektor_dyn_approx,adaptive_rank}, 
or a parametric inference method, i.e., a method that leverages 
assumptions on the shape of the probability distribution of $\bm \xi_0$.}

{\color{r}At this point, we have all the elements to solve the 
minimization problem \eqref{leastsquares1}, or an equivalent 
problem defined by the discrete 2-norm
\begin{equation}
\min_{\{\bm \rho_0^1,\ldots,\bm \rho_0^5,\bm \alpha\}} \sum_{i=1}^S   \left[h(x[i])-\widehat{q}_2(x[i])\right]^2, 
\label{points1}
\end{equation}
subject to the linear constraints\footnote{\color{r}In a discrete 
setting, the non-negativity constraints on the PDFs in \eqref{con1} are enforced using a finite set of linear inequality constraints. In practice we evaluate the Lagrange interpolation formula \eqref{lagrange} on a grid of 200 points in $[-0.4,0.4]$ and enforce that the polynomial interpolant of each PDF is non-negative at each point in the grid. Similarly, the $L^1$ normalization condition of each PDF is enforced using one-dimensional GLL quadrature.}
\begin{equation}
\left\|\rho_0^k\right\|_{L^1([-0.4,0.4])}=1, \qquad 
\rho_0^k(z)\geq 0\qquad (k=1,\ldots,5). 
\label{con1}
\end{equation}
}

\begin{figure}[t]
\centerline{\footnotesize Training over weights\hspace{6.cm} Training over noise }
\centerline{
\includegraphics[scale=0.43]{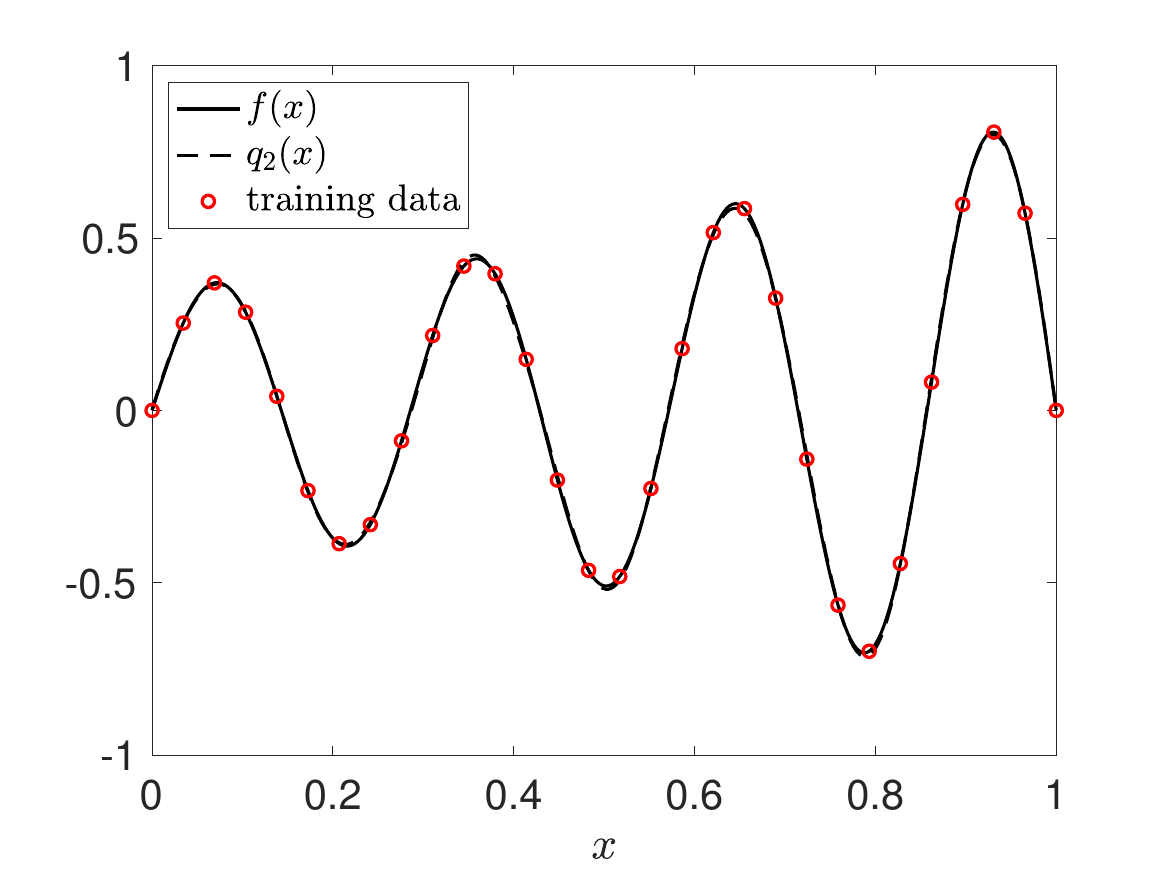}
\includegraphics[scale=0.43]{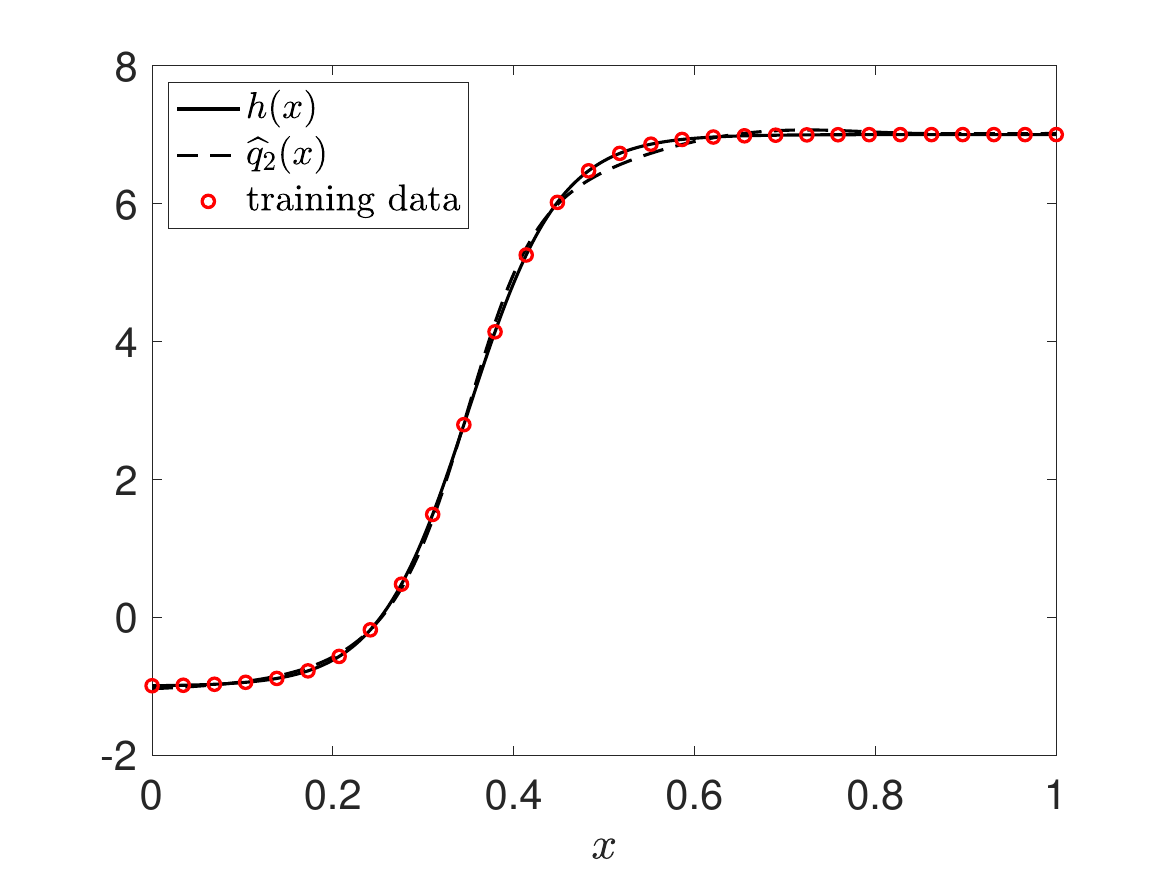}
}
\caption{\color{r} Demonstration of "training over weights" and ``training over noise" paradigms. In the training over weights paradigm we 
minimize the dissimilarity measure \eqref{points} between 
the output \eqref{NN1dout} of the two-layer neural network 
depicted in Figure \ref{fig:1DexampleNET} (with $\bm \xi_0=\bm 0$)
and the target function \eqref{f1d}. The training 
data is shown with red circles. 
In the training over noise paradigm we add random noise 
to the output of the first layer and optimize for $\bm \alpha$ and 
the PDF of the noise as in 
\eqref{leastsquares1}. This can be seen as 
an instance of {\em transfer learning}, in which 
we keep the neural network weights and 
biases fixed but update the output weights 
$\bm \alpha$ and the PDF of the random 
vector $\bm \xi_0$ to approximate a different 
function $h(x)$ defined in \eqref{h1d}.}
\label{fig:trainingf}
\end{figure}

{\color{r}In Figure \ref{fig:trainingf} we demonstrate the 
training-over-weight and the training-over-noise paradigms 
for the neural network depicted in 
Figure \ref{fig:1DexampleNET}.  
In the classical training over weight paradigm we minimize the 
error between the neural network output 
\eqref{NN1dout} and the function \eqref{f1d} 
in the discrete 2-norm \eqref{points}. The training 
data is shown with red circles. 
In the training over noise paradigm we add random noise 
$\bm \xi_0$ to the output of the first layer. This yields 
the input-output map \eqref{FHex1}. By optimizing 
for the PDF of the noise $\rho_0$ and the coefficients
$\bm \alpha$ as in \eqref{leastsquares1} 
we can   re-purpose the network previously trained on $f(x)$ 
to approximate a different function $h(x)$ defined 
in \eqref{h1d}, without changing the neural network weights and biases.

In Figure \ref{fig:PDFnoise} we plot the 
one-dimensional PDFs of each component of the random vector 
$\bm \xi_0$ we obtained from optimization. Such PDFs depend 
on the neural network weights and biases, which in this 
example are kept fixed.
\begin{figure}[t]
\centerline{\hspace{-0.2cm}
\includegraphics[scale=0.42]{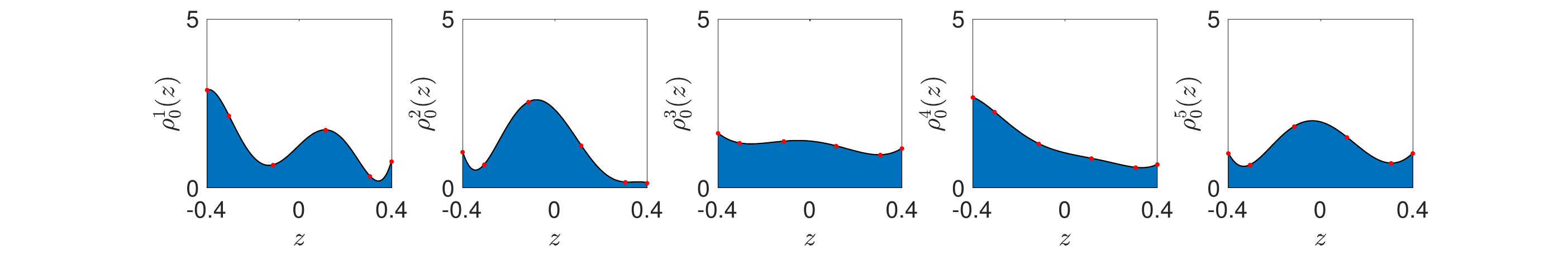}}
\caption{\color{r} Training over noise paradigm. 
One-dimensional probability density functions 
of each component of the random vector 
$\bm \xi_0$ obtained by solving the 
optimization problem \eqref{leastsquares1} for the 
function $h(x)$ defined in Eq. \eqref{h1d} 
(see Figure \ref{fig:trainingf}). 
The PDF of $\bm \xi_0$ is a product of all five 
PDFs (see Eq. \eqref{1DPDF_a}). 
The degrees of freedom of each PDF, i.e., the vectors 
defined in equation \eqref{DOF-pdf}, are visualized as 
red dots (PDF values at GLL points). Each PDF is 
a polynomial of degree at most 5.}
\label{fig:PDFnoise}
\end{figure}
The PDF of $\bm \xi_{0}$ is (by hypothesis) 
a product of five one-dimensional PDFs. 
Therefore it is quite straightforward to 
sample $\bm \xi_{0}$ by using, e.g., rejection 
sampling applied independently to each 
one-dimensional PDF shown in Figure \ref{fig:PDFnoise}. 
With the samples of $\bm \xi_0$ available, we 
can easily compute samples of the neural network output as
\begin{equation}
\tilde{q}_2(x) =\bm \alpha\cdot \tanh\left[\bm W_1 \left(\bm \xi_0+\tanh\left(\bm W_0 x+\bm b_0\right)\right)+\bm b_1\right].
\label{nnsamp}
\end{equation} 
Clearly, if we compute an ensemble average over 
a large number of output samples then we obtain an   
approximation of $\widehat{q}_2(x)$. This is demonstrated 
in Figure \ref{fig:q2 samples}.

\begin{figure}[t]
\centerline{
\includegraphics[scale=0.45]{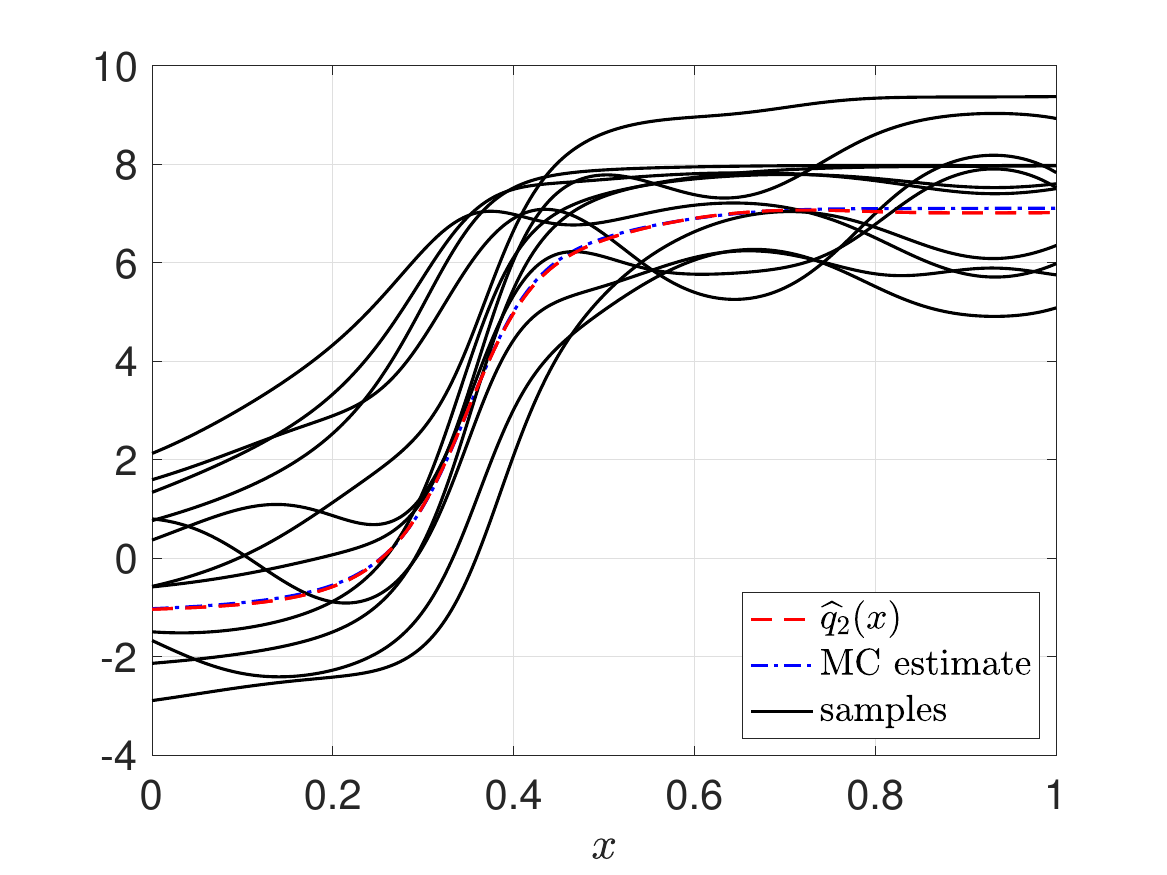}
}
\caption{\color{r} Training over noise paradigm. 
Samples of the stochastic neural network output \eqref{nnsamp} 
corresponding to random samples of $\bm \xi_0$. 
We also show that the ensemble average of the network 
output computed over $10^5$ independent samples of \eqref{nnsamp} 
converges to $\widehat{q}_2(x)$, as expected.}
\label{fig:q2 samples}
\end{figure}
}

\subsubsection{{\color{r} Random shifts}} 

{\color{r}A related but simpler setting 
for re-purposing a neural network is to introduce a 
random shift in the input variable rather than 
perturbing the network layers}\footnote{Note that if we 
do not have access to the layers of the 
neural network, then we can introduce 
{\color{r} random perturbations in the input in the form of 
random} shifts or other types of perturbations. 
In this setting one can {\color{r} 
re-purpose a pre-trained neural network in 
which the user is allowed only to modify the 
input and observe the output.}}. In this setting, 
the output of the network can be written as  
\begin{equation}
r_2(\bm x)= \int q_2(\bm x-\bm y) 
\rho\left(\bm y\right)d\bm y, 
\label{Volterra1}
\end{equation}
{\color{r}where $q_2$ is defined in \eqref{detN}, and 
$\rho$ is the PDF of vector $\bm \eta$ defining the 
random shift $\bm x\rightarrow \bm x - \bm \eta$. 
Clearly, equation \eqref{Volterra} is the expectation 
of the noiseless neural network output $q_2(\bm x)$ 
under a random shift with PDF $\rho(\bm y)$.
To re-purpose a deterministic neural net using 
random shifts in the input variable one can proceed 
by optimization, i.e., solving an optimization 
problem similar to \eqref{leastsquares1} for a target function $h(\bm x)$.}

\vs 
\noindent 
{\bf Remark:} Given a target function 
$h(\bm x)$ we can, in principle, compute 
the analytical solution of the integral equation 
\begin{equation}
h(\bm x)= \int q_2(\bm x-\bm y) 
\rho\left(\bm y\right)d\bm y, \qquad \bm x,\bm y\in \mathbb{R}^d
\end{equation}
using Fourier transforms. This 
yields\footnote{In equation \eqref{RR} we assumed 
that $\mathcal{F}[q_2(\bm x)] (\bm \zeta)\neq 0$ 
for all $\zeta\in \mathbb{R}^d$.}
\begin{equation}
\rho\left(\bm y\right) = \int 
\frac{\mathcal{F}[h(\bm x)](\bm \zeta)}
{\mathcal{F}[q_2(\bm x)](\bm \zeta)}e^{2\pi i\bm \zeta\cdot \bm x }
d\bm \zeta,
\label{RR}
\end{equation}
where $\mathcal{F}[\cdot]$ denotes the multivariate 
Fourier transform operator
\begin{equation}
\mathcal{F}[f(\bm x)](\bm \zeta) = \int f(\bm x) 
e^{-2\pi i \bm x\cdot \bm \zeta}d\bm x.
\end{equation}
However, the function $\rho(\bm y)$ defined in \eqref{RR} is, in 
general, not a PDF. 

\section{The Mori-Zwanzig formulation of deep learning}
\label{sec:MZequations}
In section \ref{sec:compositiontransfer} we defined  
two linear operators, i.e., $\N(n,q)$ and 
$\M(n,q)$ in equations \eqref{N} and \eqref{M}, 
mapping the probability density of the state $\bm X_n$ 
and the conditional expectation of a phase space function 
$\bm u(\bm X_n)$ forward or backward across 
different layers of the neural network. In 
particular, we have shown that
\begin{align}
p_{n+1}(\bm x)&=\N(n+1,n) p_n(\bm x), 
\label{50}\\
\mathbb{E}\{\bm u(\bm X_{L})|\bm  X_{n}=\bm x\}&=
\M(n,n+1)\mathbb{E}\{\bm u(\bm X_{L})|\bm X_{n+1}=
\bm x\}
\label{51}.
\end{align}
Equation \eqref{50} maps the PDF of the 
state $\bm X_j$ forward through the neural network, i.e., 
from the input to the output as $n$ increases, 
while \eqref{51} maps the conditional expectation 
backward. We have also shown in 
section \ref{sec:compositionOperator} that upon definition of
\begin{equation}
\bm q_n(\bm x)=
\mathbb{E}\{\bm u(\bm X_{L})|\bm  X_{L-n}=\bm x\}
\label{qn}
\end{equation}
we can rewrite \eqref{51} as a forward 
propagation problem, i.e., 
\begin{equation}
\bm q_{n+1}(\bm x) = \G(n+1,n)\bm q_n(\bm x),
\label{62}
\end{equation}
where $\G(n,q)=\M(L-n,L-q)$ 
and $\M$ is defined in \eqref{M}. The function
$\bm q_n(\bm x)$ is defined on the domain 
\begin{equation}
\mathscr{R}(\bm X_{L-n})=\{\bm X_{L-n}(\omega)\in 
\mathbb{R}^N\,:\, \omega\in\mathcal{S}\}, 
\end{equation}
i.e., on the range of the random variable 
$\bm X_{L-n}(\omega)$ (see Definition \eqref{Rn+1}).
$\mathscr{R}(\bm X_{L-n})$ is a deterministic subset of 
$\mathbb{R}^N$.

Eqs. \eqref{51} constitute the basis for developing the Mori-Zwanzig (MZ) 
formulation of deep neural networks. The MZ 
formulation is a technique originally developed in statistical mechanics  \cite{Mori,zwanzig1961memory}
to formally integrate out phase variables in 
nonlinear dynamical systems by means of a 
projection operator. One of the main features 
of such formulation is that it allows us to 
systematically derive exact equations for 
quantities of interest, e.g., low-dimensional observables, 
based on the equations of motion of the full system. In the context 
of deep neural networks such equations of motion 
are Eqs. \eqref{50}-\eqref{51},
and \eqref{62}.

To develop the Mori-Zwanzig formulation of 
deep learning, we introduce a 
{\em layer-dependent orthogonal projection operator} 
$\P_n$ together with the complementary projection 
$\Q_n=\I-\P_n$. The nature and properties of $\P_n$ 
will be discussed in detail in 
section \ref{sec:Projections_1}. For now,
it suffices to assume only that $\P_n$ is 
a self-adjoint bounded linear operator, and 
that $\P_n^2= \P_n$, i.e.,  $\P_n$ is
idempotent.  
To derive the MZ equation for neural networks, let us
consider a general recursion, 
\begin{equation}
\bm g_{n+1}(\bm x)= \R(n+1,n) \bm g_n(\bm x),
\label{genrec}
\end{equation}
where $\{\bm g_{n},\R(n+1,n)\}$ can be either 
$\{p_n,\N(n+1,n)\}$ or $\{\bm q_n,\G(n+1,n)\}$, 
depending on the context of the application.


\subsection{The projection-first and propagation-first approaches}
We apply the projection operators $\P_n$ and $\Q_n$ 
to \eqref{genrec} to obtain the following coupled 
system of equations 
\begin{align}
\bm g_{n+1} =& \R(n+1,n) \P_n \bm g_n + \R(n+1,n) \Q_n \bm g_n,
\label{rel}\\
\Q_{n+1} \bm g_{n+1} =& \Q_{n+1} \R(n+1,n) \P_n \bm g_n +  
\Q_{n+1} \R(n+1,n) \Q_n \bm g_n.
\label{irr}
\end{align}
By iterating the difference equation 
\eqref{irr}, we obtain the following 
formula\footnote{Note that the difference 
equation \eqref{irr} can be written as
\begin{equation}
\bm h_{n+1} = \A_n \bm h_n + \bm c_n, 
\label{irrr}
\end{equation}
where $\bm h_{n}= \Q_{n}\bm g_n$, 
$\bm c_n=\Q_{n+1}\R(n+1,n)\P_n\bm g_n$, and 
$\A_n= \Q_{n+1}\R(n+1,n)$. As is 
well-known, the solution to \eqref{irrr} is 
\begin{equation}
\bm h_{n} = \prod_{k=0}^{n-1}\A_k \bm h_0 + 
\sum_{j=0}^{n-1}\Phi(n,j+1)\bm c_{j}, 
\quad \text{where}\quad \Phi(n,m)=\A_n\cdots \A_m.
\label{JR}
\end{equation}
A substitution of $\A_n$, $\bm h_n$ and $\bm c_j$ into 
\eqref{JR} yields \eqref{orthDyn}.} 
for $\Q_n \bm g_{n}$
\begin{equation}
\Q_n \bm g_n = \Phi_\R(n,0)\Q_0 \bm g_0+
\sum_{m=0}^{n-1}\Phi_\R(n,m)\P_m \bm g_m,
\label{orthDyn}
\end{equation}
where $\Phi_\R(n,m)$ is the (forward) propagator of 
the orthogonal dynamics, i.e.,  
\begin{equation}
\Phi_\R(n,m)= \Q_n\R(n,n-1)\cdots \Q_{m+1} \R(m+1,m).
\label{Phi}
\end{equation}
{\color{r}Since $\bm g_n=\mathcal{R}(n,0)\bm g_0$, and $\bm g_0$ is arbitrary, we have that  \eqref{orthDyn} implies the operator identity}
\begin{equation}\label{eq: op-id}
    \Q_n \R(n,0)=  \Phi_\R(n,0)\Q_0 + \sum_{m=0}^{n-1}\Phi_\R(n,m)\P_m  \R(m,0).
\end{equation}
A substitution of \eqref{orthDyn} into \eqref{rel} yields 
the Mori-Zwanzig equation 
\begin{equation}
\bm g_{n+1} = \underbrace{\R(n+1,n) \P_n \bm g_n}_{\text{streaming term}} + \underbrace{\R(n+1,n)\Phi_\R(n,0)\Q_0 \bm g_0}_{\text{noise term}} + 
\underbrace{\R(n+1,n) \sum_{m=0}^{n-1}\Phi_\R(n,m)\P_m \bm g_m}_{\text{memory term}}.
\label{MZDISCRETE}
\end{equation} 
We shall call the first term at the right hand side of 
\eqref{MZDISCRETE} {streaming} (or Markovian) term, 
{\color{r}in agreement with classical literature on MZ equations}.  
The streaming term represents the change in $\P_n\bm g_n$ 
as we go from one layer to the next. {\color{r} The second term 
is known as ``noise term'' in classical statistical 
mechanics. The reason for this definition is 
that $\Phi_\R(n,0)\Q_0 \bm g_0$ represents the 
effects of the dynamics generated by $\Q_{m}\mathcal{R}(m,m-1)$, 
which is usually under-resolved in classical particle systems and 
therefore modeled as random noise. Such noise, however, is 
very different from the random noise 
$\{\bm \xi_0,\ldots, \bm \xi_{L-1}\}$ we introduced 
into the neural network model \eqref{discrete_dyn}.}
The third term represents the {\em memory of the neural network}, and it 
encodes the interaction between the projected 
dynamics and its entire history.  

Note that if $\bm g_0$ is in the range of $\P_0$, i.e., 
if $\P_0 \bm g_0=\bm g_0$, then the second term drops out, 
yielding a simplified MZ equation,
\begin{equation}
\bm g_{n+1} = \R(n+1,n) \P_n \bm g_n + 
\R(n+1,n) \sum_{m=0}^{n-1}\Phi_\R(n,m)\P_m \bm g_m.
\label{MZDISCRETE1}
\end{equation} 
To integrate \eqref{MZDISCRETE1} forward, i.e., from one layer to the next, 
we first project $\bm g_m$ using $\P_m$ (for $m=0,\ldots,n$), then apply the 
evolution operator $\R(n+1,n)$ to $\P_n\bm g_n$, and the memory operator $\Phi_\R$ 
to the entire history of $\bm g_m$ (memory of the network). 
It is also possible to construct an MZ equation based on 
the reversed mechanism, i.e., by projecting 
$\R(n+1,n)\bm g_n$ rather than $\bm g_m$. 
To this end,  re-write \eqref{rel} as  
\begin{equation}
\bm g_{n+1} =\P_{n+1}  \R(n+1,n)  \bm g_n +  \Q_{n+1} \R(n+1,n) \bm g_n,
\label{vac}
\end{equation}
i.e., the propagation via $\R(n+1,n)$ precedes 
projection (propagation-first approach). By applying the variation of constant formula \eqref{eq: op-id} to \eqref{vac} 
we arrive at a slightly different (though completely equivalent)
form of the MZ equation, namely
\begin{equation}
    \bm g_{n+1} =  \P_{n+1} \R(n+1,n)  \bm g_{m} +  \Phi_\R(n+1,0) \bm g_0 + 
     \sum_{m=0}^{n-1}\Phi_\R(n+1,m+1)\P_{m+1} \R(m+1,m)  \bm g_{m}. 
\end{equation}

\subsection{Discrete Dyson's identity}
Another form of the MZ equation \eqref{MZDISCRETE} 
can be derived based on a discrete version of the Dyson 
identity\footnote{For continuous-time 
autonomous dynamical systems the Dyson's identity 
can be written as  \cite{zhu2018estimation,zhu2018faber,
zhu2020,VenturiBook,HeyrimPRS_2014}
\begin{equation}
e^{t\L}= e^{t\Q\L}+ \int_0^{t} e^{(t-s)\L}\P\L e^{s\Q\L} ds.
\label{Dyson}
\end{equation}
where $\L$ is the (time-independent) Liouvillian of 
the system. {\color{r}The discrete Dyson identity and the corresponding 
discrete MZ formulation was first derived by Dave {\em et. al} 
in \cite{Darve_PNAS_2009}, and later revisited by 
Lin and Lu in \cite{Lin2020}. Both these derivations are for 
autonomous (time-invariant) discrete dynamical systems, 
while our derivations also apply to non-autonomous systems, such as those
generated by neural networks.}}. 
To derive this identity, consider the sequence 
\begin{align}
\bm y_{n+1} = &\Q_{n+1}\R(n+1,n) \bm y_n\label{AA}\\
 = &\R(n+1,n) \bm y_n - \P_{n+1} \R(n+1,n) \bm y_n.\label{BB}
\end{align}
By using the discrete variation of constant formula, 
we can rewrite \eqref{BB} as 
\begin{equation}
\bm y_n = \R(n,0)\bm y_0 - 
\sum_{m=0}^{n-1}\R(n,m+1)\P_{m+1} \R(m+1,m)\bm y_m.
\label{CC}
\end{equation}
Similarly, solving \eqref{AA} yields 
\begin{equation}
\bm y_n = \Phi_\R(n,0)\bm y_0,
\label{DD}
\end{equation}
where $\Phi_\R$ is defined in \eqref{Phi}.
By substituting \eqref{DD} into \eqref{CC} for both $\bm y_n$ and $\bm y_m$, and observing that $\bm y_0$ is arbitrary,   we obtain 
\begin{equation}
\R(n,0) = \Phi_\R(n,0) +
\sum_{m=0}^{n-1}\R(n,m+1)\P_{m+1} \R(m+1,m) \Phi_\R(m,0).
\label{discreteDyson1}
\end{equation}
The operator identity \eqref{discreteDyson1} 
is the discrete version of the 
well-known  continuous-time Dyson's identity.
A substitution of \eqref{discreteDyson1} into 
$\bm g_{n} = \R(n,0)\bm g_{0}$
yields the following form of the MZ equation \eqref{MZDISCRETE}
\begin{equation}
\bm g_{n+1} =\P_{n+1} \R(n+1,n)\Phi_\R(n,0) \bm g_{0} +  \Phi_\R(n+1,0) \bm g_{0} +
\sum_{m=0}^{n-1}\R(n+1,m+1)\P_{m+1} \R(m+1,m)
 \Phi_\R(m,0)\bm g_{0}.
 \label{MZNET}
\end{equation}
Here we have arranged the terms in the same way as in \eqref{MZDISCRETE}.

\subsection{Mori-Zwanzig equations for probability density functions}
\label{sec:MZequations }
We have seen that the PDF of the random vector 
$\bm X_n$ can be mapped forward and backward 
through the neural network via the transfer 
operator $\N(q,n)$ in \eqref{N}. 
Replacing $\R$ with $\N$ in \eqref{MZDISCRETE} 
 yields the following Mori-Zwanzig equation for the 
PDF of $\bm X_n$
\begin{equation}
p_{n+1} = \underbrace{ \N(n+1,n) 
\P_n p_n}_{\text{streaming term}} + 
\underbrace{\N(n+1,n)
\Phi_{\N}(n,0)\Q_0 p_0}_{\text{noise term}} + 
\underbrace{\N(n+1,n) \sum_{m=0}^{n-1}
\Phi_{\N}(n,m)\P_m p_m}_{\text{memory term}},
\label{MZDISCRETEPDF}
\end{equation} 
Alternatively, by using the MZ equation \eqref{MZNET}, we can write 
\begin{equation}
p_{n} = \Phi_{\N}(n,0) p_{0} +
\sum_{m=0}^{n-1}\N(n,m+1)\P_{m+1} \N(m+1,m)
 \Phi_{\N}(m,0) p_{0},
 \label{MZDISCRETE2PDF}
\end{equation}
where 
\begin{equation}\label{eq: phiN}
\Phi_{\N}(n,m)= \Q_n\N(n,n-1)\cdots \Q_{m+1}\N(m+1,m).
\end{equation}

\subsection{Mori-Zwanzig equation for conditional expectations}
\label{sec:MZcondexp}
Next, we discuss MZ equations in 
neural nets propagating conditional expectations
\begin{equation}
\bm q_n(\bm x) = \mathbb{E}\{\bm u(\bm X_L)|
\bm X_{L-n}=\bm x\}
\end{equation} 
backward across the network, i.e., from 
$\bm q_0(\bm x)=\bm u(\bm x)$ into 
$\bm q_{L}(\bm x) = \mathbb{E}\{\bm u(\bm X_L)|
\bm X_{0}=\bm x\}$. 
To simplify the notation, 
we denote the projection operators in the space 
of conditional expectations with the same 
letters as in the space of PDFs, i.e., 
$\P_n$ and $\Q_n$\footnote{ \label{foot:9}The 
orthogonal projection for conditional 
expectations is the operator adjoint of 
the projection $\P_m$ that operates on 
probability densities, i.e., 
\begin{equation}
\int \mathbb{E}\{\bm u(\bm X_k)|\bm X_q=\bm x\}
\P_q p_q(\bm x)d\bm x = \int \P_q^*
\mathbb{E}\{\bm u(\bm X_k)|\bm X_q=\bm x\}  
p_q(\bm x)d\bm x.
\end{equation} 
Such adjoint relation is the same that connects 
the composition and transfer operators ($\M(q,n)$ 
and $\N(n,q)$ in equation \eqref{adjointrelation}). 
The connection between projections for 
probability densities and conditional expectations 
was extensively discussed in \cite{dominy2017duality} 
in the setting of operator algebras.}. 
Replacing 
$\R$ with $\G$ in \eqref{MZDISCRETE} 
yields the following MZ equation for 
the conditional expectations
\begin{equation}
\bm q_{n+1} = 
\underbrace{\G(n+1,n) \P_n \bm q_n}_{\text{streaming term}} + 
\underbrace{\G(n+1,n)\Phi_{\G}(n,0)
\Q_0 \bm q_0}_{\text{noise term}} + 
\underbrace{\G(n+1,n) \sum_{m=0}^{n-1}
\Phi_{\G}(n,m)\P_m \bm q_m}_{\text{memory term}},
\label{MZDISCRETEEXP}
\end{equation} 
where 
\begin{equation}
\Phi_{\G}(n,m)= \Q_n\G(n,n-1)\cdots \Q_{m+1}\G(m+1,m).
\label{PHIG}
\end{equation}
Equation \eqref{MZDISCRETEEXP} can be equivalently 
written by incorporating the streaming term into 
the summation of the memory term
\begin{equation}
\bm q_{n+1} =  
\G(n+1,n)\Phi_{\G}(n,0)
\Q_0 \bm q_0+ 
\G(n+1,n) \sum_{m=0}^{n}\Phi_{\G}(n,m)\P_m \bm q_m.
\label{MZDISCRETEEXP_bis}
\end{equation} 
Alternatively, by using Eq. \eqref{MZNET} we obtain
\begin{equation}
\bm q_{n} = \Phi_\G(n,0) \bm q_{0} +
\sum_{m=0}^{n-1}\G(n,m+1)\P_{m+1} \G(m+1,m)
\Phi_\G(m,0)\bm q_{0}.
\label{MZDISCRETE2EXP}
\end{equation}

\vspace{0.5cm}
\noindent
{\bf Remark:} The Mori-Zwanzig 
equations \eqref{MZDISCRETEPDF}-\eqref{MZDISCRETE2PDF}
and  \eqref{MZDISCRETEEXP}-\eqref{MZDISCRETE2EXP} 
allow us to perform dimensional reduction within each 
layer of the network (number of neurons per layer, via projection), 
or across different layers (total number of layers, via memory approximation). 
The MZ formulation is also useful to perform 
theoretical analysis of deep learning by 
using tools from operator theory. As we shall 
see in section \ref{sec:fadingMemory}, the memory of 
the neural network can be controlled by controlling the noise 
process $\{\bm \xi_0,\bm \xi_1,\ldots,\bm \xi_{L-1}\}$.

\section{Mori-Zwanzig projection operator}
\label{sec:Projections_1}
Suppose that the neural network 
model \eqref{discrete_sys_10} is 
perturbed by independent random variables 
$\{\bm \xi_n\}$ with bounded range $\mathscr{R}(\bm \xi_n)$. 
{\color{r}In this hypothesis, the range of each 
random vector $\bm X_m$, i.e. $\mathscr{R}(\bm X_m)$, 
is bounded. In fact,  
\begin{equation}
\mathscr{R}(\bm X_{m})\subseteq \Omega_{m}=\{\bm c\in \mathbb{R}^N:\, 
\bm c=\bm a+\bm b \quad \bm a\in[-1,1]^N,\quad  
\bm b\in \mathscr{R}(\bm \xi_{m-1})\}, 
\label{OMEGAm1}
\end{equation}
and $\Omega_{m}$ is clearly a bounded set  if 
$\mathscr{R}(\bm \xi_{m-1})$ is bounded.} 
With specific reference 
to MZ equations for {\em scalar} conditional 
expectations (i.e., conditional averages of scalar 
quantities of interest) 
\begin{equation} 
q_m(\bm x)=\mathbb{E}\left[u(\bm X_L)|\bm X_{L-m}=\bm x\right],
\label{qmm}
\end{equation}
and recalling that
\begin{align}
q_m(\bm x) =& \G(m,m-1) q_{m-1}\nonumber\\
=&\M(L-m,L-m+1)q_{m-1}(\bm x) \nonumber\\
=& \int_{\mathscr{R}(\bm X_{L-m+1})} \underbrace{p_{L-m+1|L-m}(\bm y,\bm x)}_{\rho_{L-m}(\bm y-\bm F_{L-m}\left(\bm x,\bm w_{L-m})\right)}q_{m-1}(\bm y) d\bm y,
\end{align}
we define the following orthogonal projection 
operator\footnote{The projection operator \eqref{kernelProjection} 
can be extended to vector-valued functions and conditional 
expectations by defining an appropriate matrix-valued kernel 
$\bm  K(\bm x,\bm y)$.} on $L^2(\mathscr{R}(\bm X_{L-m}))$

\begin{align}
\P_m: L^2(\mathscr{R}(\bm X_{L-m}))&\mapsto 
L^2(\mathscr{R}(\bm X_{L-m}))\nonumber\\
f &\mapsto \P_m f=\int_{\mathscr{R}(\bm X_{L-m})} 
K_{L-m}(\bm x,\bm y)f(\bm y) d\bm y.
\label{kernelProjection}
\end{align}
Since $\P_m$ is, by definition, an orthogonal 
projection we have that $\P_m$ is idempotent 
($\P_m^2=\P_m$), bounded, and self-adjoint 
relative to the inner product in 
{$L^2\left(\mathscr{R}(\bm X_{L-m})\right)$}.
These conditions imply that the kernel 
$K_{L-m}(\bm x,\bm y)$ is a symmetric 
Hilbert-Schmidt kernel that 
satisfies the {\em reproducing kernel condition}

\begin{equation}
\int_{\mathscr{R}(\bm X_{L-m})} K_{L-m}(\bm x,\bm y)
K_{L-m}(\bm y,\bm z) d\bm y = K_{L-m}(\bm x,\bm z), \qquad \forall \bm x,\bm z\in \mathscr{R}(\bm X_{L-m}).
\label{c2}
\end{equation}
Note that the classical Mori's projection 
\cite{zhu2018faber,zhu2018estimation} can be 
written in the form \eqref{kernelProjection} if we set 
\begin{equation}
K_{L-m}(\bm x,\bm y)= \sum_{i=0}^M \eta^m_i(\bm x)\eta^m_i(\bm y),
\label{MoriPP}
\end{equation}
where $\{\eta^m_0,\ldots,\eta^m_M\}$ are orthonormal 
functions in {$L^2\left(\mathscr{R}(\bm X_{L-m})\right)$}. 
Since the range of $\bm X_{L-m}$ can vary from layer to layer we have 
that the set of orthonormal functions $\{\eta^m_j(\bm x)\}$ also depends 
on the layer (hence the label ``$m$''). 
The projection operator $\P_m$ is said 
to be {\em non-negative} if 
for all positive functions $v(\bm x)\in 
L_{\mu_m}^2(\mathscr{R}(\bm X_{L-m}))$ ($v>0$) 
we have that $\P_m v\geq 0$ \cite{Pinkus1983}. 
Clearly, this implies that the kernel 
$K_{L-m}(\bm x,\bm y)$ is non-negative in  
$\mathscr{R}(\bm X_{L-m})\times \mathscr{R}(\bm X_{L-m})$ \cite{Gilbert1988}. 
An example of a kernel defining a non-negative 
orthogonal projection is 
\begin{equation}\label{eq: mori1}
K_{L-m}(\bm x,\bm y) =  \eta^m(\bm x)\eta^m(\bm y), \qquad 
\eta^m(\bm x)\geq 0, \qquad 
\left\|\eta^m\right\|_{L_{\mu_m}^2(\mathscr{R}(\bm X_{L-m}))}=1.
\end{equation}
More generally, if $K_{L-m}(\bm x,\bm y)$ is any 
square-integrable symmetric conditional 
probability density function on 
$\mathscr{R}(\bm X_{L-m})\times 
\mathscr{R}(\bm X_{L-m})$, 
then $\P_m$ is a non-negative projection.

{\color{r}
\section{Analysis of the MZ equation}
\label{sec:fadingMemory}}

{\color{r}We now turn to the theoretical analysis of 
the MZ equation. In particular, we study the MZ equation 
for conditional expectations discussed in 
section \ref{sec:MZcondexp}, i.e., equation \eqref{MZDISCRETEEXP}. 
Clearly, the operator $\Q_{m}\G(m,m-1)$ plays a very important role 
in such an equation via the memory operator $\Phi_\G$ defined in \eqref{PHIG}. Indeed, 
$\Phi_\G$ appears in both the memory term and the noise term, and is 
defined by operator products involving $\Q_{m}\G(m,m-1)$.}

In this section, we aim at determining conditions 
on  $Q_{m}\G(m,m-1)=(\I-\P_{m})\G(m,m-1)$, 
e.g., noise level and distribution, such that 
\begin{equation}
\left\|\Q_{m}\G(m,m-1)\right\|=
\sup_{\substack{v\in L^2(\mathscr{R}(\bm X_{L-m+1}))\\ 
v\neq 0}} 
\frac{\left\|\Q_{m}\G(m,m-1) v 
\right\|_{L^2(\mathscr{R}(\bm X_{L-m}))}}
{\left\| v \right\|_{L^2(\mathscr{R}(\bm X_{L-m+1}))}} 
< 1.
\label{contraction}
\end{equation}
In this way, the operator $\Q_{m}\G(m,m-1)$ becomes a contraction,
and therefore the MZ memory term in \eqref{MZDISCRETEEXP} 
decays with the number of layers, while the noise 
term decays zero. {\color{r}Indeed, if \eqref{contraction} 
holds true, then the norm of memory operator $\Phi_\G(n,m)$ defined 
in \eqref{PHIG} (similar in \eqref{MZDISCRETEEXP_bis} and \eqref{MZDISCRETE2EXP} ) decays with the number of 
``$\Q\G$'' operator products taken, i.e.,  with the number of layers.}

\subsection{\color{r} Deterministic neural networks}

{\color{r} Before turning to the theoretical analysis of 
the operator $\Q_{m}\G(m,m-1)$, it is convenient to dwell 
on the case where the neural network is deterministic 
(no random perturbations), and has $\tanh()$ activation 
functions. This case is quite common in practical 
applications, and also allows for significant simplifications 
of the MZ framework.  
First of all, in the absence of noise the 
output of each neural network layer has the same range, i.e.,
\begin{equation}
\mathscr{R}\left(\bm X_n\right)=[-1,1]^N \qquad n=1,\ldots,L, 
\end{equation}
where $N$ is the number of neurons, assumed to be constant 
for each layer. Hence, we can choose a projection operator \eqref{kernelProjection} 
that does not depend on the particular layer. For simplicity, 
we consider  
\begin{equation}
\P f=\int_{[-1,1]^N}  K(\bm x,\bm y) f(\bm y)d\bm y,
\label{1Dproj}
\end{equation}
where 
\begin{equation}
K(\bm x,\bm y) = \eta_0 + \sum_{k=1}^M \eta_k(\bm x) \eta_k(\bm y).
\label{kernelDet}
\end{equation}
Here $\{\eta_0,\ldots,\eta_M\}$  are orthonormal 
functions in $L^2\left([-1,1]^N\right)$, e.g., normalized multivariate 
Legendre polynomials \cite{Wang2020}. We sort $\{\eta_k\}$ based 
on degree lexicographic order. In this way, the first $N+1$ orthonormal functions in \eqref{kernelDet} are explicitly written as
\begin{equation}
\eta_0=2^{-N/2}, \qquad \eta_k(\bm x) = 2^{-N/2}\sqrt{3} x_k \qquad k=1,\ldots, N.
\label{orthB}  
\end{equation}  
Moreover, if the neural network has linear output we have 
$q_0(\bm x)=\bm \alpha \cdot \bm x$ and therefore
\begin{equation}
\P q_0 = q_0\qquad \Q q_0=(\I-\P)q_0=0.
\label{Qq0}
\end{equation}
This implies that the noise term in the MZ 
equation \eqref{MZDISCRETEEXP} is zero for the 
projection kernel \eqref{kernelDet}-\eqref{orthB} and 
networks with linear output. 
 
To study the MZ memory term we consider a simple example 
involving a two-layer deterministic neural net 
with $d$-dimensional input $\bm x\in\Omega_0\subseteq \mathbb{R}^d$ 
and scalar output $q_2(\bm x)$. The MZ 
equation \eqref{MZDISCRETEEXP} with projection 
operator \eqref{1Dproj}-\eqref{orthB} can be 
written as
\begin{align}
q_{2}(\bm x) = \G(2,1) \P q_1+ \underbrace{\G(2,1)\left[q_1-\P q_1\right]}_{\text{memory term}},\qquad \bm x\in \Omega_0.
\label{MZnet2}
\end{align}
Clearly, if $q_1$ is approximately in the range 
of $\P$ (i.e., if $q_1\simeq \P q_1$) then 
the neural network is essentially {\em memoryless} 
(the memory term in \eqref{MZnet2} drops out).
The next question is whether the nonlinear function 
$q_1$ can indeed be approximated accurately 
by $\P q_1$. This is a well-established result 
in multivariate polynomial approximation theory. In particular, 
it can be shown that $\P q_1$ converges exponentially 
fast to $q_1$as we increase the polynomial degree 
in the multivariate Legendre expansion 
(i.e., as we increase $M$ in \eqref{kernelDet}\footnote{{\color{r}From 
an approximation theory viewpoint, the number of 
basis functions $M$ in \eqref{kernelDet} should be 
defined as the radius of an $\ell^q$ ball index set 
in $\mathbb{N}^N_0$ (see \cite[\S 4.2]{Wang2020} 
and \cite{Trefethen2017}).}}).
Exponential convergence follows immediately from the 
fact that the function 
\begin{align}
q_1(\bm x) &=\G(1,0)q_0 \nonumber\\
& =\bm \alpha \cdot \tanh(\bm W_{1} \bm x+\bm b_{1}),\qquad \bm x\in [-1,1]^N
\label{nonlinq1}
\end{align}
admits an analytical extension on a Bernstein poly-ellipse 
enclosing $[-1,1]^N$ (see \cite{Wang2020} for details).
The projection of the nonlinear function 
$q_1(\bm x)$ onto the linear space spanned by the $N+1$ orthonormal 
basis functions \eqref{orthB} (i.e., the space of affine functions 
defined on $[-1,1]^N$) can be written as
\begin{equation}
\P q_1= \beta_0 + \bm \beta \cdot \bm x,
\label{projectedq1}
\end{equation}
where the coefficients $\{\beta_0,\ldots,\beta_N\}$ are given by
\begin{equation}
\beta_0 = \frac{1}{2^{N}} \int_{[-1,1]^N} q_1(\bm x) d\bm x, \qquad 
\beta_j = \frac{3}{2^{N}} \int_{[-1,1]^N} q_1(\bm x) x_j d\bm x \qquad j=1,\ldots,N.
\label{betak}
\end{equation} 
Hence, if $q_1$ is approximately in the 
range of $\P$ (i.e., $\P q_1 \simeq q_1$), then we can 
explicitly write the MZ equation \eqref{MZnet2} as
\begin{equation}
q_2(\bm x) \simeq \beta_0 + \bm \beta\cdot \tanh(\bm W_0 \bm x+\bm b_0),\qquad \bm x\in\Omega_0.
\label{NNapprox}
\end{equation}
Note that this reduces the total number of degrees 
of freedom of the two-layer neural network 
from $N(N+d+3)$ to $N(d+2)+1$, under the condition 
that $q_1$ in equation \eqref{nonlinq1} can be accurately 
approximated by the hyperplane $\P q_1$ in equation \eqref{projectedq1}. 
This depends of course on the weights $\bm W_{1}$ and 
biases $\bm b_{1}$ in \eqref{nonlinq1}. In particular, if the entries 
of the weight matrix $\bm W_{1}$ are sufficiently small 
then by using Taylor series it is immediate to prove that  
$\P q_1\simeq q_1$.

\vs
\noindent
{\em An example:} In Figure \ref{fig:MZplot} we compare 
the MZ streaming and memory terms for the  
two-layer deterministic neural network we studied in section \ref{sec:trainingParadigms} and the target function \eqref{f1d}. 
Here we consider $N=20$ neurons, and approximate the 
integrals in \eqref{betak} using Monte Carlo quadrature. Clearly, 
it is possible to constrain the norm of the weight matrix $\bm W_1$ 
during training so  that the nonlinear function $q_1$ in \eqref{nonlinq1} is
approximated well by the affine function $\P q_1$ in \eqref{projectedq1}. 
This essentially allows us to control the approximation 
error $\left\|q_1-\P q_1\right\|_{L^2([-1,1]^N)}$ and therefore 
the {the amplitude of the MZ memory term} in \eqref{MZnet2}. 
For this particular example, we set 
$\left\|\bm W_1\right\|_{\infty}\leq 0.1$, which yields the 
following contraction factor 
\begin{equation}
\frac{\left\|\Q \G(2,1) q_0\right\|_{L^2([-1,1]^N)}}{\left\|q_0\right\|_{L^2([-1,1]^N)}}=\frac{\left\|q_1-\P q_1\right\|_{L^2([-1,1]^N)}}{\left\|q_0\right\|_{L^2([-1,1]^N)}}  = 7.5\times 10^{-4}.
\label{contindex}
\end{equation}
Note that \eqref{contindex} is not the operator norm 
of $\Q \G(2,1)$ we defined in \eqref{contraction}. 
In fact, the operator norm requires computing the supremum of $\|\Q \G(2,1) v\|_{L^2([-1,1])}/\|v\|_{L^2([-1,1])}$ over all nonzero
functions $v\in L^2([-1,1]^N)$, not just the linear function $v=q_0$
If training-over-weights of deterministic nets 
is done in a fully unconstrained optimization setting then 
there is no guarantee that the MZ memory term is small. 
\begin{figure}
\centerline{\includegraphics[height=6.5cm]{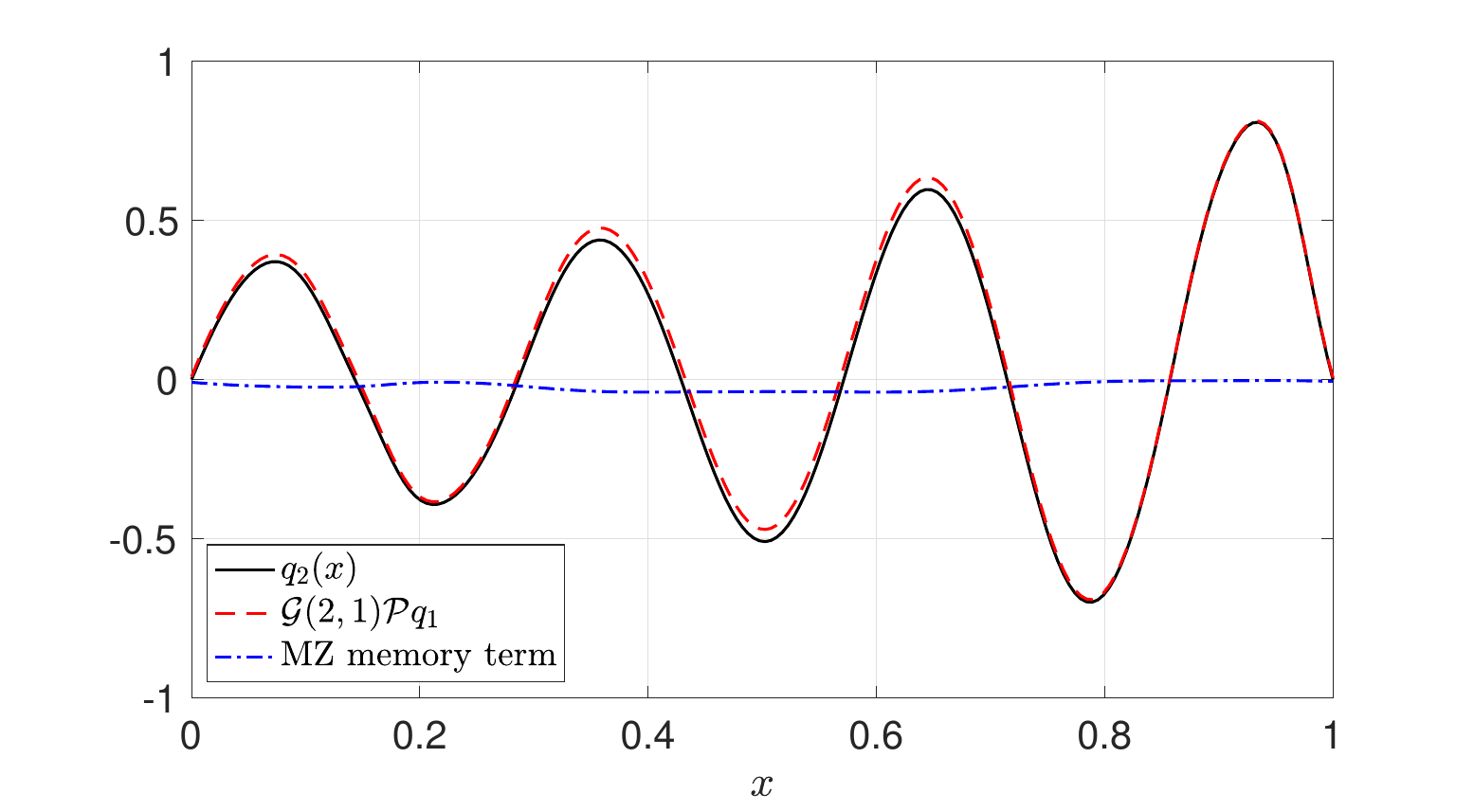}}
\caption{\color{r}Comparison between the MZ streaming and memory terms 
for the two-layer deterministic neural network we studied in section \ref{sec:trainingParadigms} and the target function \eqref{f1d}. Here 
we consider $N=20$ neurons, and approximate the 
high-dimensional integrals in \eqref{betak} by using Monte 
Carlo quadrature. The neural network is trained by 
constraining the entries of the weight matrix $\bm W_1$ 
as $\left\|\bm W_1\right\|_{\infty}\leq 0.1$. This 
allows us to control the approximation 
error $\left\|q_1-\P q_1\right\|_{L^2([-1,1]^N)}$ when projecting the 
nonlinear function \eqref{nonlinq1} onto the space 
of affine functions \eqref{projectedq1} which, in turn, 
controls the magnitude of the MZ memory term.}
\label{fig:MZplot}
\end{figure}

\vs
\noindent
The discussion about the approximation of the MZ memory 
term can be extended  to deterministic neural networks 
with an increasing number of layers. For example, 
the output of a three-layer deterministic neural 
network can be written as 
\begin{align}
q_3(\bm x) = &\G(3,2)\P q_2 + \G(3,2)\Q \G(2,1)\P q_1 + 
\G(3,2)\Q\G(2,1)\Q\G(1,0)\P q_0 \nonumber\\
 = &\G(3,2)\P q_2 + \underbrace{\G(3,2)\left[\I- 
 \P \right]\G(2,1)\P q_1+
 \G(3,2)\left[ \I-\P \right]\G(2,1)
 \left[q_1-\P q_1\right]}_{\text{memory term}}.
 \label{130}
\end{align}
Note that if $\P q_1$ is a linear function of the form \eqref{projectedq1}, 
then the term $\G(3,2)\left[\I- \P \right]\G(2,1)\P q_1$ has exactly 
the same functional form as the MZ memory term 
$\G(2,1)\left[\I- \P \right]\G(1,0)\P q_0 = \G(2,1)[q_1-\P q_1]$. 
Hence, everything we said about the accuracy of 
a linear approximation of $\bm \alpha\cdot \tanh(\bm W_{1}\bm x
+\bm b_{1})$ can be directly applied now to 
$\G(2,1)\P q_1=\bm \beta\cdot \tanh(\bm W_{2}\bm x+\bm b_{2})$.

On the other hand, if $q_1$ can be approximated 
with accuracy by the linear function $\P q_1$, 
then the term $\G(2,1)[q_1-\P q_1]$ 
is likely to be small. This implies that the last 
term in \eqref{130} is likely to be small as well 
(bounded operator $\G(3,2)$ applied to the difference 
between two small functions). In other words, if 
the weights and biases of the network are such that 
$q_1(\bm x)=\bm \alpha \cdot \tanh(\bm W_n\bm x+ b_n)$ can be 
approximated with accuracy by the linear 
function \eqref{projectedq1} then the MZ memory 
term of the three-layer network is small. 

More generally, by using error estimates for 
multivariate polynomial approximation of analytic functions 
\cite{Wang2020}, it is possible to derive 
an upper bound for the operator norm of $\Q \G(m,m-1)$ in 
\eqref{contraction}.
Such a bound is rather involved, but in principle it 
allows us to determine conditions on the weights 
and biases of the neural network such that  
$\left\|\Q \G(m,m-1)\right\|\leq \kappa$, where 
$\kappa$ is a given constant smaller than one. 
This allows us to simplify the memory 
term in \eqref{MZDISCRETEEXP} by neglecting terms 
involving a large number of ``$\Q_m \G(m,m-1)$'' 
operator products in \eqref{PHIG}.
Hereafter, we determine general conditions for the 
operator $\Q_m \G(m,m-1)$ to be a contraction in the presence 
of random perturbations.}

\subsection{\color{r} Stochastic neural networks}

{\color{r} Consider the stochastic neural network model 
\eqref{discrete_sys_10} with $L$ layers, $N$ neurons per layer,  
and transfer functions $\bm F_n$ with range in $[-1,1]^{N}$ 
for all $n$. In this section we determine 
general conditions for the operator 
$\Q_m \G(m,m-1)$ to be a contraction (i.e., 
to satisfy the inequality \eqref{contraction}) 
{independently} of the neural network weights.}
To this end, we first write the operator $\Q_{m}\G(m,m-1)$ as 
\begin{align}
(\Q_{m}\G(m,m-1) v)(\bm x) &= \Q_{m}
\int_{\mathscr{R}(\bm X_{L-m+1})}\underbrace{\rho_{L-m}\left(\bm y-\bm F_{L-m}(\bm x,\bm w_{L-m})\right)}_{p_{L-m+1|L-m}(\bm y|\bm x)}v(\bm y)d\bm y \nonumber\\
&=\int_{\mathscr{R}(\bm X_{L-m+1})} \gamma_{L-m} (\bm y,\bm x) v(\bm y)d\bm y,
\label{QG}
\end{align}
where 
\begin{equation}
\gamma_{L-m} (\bm y,\bm x)=\rho_{L-m}(\bm y-\bm F_{L-m}(\bm x, \bm w_{L-m}))-
\int_{\mathscr{R}(\bm X_{L-m})} K_{L-m}(\bm x,\bm z)\rho_{L-m} (\bm y-\bm F_{L-m}(\bm z, \bm w_{L-m}))d\bm z.
\label{gammaM}
\end{equation}
{\color{r}The conditional density 
$p_{L-m+1|L-m}(\bm y|\bm x)=\rho_{L-m}
(\bm y-\bm F_{L-m}(\bm x, \bm w_{L-m}))$ is 
defined on the set 
\begin{equation}
\mathscr{B}_{L-m}=\{(\bm x,\bm y)\in \mathscr{R}
(\bm X_{L-m})\times \mathscr{R}(\bm X_{L-m+1}): \, 
(\bm y-\bm F_{L-m}(\bm x,\bm w_{L-m}))\in 
\mathscr{R}(\bm \xi_{L-m})\}.
\label{setB11}
\end{equation}}
As before, we assume that $K_{L-m}$ is an element of  $L^2(\mathscr{R}(\bm X_{L-m})\times \mathscr{R}(\bm X_{L-m}))$
and expand it as\footnote{As is well known, 
if $K(\bm x,\bm y)$ is a (symmetric) bounded projection kernel 
satisfying \eqref{c2} then $K$ is necessarily separable, i.e., it 
can be written in the form \eqref{Kfin}.} 
\begin{equation}
K_{L-m}(\bm x,\bm y) = c_m + \sum_{k=1}^M \eta^m_{i}(\bm x)\eta^m_i(\bm y),
\label{Kfin}
\end{equation}
where $c_m$ is a real number and $\eta^m_i$ are zero-mean 
orthonormal basis functions in $L^2(\mathscr{R}(\bm X_{L-m})$, i.e.,
\begin{equation}
\int_{\mathscr{R}(\bm X_{L-m})}\eta^m_i(\bm x)d\bm x=0, \qquad 
\int_{\mathscr{R}(\bm X_{L-m})}\eta^m_i(\bm x)\eta^m_j(\bm x)d\bm x = \delta_{ij}.
\label{zeromeanortho}
\end{equation}

\begin{lemma}
\label{lemma:c}
The kernel \eqref{Kfin} satisfies the idempotency 
requirement \eqref{c2} if and only if 
\begin{equation}
c_m=0\qquad \text{or} \qquad c_m=\frac{1}{\lambda(\mathscr{R}(\bm X_{L-m}))}, 
\end{equation}
where
$\lambda(\mathscr{R}(\bm X_{L-m}))$ is the 
Lebesgue measure of the set
$\mathscr{R}(\bm X_{L-m})$.
\end{lemma}
\begin{proof}
By substituting  \eqref{Kfin} into \eqref{c2} 
and taking into account \eqref{zeromeanortho} we obtain
\begin{equation}
c_m^2\lambda(\mathscr{R}(\bm X_{L-m}))=c_m,
\end{equation} 
from which we obtain $c_m=0$ or
$c_m=1/\lambda(\mathscr{R}(\bm X_{L-m}))$.

\end{proof}

\noindent
Clearly, if $\G(m,m-1)$ is itself a contraction and 
$\Q_{m}$ is an orthogonal projection, then the operator 
product $\Q_{m}\G(m,m-1)$ is a contraction. In the following 
Proposition, we compute a simple bound for 
the operator norm of $\Q_{m}\G(m,m-1)$.
\begin{proposition}
\label{lemma:contraction1}
{\color{r}Let $\Q_{m}$ be an orthogonal projection in 
$L^2(\bm X_{L-m})$. Suppose that the PDF of $\bm \xi_{L-m}$, 
i.e. $\rho_{L-m}$, is in $L^2(\mathscr{R}(\bm \xi_{L-m}))$. Then  
\begin{equation}
\left\|\Q_{m}\G(m,m-1)\right\|^2\leq 
\lambda(\Omega_{L-m})
\left\|\rho_{L-m}\right\|^2_{L^2(
\mathscr{R}(\bm \xi_{L-m}))}, 
\label{UBG}
\end{equation}
where $\lambda(\Omega_{L-m})$ is the Lebesgue 
measure of the set $\Omega_{L-m}$  defined in \eqref{OMEGAm1} 
and
\begin{equation}
\left\|\rho_{L-m}\right\|_{L^2(\mathscr{R}(\bm \xi_{L-m}))}^2= 
\int_{\mathscr{R}(\bm \xi_{L-m})} \rho_{L-m}(\bm x)^2d\bm x.
\label{IL-m}
\end{equation}}
In particular, if $\G(m,m-1)$ is a 
contraction then $\Q_{m} \G(m,m-1)$ is a contraction. 
\end{proposition}
\begin{proof}
The last statement in the Proposition is trivial. 
In fact, if $\Q_{m}$ is an orthogonal projection then 
its operator norm is less or equal to one. Hence,
\begin{equation}
\left\|Q_{m}\G(m,m-1)\right\|^2\leq 
\underbrace{\left\|Q_{m}\right\|^2}_{\leq 1}
\left\|\G(m,m-1)\right\|^2. 
\end{equation}
Therefore if $\G(m,m-1)$ is a contraction 
and $\Q_{m}$ is an orthogonal projection then 
$\Q_{m}\G(m,m-1)$ is a contraction.
{\color{r} We have shown in \ref{sec:functionalsetting} 
that if $\rho_{L-m}\in L^2(\mathscr{R}(\bm \xi_{L-m}))$ then 
$\G(m,m-1)$ is a bounded linear 
operator from $L^2(\mathscr{R}(\bm X_{L-m+1})$ 
to $L^2(\mathscr{R}(\bm X_{L-m})$. Moreover, 
the operator norm of $\G(m,m-1)$ can be bounded as 
(see Eq. \eqref{Kn}) 
\begin{equation}
\left\|\G(m,m-1)\right\|^2\leq \lambda(\Omega_{L-m})
\left\|\rho_{L-m}\right\|^2_{L^2(
\mathscr{R}(\bm \xi_{L-m}))}.
\label{TROT}
\end{equation}
Hence,
\begin{equation} 
\left\|\Q_m\G(m,m-1)\right\|^2\leq 
\underbrace{\left\|\Q_m\right\|^2}_{\leq 1}\lambda(\Omega_{L-m})
\left\|\rho_{L-m}\right\|^2_{L^2(\mathscr{R}(\bm \xi_{L-m}))},
\end{equation}
which completes the proof of \eqref{UBG}.
}

\end{proof}

\noindent
The upper bound in \eqref{UBG} can be 
slightly improved using the definition of 
the projection kernel $K_{L-m}$. This is stated in 
the following Theorem.

\begin{theorem}
\label{prop:contraction}
\color{r} Let $K_{L-m}$ be the projection kernel 
\eqref{Kfin} with $c_m=1/\lambda(\mathscr{R}(\bm X_{L-m}))$. 
Then the operator norm of $\Q_{m}\G(m,m-1)$ can be bounded as 
\begin{equation}
\left\|\Q_{m}\G(m,m-1)\right\|^2\leq
\lambda(\Omega_{L-m})\left(
\left\|\rho_{L-m}\right\|_{L^2(\mathscr{R}
(\bm \xi_{L-m}))}^2-
\frac{1}
{\lambda(\Omega_{L-m+1})}\right).
\label{nbound}
\end{equation}
The upper bound in \eqref{nbound}
is independent of the neural network weights. 
\end{theorem}

\begin{proof}
The function $\gamma_{L-m}(\bm y,\bm x)$ defined in \eqref{gammaM} 
is a Hilbert-Schmidt kernel. Therefore, 
\begin{equation}
\left\|\Q_{m}\G(m,m-1)\right\|^2\leq 
\left\|\gamma_{L-m} \right\|^2_{L^2(\mathscr{R}(\bm X_{L-m+1})\times \mathscr{R}(\bm X_{L-m}))}.
\label{QGm}
\end{equation}
The $L^2$ norm of $\gamma_{L-m}$ can be written as (see \eqref{gammaM})
\begin{align}
&\int_{\mathscr{R}(\bm X_{L-m+1})} \int_{\mathscr{R}(\bm X_{L-m})} \gamma_{L-m} (\bm y,\bm x)^2 d\bm y d\bm x = \nonumber\\
&\int_{\mathscr{R}(\bm X_{L-m+1})} \int_{\mathscr{R}(\bm X_{L-m})} \rho_{L-m} (\bm y-\bm F_{L-m}(\bm x,\bm w_{L-m}))^2 d\bm yd\bm x +
\nonumber \\
&\int_{\mathscr{R}(\bm X_{L-m+1})} \int_{\mathscr{R}(\bm X_{L-m})}\left(\int_{\mathscr{R}(\bm X_{L-m})} K_{L-m}(\bm x,\bm z)
\rho_{L-m} (\bm y-\bm F_{L-m}(\bm z,\bm w_{L-m}))d\bm z\right)^2d\bm yd\bm x -
\nonumber\\  
&2 \int_{\mathscr{R}(\bm X_{L-m+1})} \int_{\mathscr{R}(\bm X_{L-m})}\rho_{L-m} (\bm y-\bm F_{L-m}(\bm x,\bm w_{L-m}))\times\nonumber\\
&
\left(\int_{\mathscr{R}(\bm X_{L-m})} K_{L-m}(\bm x,\bm z)
\rho_{L-m} (\bm y-\bm F_{L-m}(\bm z,\bm w_{L-m}))d\bm z\right)
d\bm yd\bm x.
\label{re}
\end{align}
{\color{r}By using \eqref{TROT}, we can write the first term 
at the right hand side of \eqref{re} as
\begin{align}
\int_{\mathscr{R}(\bm X_{L-m+1})} \int_{\mathscr{R}(\bm X_{L-m})} \rho_{L-m} (\bm y-\bm F_{L-m}(\bm x,\bm w_{L-m}))^2 d\bm yd\bm x =&\left\|\G(m,m-1)\right\|^2\nonumber\\
\leq &\lambda(\Omega_{L-m})
\left\|\rho_{L-m}\right\|_{L^2(\mathscr{R}
(\bm \xi_{L-m}))}^2.
\label{firstterm}
\end{align}}
A substitution of the series  expansion \eqref{Kfin} into the second  
term at the right hand side of \eqref{re} yields
\begin{align}
&\int_{\mathscr{R}(\bm X_{L-m+1})} \int_{\mathscr{R}(\bm X_{L-m})}\left(\int_{\mathscr{R}(\bm X_{L-m})} K_{L-m}(\bm x,\bm z)
\rho_{L-m} (\bm y-\bm F_{L-m}(\bm z,\bm w_{L-m}))d\bm z\right)^2d\bm yd\bm x =\nonumber\\
& \frac{1}{\lambda(\mathscr{R}(\bm X_{L-m}))}
\int_{\mathscr{R}(\bm X_{L-m+1})} 
\left(\int_{\mathscr{R}(\bm X_{L-m})}\rho_{L-m}(\bm y-
\bm F_{L-m}(\bm x,\bm w_{L-m}))d\bm z\right)^2d\bm y+\nonumber\\
&\sum_{k=1}^M\int_{\mathscr{R}(\bm X_{L-m+1})}\left(\int_{\mathscr{R}(\bm X_{L-m})} 
\rho_{L-m} (\bm y-\bm F_{L-m}(\bm z,\bm w_{L-m}))
\eta^m_k(\bm z)d\bm z\right)^2d\bm y.
\label{u7}
\end{align}
Here, we used the fact that the basis 
functions $\eta^m_k(\bm x)$  
are zero-mean and orthonormal in 
$\mathscr{R}(\bm X_{L-m})$ 
(see Eq. \eqref{zeromeanortho}). 
Similarly, by substituting the expansion \eqref{Kfin} 
in the third  term at the right hand 
side of \eqref{re} we obtain 
\begin{align}
&2 \int_{\mathscr{R}(\bm X_{L-m+1})} 
\int_{\mathscr{R}(\bm X_{L-m})}
\rho_{L-m} (\bm y-\bm F_{L-m}(\bm x,\bm w_{L-m}))
\times \nonumber\\
&\left(\int_{\mathscr{R}(\bm X_{L-m})} K_{L-m}(\bm x,\bm z)
\rho_{L-m} (\bm y-\bm F_{L-m}(\bm z,\bm w_{L-m}))d\bm z\right)
d\bm yd\bm x\nonumber \\
&=\frac{2}{\lambda(\mathscr{R}(\bm X_{L-m}))} 
\int_{\mathscr{R}(\bm X_{L-m+1})}
\left( 
\int_{\mathscr{R}(\bm X_{L-m})}
\rho_{L-m} (\bm y-\bm F_{L-m}(\bm z,\bm w_{L-m}))
d\bm z\right)^2d\bm y\nonumber\\
&+2 \sum_{k=1}^M  
\int_{\mathscr{R}(\bm X_{L-m+1})}
\left( 
\int_{\mathscr{R}(\bm X_{L-m})}
\rho_{L-m} (\bm y-\bm F_{L-m}(\bm z,\bm w_{L-m}))
\eta^m_k(\bm z)d\bm z\right)^2d\bm y.
\label{u8}
\end{align}
Combining \eqref{re}-\eqref{u8} yields
\begin{align}
&\int_{\mathscr{R}(\bm X_{L-m+1})}
\int_{\mathscr{R}(\bm X_{L-m})} \gamma_{L-m} (\bm y,\bm x)^2 d\bm y d\bm x \leq 
\lambda(\Omega_{L-m})\left\|\rho_{L-m}
\right\|_{L^2(\mathscr{R}(\bm \xi_{L-m})}^2 - 
\nonumber \\
&\frac{1}{\lambda(\mathscr{R}(\bm X_{L-m}))}
\int_{\mathscr{R}(\bm X_{L-m+1})}\left( 
\int_{\mathscr{R}(\bm X_{L-m})}\rho_{L-m}(\bm y-
\bm F_{L-m}(\bm z,\bm w_{L-m}))d\bm z\right)^2d\bm y -
\nonumber \\
&\sum_{k=1}^M  
\int_{\mathscr{R}(\bm X_{L-m+1})}
\left(\int_{\mathscr{R}(\bm X_{L-m})}
\rho_{L-m} (\bm y-\bm F_{L-m}(\bm z,\bm w_{L-m}))
\eta^m_k(\bm z)d\bm z\right)^2d\bm y.
\label{re1}
\end{align}
At this point we use the Cauchy-Schwarz inequality\footnote{The 
inequality in \eqref{r11} follows from the 
Cauchy-Schwarz inequality. Specifically, let 
\begin{equation}
f(\bm y) = 
\int_{\mathscr{R}(\bm X_{L-m})} \rho_{L-m}(\bm y-\bm F(\bm z,\bm w_{L-m}))d\bm z.
\end{equation}
Then 
\begin{equation}
\underbrace{ 
\int_{\mathscr{R}(\bm X_{L-m+1})} 1^2 d\bm y}_{\lambda(\mathscr{R}(\bm X_{L-m+1}))} \int_{\mathscr{R}(\bm X_{L-m+1})} f(\bm y)^2 d\bm y\geq \left( 
\int_{\mathscr{R}(\bm X_{L-m+1})} 1\cdot f(\bm y) 
d\bm y\right)^2=\lambda(\mathscr{R}(\bm X_{L-m}))^2.
\end{equation}
} and well-known properties of conditional PDFs 
to bound the integral in the second term and 
the integrals in the last summation, respectively, as
\begin{align}
\frac{\lambda(\mathscr{R}(\bm X_{L-m}))^2}{\lambda(\mathscr{R}(\bm X_{L-m+1}))} \leq \int_{\mathscr{R}(\bm X_{L-m+1})}\left(
\int_{\mathscr{R}(\bm X_{L-m})} \rho_{L-m}(\bm y-
\bm F_{L-m}(\bm z,\bm w_{L-m}))d\bm z\right)^2d\bm y,
\label{r11}
\end{align}
and
\begin{align}
&\sum_{k=1}^M\int_{\mathscr{R}(\bm X_{L-m+1})}
\left(\int_{\mathscr{R}(\bm X_{L-m})}
\rho_{L-m} (\bm y-\bm F_{L-m}(\bm z,\bm w_{L-m}))
\eta^m_k(\bm z)d\bm z\right)^2d\bm y
\geq 0.
\label{r8}
\end{align}
By combining  \eqref{re1}-\eqref{r8} we finally obtain
\begin{align}
\left\|\Q_{m}\G(m,m-1)\right\|^2\leq& 
\left\|\gamma_{L-m} \right\|^2_{L^2(\mathscr{R}
(\bm X_{L-m+1})\times \mathscr{R}(\bm X_{L-m}))}
\nonumber\\
\leq&\lambda(\Omega_{L-m})\left(
\left\|\rho_{L-m}\right\|_{L^2(\mathscr{R}
(\bm \xi_{L-m}))}^2-
\frac{1}
{\lambda(\Omega_{L-m+1})}\right),
\end{align}
which proves the Theorem.

\end{proof}

\vs
\noindent
{\bf Remark:} The last two terms in \eqref{re1} represent the $L^2$ 
norm of the projection of $\rho_{L-m}$ onto the 
orthonormal basis 
$\{\lambda(\mathscr{R}(\bm X_{L-m}))^{-1/2},
\eta^m_1,\ldots,\eta^m_M\}$. 
If we assume that 
$\rho_{L-m}(\bm y-\bm F_{L-m}(\bm x,\bm w_{L-m})$ 
is in 
$L^2(\mathscr{R}(\bm X_{L-m+1})\times
\mathscr{R}(\bm X_{L-m}))$, then by using Parseval's 
identity we can write \eqref{re} as 
\begin{align}
&\int_{\mathscr{R}(\bm X_{L-m+1})}
\int_{\mathscr{R}(\bm X_{L-m})} 
\gamma_{L-m} (\bm y,\bm x)^2 d\bm y d\bm x = 
\nonumber\\
&\sum_{k=M+1}^{\infty} \int_{\mathscr{R}(\bm X_{L-m+1})}
\left(\int_{\mathscr{R}(\bm X_{L-m})}
\rho_{L-m} (\bm y-\bm F(\bm z,\bm w_{L-m}))
\eta^m_k(\bm z)d\bm z\right)^2d\bm y,
\end{align}
where $\{\eta_{M+1},\eta_{M+2}\ldots\}$ is an 
orthonormal basis 
 for the orthogonal complement 
(in $L^2(\mathscr{R}(\bm X_{L-m}))$) 
of the space spanned by the basis
$\{\lambda(\mathscr{R}(\bm X_{L-m}))^{-1/2},
\eta^m_1,\ldots,\eta^m_M\}$. This allows us 
to bound \eqref{r8} from below (with a nonzero bound). 
Such lower bound depends on the basis functions 
$\eta^m_k$, on the weights $\bm w_{L-m}$ 
as well as on the choice of the 
transfer function $\bm F_{L-m}$. This implies that 
the bound \eqref{nbound} can be improved, 
if we provide information on $\eta^m_k$ and the
activation function $\bm F$.
Note also that the bound \eqref{nbound} is 
formulated in terms of the Lebesgue 
measure of $\Omega_{L-m}$, i.e., $\lambda(\Omega_{L-m})$. 
The reason is that $\lambda(\Omega_{L-m})$ depends only on the 
range of the noise (see definition \eqref{OMEGAm1}), 
while $\lambda\left(\mathscr{R}(\bm X_{L-m})\right)$ 
depends on the range of the noise, on the weights of 
the layer $L-m$, and on the range of $\bm X_{L-m+1}$.

\begin{lemma}
\label{lemma:contraction}
Consider the projection kernel \eqref{Kfin} 
with $c_m=1/\lambda(\mathscr{R}(\bm X_{L-m}))$ 
and let $\kappa\geq 0$. If
\begin{equation}
\left\|\rho_{L-m}\right\|_{L^2(
\mathscr{R}(\bm \xi_{L-m}))}^2\leq \frac{\kappa}{\lambda(\Omega_{L-m})}+\frac{1}
{\lambda(\Omega_{L-m+1})},
\label{bound3}
\end{equation}
then 
\begin{equation}
\left\|\Q_{m}\G(m,m-1)\right\|^2\leq \kappa.
\label{c5}
\end{equation}
In particular, if $0\leq \kappa <1$ then
$\Q_{m}\G(m,m-1)$ is a contraction.

\end{lemma}
\begin{proof}
The proof follows immediately from equation \eqref{nbound}.

\end{proof}

\noindent
The upper bound in \eqref{bound3} is 
a slight improvement of the bound we obtained 
in \ref{sec:functionalsetting}, 
Lemma \ref{lemma:contractionsNM}.

\subsection{Contractions induced by uniform random noise}
\label{sec:uniformrandomnoise}
{\color{r}Consider the neural network model \eqref{discrete_sys_10}
and suppose that each $\bm \xi_n$ is a random vector with 
i.i.d. uniform components supported in $[-b_{n},b_{n}]$ ($b_n>0$).} 
In this assumption, the $L^2(\mathscr{R}(\bm \xi_{L-m}))$ 
norm of $\rho_{L-m}$ appearing in Theorem 
\ref{prop:contraction} and Lemma \ref{lemma:contraction} 
can be computed analytically as  
\begin{equation}
\left\|\rho_{L-m}\right\|_{L^2(
\mathscr{R}(\bm \xi_{L-m}))}^2 = 
\frac{1}{\lambda(\mathscr{R}(\bm \xi_{L-m}))}=
\frac{1}{(2b_{L-m})^N},
\label{lea}
\end{equation}
where $N$ is the number of neurons in each layer. 
{\color{r} For uniform random variables with independent components 
it straightforward to show that the Lebesgue measure 
of the set $\Omega_{L-m}$ defined in \eqref{OMEGAm1}
and appearing in Lemma \ref{lemma:contraction} is 
\begin{equation}
\lambda(\Omega_{L-m})=2^N(1+b_{L-m-1})^N,
\label{T4}
\end{equation}
i.e.,
\begin{equation}
\lambda(\Omega_{1})=2^N(1+b_{0})^N, \qquad 
\lambda(\Omega_{2})=2^N(1+b_{1})^N,\qquad \text{etc.} 
\end{equation}
A substitution of \eqref{lea} and \eqref{T4} into
the inequality \eqref{bound3} yields}
\begin{equation}
\left(\frac{1+b_{L-m-1}}{b_{L-m}}\right)^N\leq \kappa
+\left(\frac{1+b_{L-m-1}}
{1+b_{L-m}}\right)^N.
\end{equation}
{\color{r}Upon definition of $n=L-m$ this can be written as}
\begin{equation}
\frac{b_n(b_n+1)}{\left[(b_n+1)^N-
b_n^N\right]^{1/N}}\geq 
\frac{b_{n-1}+1}{\kappa^{1/N}}, 
\qquad n=1,\ldots, L-1.
\label{unifomCondition}
\end{equation}
A lower bound for the coefficient $b_0$ can be set 
using Proposition \ref{prop:impossibility_to_train} 
in \ref{sec:functionalsetting}, i.e., 
\begin{equation}
b_0\geq \frac{1}{2}\left(\frac{\lambda(\Omega_0)}{\kappa}\right)^{1/N}.
\end{equation}
With a lower bound for $b_0$ available, we can 
compute a lower bound for each 
$b_n$ ($n=1,2,\ldots$) by solving the recursion \eqref{unifomCondition} with an equality sign. 
\begin{figure}[t]
\centerline{\hspace{0.2cm}$\kappa=0.2$\hspace{6.2cm}$\kappa=0.2$}
\centerline{
\includegraphics[height=5.6cm]{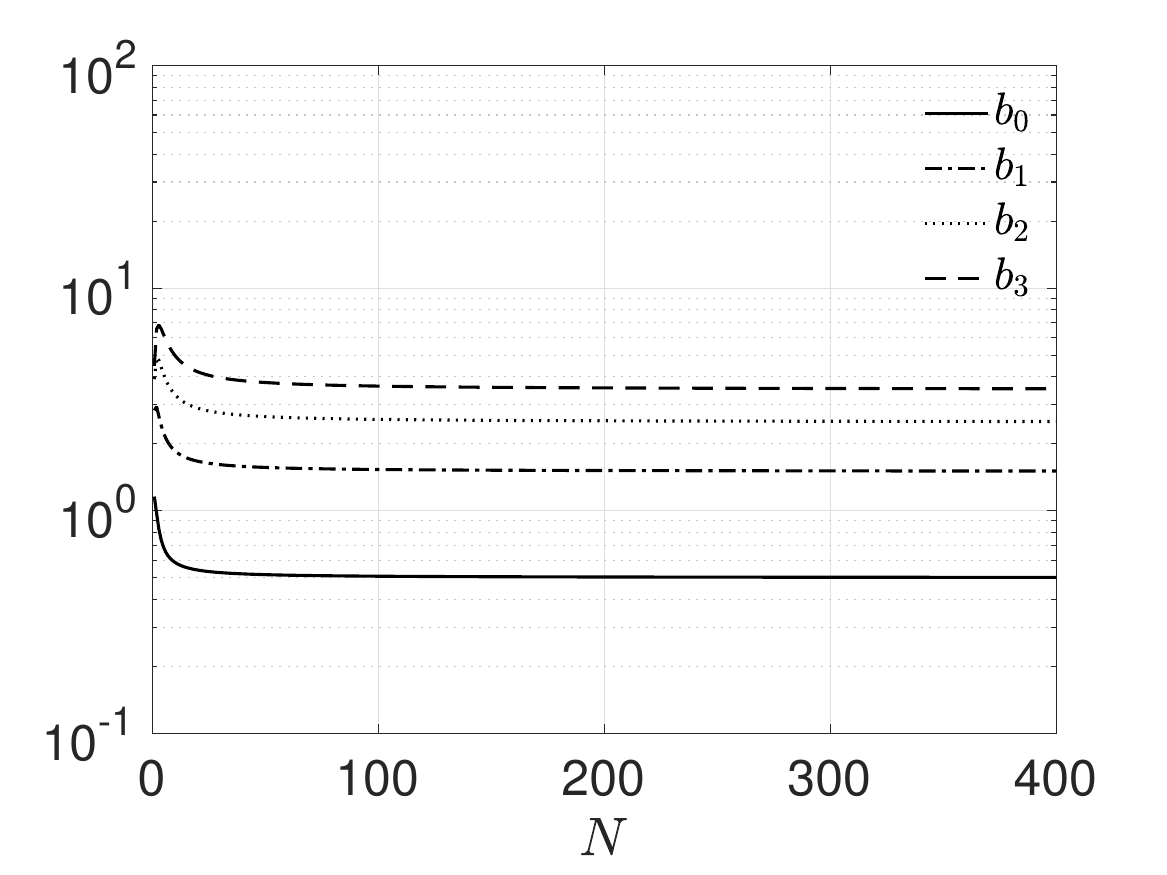}
\includegraphics[height=5.6cm]{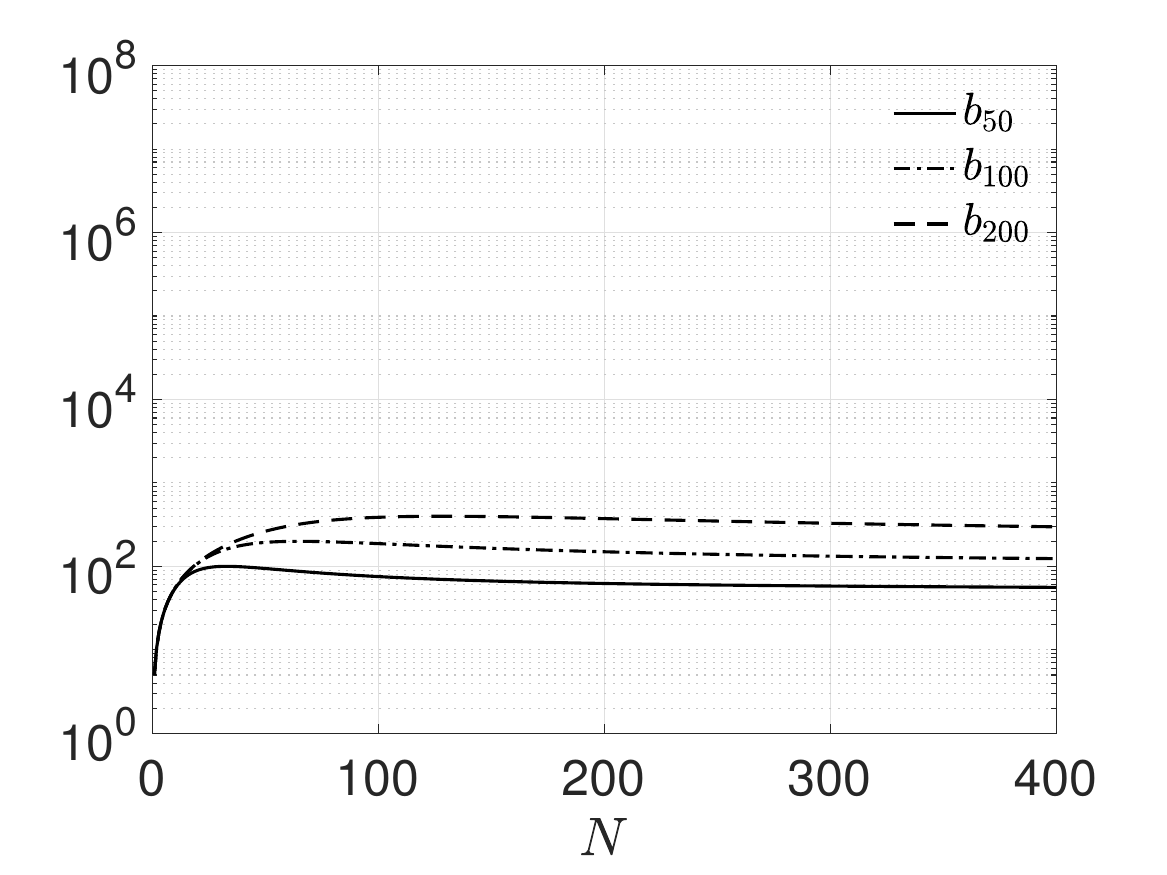}}
\centerline{\hspace{0.4cm}$\kappa=10^{-4}$\hspace{6.0cm}$\kappa=10^{-4}$}

\centerline{
\includegraphics[height=5.6cm]{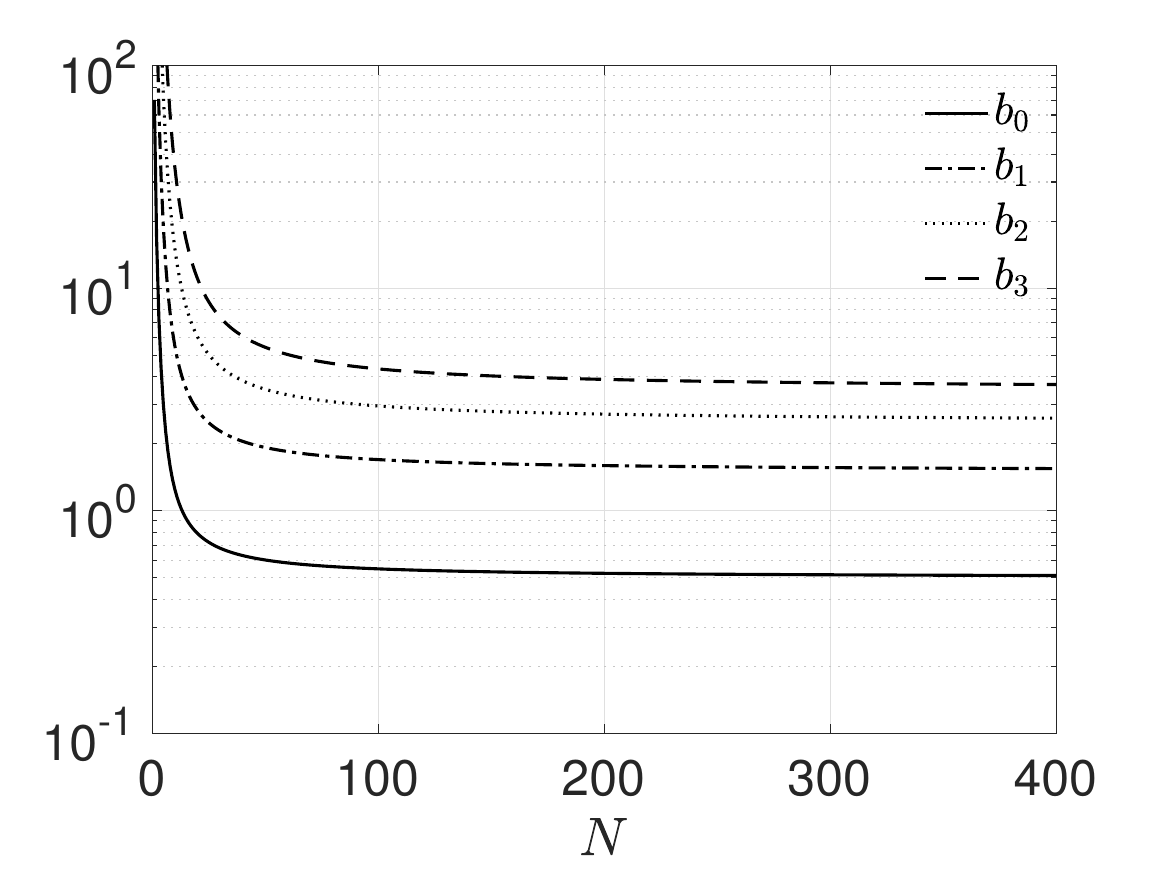}
\includegraphics[height=5.6cm]{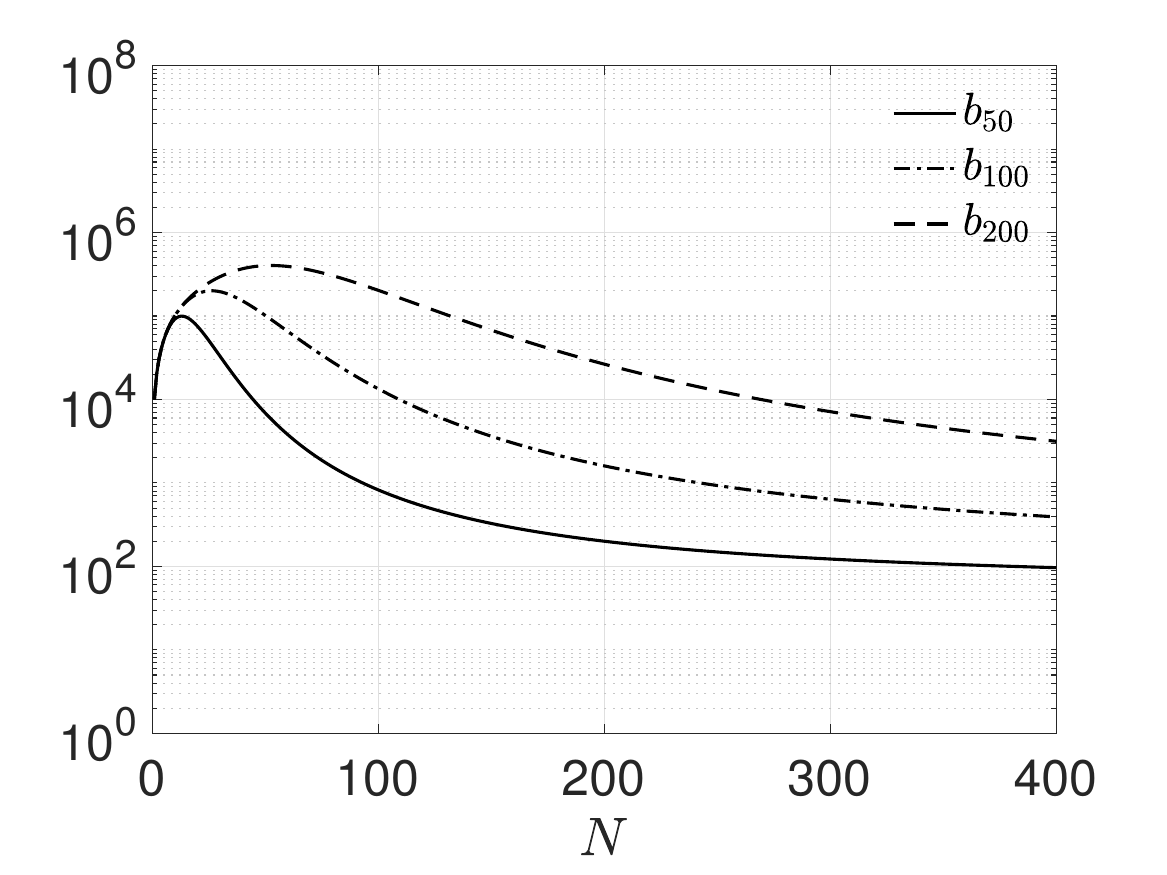}}

\caption{Lower bound on the noise amplitude \eqref{unifomCondition} 
versus the number of neurons ($N$) for $\lambda(\Omega_0)=1$ 
{\color{r}(Lebesgue measure of domain $\Omega_0$ defining the 
neural network input)},  
and different {\color{r} user-defined} contraction factors $\kappa$. 
With these values of $b_n$ the operator
$\Q_{L-n}\G(L-n,L-n-1)$ is a contraction satisfying 
$\left\|\Q_{L-n}\G(L-n,L-n-1)\right\|^2\leq \kappa$ regardless of the 
neural network weights and biases.}
\label{fig:uniformCondition}
\end{figure}
This is done in Figure \ref{fig:uniformCondition} for 
different {\color{r} user-defined} contraction factors 
$\kappa$\footnote{To compute the lower bounds 
of $b_n$ we solved the recursion 
\eqref{unifomCondition} numerically (with an equality sign) 
for $b_n$ using the Newton method. 
To improve numerical accuracy 
we wrote the left hand side 
of \eqref{unifomCondition} in the 
equivalent form 
\begin{equation}
\frac{b_n(b_n+1)}{\left[(b_n+1)^N-
b_n^N\right]^{1/N}}=
\displaystyle\frac{b_n}{\left[1-\left(\displaystyle\frac{1}{b_n}+1\right)^{-N}\right]^{1/N}}.
\label{unifomCondition1}
\end{equation}
}.
{\color{r}
It is seen that for a fixed number of neurons $N$, 
the noise level (i.e., a lower bound for $b_n$) that 
yields operator contractions in the sense of 
\begin{equation}
\left\|\Q_{L-n}\G(L-n,L-n-1)\right\|^2\leq \kappa, \qquad 
\kappa<1, \qquad n=1,2\ldots L-1,
\end{equation}
increases as we move from the input to the output, i.e., 
\begin{equation}
b_0<b_1<b_2<\cdots<b_L.
\end{equation} 
For instance, for a neural network with layers and 
$N=100$ neurons per layer the noise amplitude that induces 
a contraction factor $\kappa = 10^{-4}$ independently 
of the neural network weights is $b_0\simeq 0.55$.
This means that if each component of the random 
vector $\bm \xi_0$ is a uniform random variable with 
range $[-0.55,0.55]$ then the operator norm of 
$\Q_2\G(2,1)$ is bounded by $10^{-4}$. 
Moreover, we notice that as we increase the number of neurons $N$, the 
smallest noise amplitude that satisfies the operator contraction condition 
\begin{equation}
\left\|\Q_{L-n}\G(L-n,L-n-1)\right\|^2\leq \kappa\qquad  \kappa<1,
\label{contraction-final}
\end{equation}
converges to a constant value that depends 
on the layer $n$ but not on the contraction factor $\kappa$. Such asymptotic 
value can be computed analytically. }

\begin{lemma}
\label{lemma:asymptoticswithN_QG}
{\color{r} Consider the neural network model \eqref{discrete_sys_10} 
and suppose that each perturbation vector $\bm \xi_n$ 
has i.i.d. components distributed uniformly 
in $[-b_n,b_n]$. The smallest 
noise amplitude $b_n$ that satisfies the operator 
contraction condition \eqref{contraction-final} 
satisfies the asymptotic result
\begin{equation}
\lim_{N\rightarrow \infty} b_n = \frac{1}{2}+n\qquad n=0,\ldots, L-1
\label{bnQg}
\end{equation} 
independently of the contraction factor $\kappa$ and 
$\Omega_0$ (domain of the neural 
network input). }
\end{lemma}

\begin{proof}
The proof follows immediately by substituting the identity 
\begin{equation}
\lim_{N\rightarrow \infty} \left[(b_n+1)^N-
b_n^N\right]^{1/N} = b_n+1
\end{equation}
into \eqref{unifomCondition}.
\end{proof}

\subsection{Fading property of the neural network memory operator}

We now discuss the implications of the contraction 
property of $\Q_m\G(m,m-1)$ on the MZ equation. It is 
straightforward to show that if Proposition 
\ref{lemma:contraction1} or Lemma \ref{lemma:contraction}
holds true then the MZ memory and noise terms in \eqref{MZDISCRETEEXP} 
decay with the number of layers. This property is summarized 
in the following Theorem. 
\begin{theorem}
\label{lemma:memory}
If the conditions of Lemma \ref{lemma:contraction} are satisfied,  
then the MZ memory operator in Equation \eqref{PHIG} 
decays with the number of layers in the neural 
network, i.e., 
\begin{equation}
\left\|\Phi_\G(n,m+1)\right\|^2\geq \frac{1}{\kappa}\left\|\Phi_\G(n,m)\right\|^2\qquad \forall n\geq m+1, \qquad  0<\kappa<1.
\end{equation}
Moreover, 
\begin{equation}
\left\|\Phi_\G(n,0)\right\|^2\leq \kappa^n, 
\label{fading}
\end{equation}
i.e., the memory operator $\Phi_\G(n,0)$ decays 
exponentially fast with the number of layers.
\end{theorem}
\begin{proof}
The proof follows  from  
$\left\|Q_{m+1}\G(m+1,m)\right\|^2\leq \kappa$ and equation \eqref{PHIG}. 
In fact, for all $n\geq m+1$
\begin{equation}
\left\|\Phi_\G(n,m)\right\|^2=
\left\|\Phi_\G(n,m+1)\Q_{m+1}\G(m+1,m)\right\|^2\leq \kappa 
\left\|\Phi_\G(n,m+1)\right\|^2.
\end{equation}

\end{proof}

\noindent
This result can be used to approximate the MZ 
equation of a neural network with a large number 
of layers to an equivalent one involving only 
a few layers. {\color{r} A simple numerical 
demonstration of the fading memory 
property \eqref{fading} is provided in 
Figure \ref{fig:MZplot} for a two-layer neural 
deterministic network. }

The fading memory property allows us to simplify terms in the MZ equation that are smaller than others. The most extreme case would be a memoryless neural network, i.e., a neural network in which the MZ memory term is zero. Such network is essentially equivalent to a one-layer network. To show this, consider the MZ equation \eqref{MZDISCRETEEXP_bis} in the case where the neural network is deterministic. Suppose that the $L^2$ projection 
operator $\mathcal{P}$ is the same for each layer and it satisfies 
$(\mathcal{I}-\mathcal{P})q_0=0$, i.e., $q_0$ is in the range of $\mathcal{P}$. Then the output of the memoryless network, with input $\bm x\in \Omega_0\subseteq \mathbb{R}^d$, $\tanh()$ activation function, $L$ layers, $N$ neurons 
per layer, can be written as  
\begin{align}
q_L(\bm x) \simeq &\mathcal{G}(L,L-1)\mathcal{P} q_{L-1} \nonumber\\
=&\bm \beta \cdot \bm \eta\Big(\tanh\left(\bm W_{0}\bm x+\bm b_0\right)\Big),\qquad \bm x\in \Omega_0.
\label{NNNN1}
\end{align}
where $\mathcal{P} q_{L-1}= \bm \beta \cdot \bm \eta(\bm x)$ and 
$\bm \eta=[\eta_0(\bm x),\ldots,\eta_M(\bm x)]^T$ is a vector of 
orthonormal basis functions on $[-1,1]^N$. Regarding what types of 
input-output maps can be represented by memoryless neural networks, 
the answer is provided by the universal approximation theorem 
for non-affine activation functions of the form \eqref{NNNN1}.
We emphasize that there is no information loss 
associated with the fading MZ memory property as 
the MZ equation is formally exact. However, if we 
approximate the MZ equation by neglecting small 
terms then we may lose some information.

\subsection{Reducing deep neural networks to shallow neural networks}

Consider the MZ equation \eqref{MZDISCRETE2EXP}, 
hereafter rewritten for convenience
\begin{equation}
\bm q_L=\Phi_\G(L,0)\bm q_0 +\sum_{m=0}^{L-1}
\G(L,m+1)\P_{m+1}\G(m+1,m)\Phi_\G(m,0)\bm q_0.
\label{MZ10}
\end{equation}
We have seen that the memory operator  $\Phi_\G(m,0)$ decays 
exponentially fast with the number of layers if the operator 
$\Q_m\G(m,m-1)$ is a contraction 
(see Lemma \ref{lemma:contraction}) . Specifically, 
we proved in Theorem \ref{lemma:memory} that 
\begin{equation}
\left\|\Phi_\G(m,0)\right\|^2\leq \kappa^m\qquad 0\leq\kappa<1,
\end{equation}
where $\kappa$ is a contraction factor our 
choice\footnote{\color{r}Recall for any choice of  
contraction factor $\kappa$ there always 
exists a sequence of uniformly distributed 
independent random vectors $\bm \xi_n$ with increasing 
amplitude such that $\left\|\Q_m\G(m+1,m)\right\|\leq \kappa$ 
for all $m$, 
independently of the neural network weights 
(see Lemma \ref{lemma:contraction} and the 
discussion in section \ref{sec:uniformrandomnoise}).}.
{\color{r}Hereafter we show that the magnitude 
of each term at the right hand side of
\eqref{MZ10} can be controlled by $\kappa$ 
independently of the neural network weights. In principle, 
this allows us to approximate a deep stochastic neural 
network using only a subset of terms 
in \eqref{MZ10}}.

\begin{proposition}
\label{prop:MZbound}
{\color{r} Consider the stochastic neural network 
model \eqref{discrete_sys_10} and assume that 
each random vector $\bm \xi_{m}$ has bounded range $\mathscr{R}(\bm \xi_{m})$ 
and PDF $\rho_{m}\in L^2(\mathscr{R}(\bm \xi_{m}))$.} Then 
\begin{equation}
\left\|\G(L,m+1)\P_{m+1}\G(m+1,m)\Phi_\G(m,0)\right\|^2\leq 
B^{L}\left(\frac{\kappa}{B}\right)^m,\qquad m=0,\ldots, L-1
\label{ub}
\end{equation}
where $\kappa$ is defined in Lemma \ref{prop:contraction} 
and  
\begin{equation}
\max_{m=0,\ldots,L-1}\left\|\G(m+1,m)\right\|^2\leq B, 
\qquad (B<\infty).
\label{B}
\end{equation}
{\color{r}The upper bound in \eqref{ub}} is independent 
of the neural network weights. 
\end{proposition}
\begin{proof}
{\color{r}We have shown in \ref{sec:functionalsetting} 
(Proposition \ref{lemma:operatorbounds}) that if 
 $\bm \xi_{m}$ has bounded range $\mathscr{R}(\bm \xi_{m})$ the 
 PDF $\rho_{m}\in L^2(\mathscr{R}(\bm \xi_{m}))$ then  
it is possible to to find an upper bound for $\G(m+1,m)$ 
that is independent of the neural network weights 
and $\rho_m$.} By using standard operator norm 
inequalities and recalling Theorem 
\ref{lemma:memory} we immediately obtain 
\begin{align}
\left\|\G(L,m+1)\P\G(m+1,m)\Phi_\G(m,0)\right\|^2\leq &
\left\|\G(L,L-1)\cdots \G(m+2,m+1)\right\|^2\left\|\G(m+1,m)
\right\|^2\kappa^m\nonumber\\
\leq &
\left(\max_{i=m,\ldots,L-1}\left\|\G(i+1,i)\right\|^2\right)^{L-m}
\kappa^m\nonumber\\
\leq & B^{L-m}\kappa^m,
\label{171}
\end{align}
where $B$ is defined in \eqref{B}. 

\end{proof}


\section{Summary}
\label{sec:summary}
We developed a new formulation of deep learning 
based on the Mori-Zwanzig (MZ) projection operator 
formalism of irreversible statistical mechanics.
The new formulation provides new insights 
on how information propagates through  
neural networks in terms of formally exact linear 
operator equations, and it introduces a new important 
concept, i.e., the {\em memory} of the neural network, 
which plays a fundamental role in low-dimensional modeling 
and parameterization of the network (see, e.g., \cite{lei2016data}). 
By using the theory of contraction mappings, we developed 
sufficient conditions for the memory of the neural network to 
decay with the number of layers. This allowed us to 
rigorously transform deep networks into shallow ones, 
e.g., by reducing the number of neurons per layer (using projections),
or  by reducing the total number of layers (using the decay property
of the memory operator).  
We developed most of the analysis for MZ equations involving conditional 
expectations, i.e., Eqs. \eqref{MZDISCRETEEXP}-\eqref{MZDISCRETE2EXP}. 
However, by using the well-known duality between PDF dynamics and conditional expectation dynamics \cite{dominy2017duality}, it is straightforward to 
derive similar analytic results for MZ equations involving PDFs, 
i.e., Eqs. \eqref{MZDISCRETEPDF}-\eqref{MZDISCRETE2PDF}.
Also, the mathematical techniques we developed in this paper can be 
generalized to other types of stochastic neural network models, e.g., 
neural networks with random weights and biases.

An important open question is the development of 
effective approximation methods for the MZ memory operator 
and the noise term. Such approximations can be built 
upon continuous-time approximation methods, 
e.g., based on functional analysis 
\cite{zhu2018estimation,zhu2018faber,Lin2020}, 
combinatorics \cite{zhu2019generalized}, data-driven methods 
\cite{Brennan2018,price2017renormalized,lu2016comparison,ma2019coarse}, 
Markovian embedding techniques 
\cite{1367-2630-15-12-125015,hijon2006markovian,ma2016derivation,lei2016data,chu2019mori}, 
or projections based on reproducing kernel Hilbert or Banach spaces \cite{Bartolucci2021,Parhi2021,Zhang2009}. 

\section*{Acknowledgements}
\noindent
Dr. Venturi was partially supported by the U.S. Air Force 
Office of Scientific Research grant FA9550-20-1-0174, and by the U.S. Army 
Research Office grant W911NF1810309. Dr. Li was supported by the 
NSF grant DMS-1953120.

\section*{Data availability statement}
\noindent
The data that support the findings of this study 
are available from the corresponding author upon 
request.

\section*{Conflict of interest statement} 
\noindent 
The authors declare that there is no conflict of interest.

\appendix

\section{Functional setting}
\label{sec:functionalsetting}
Let $(\mathcal{S},\mathcal{F},\mathscr{P})$ be a 
probability space. Consider the neural network 
model \eqref{discrete_sys_10}, hereafter rewritten 
for convenience as
\begin{equation}
\bm X_{n+1}=\bm F_n(\bm X_n,\bm w_n)+\bm \xi_n 
\qquad n=0,\ldots,L-1,
\label{NNfunctional}
\end{equation}
We assume that the following conditions are satisfied
\begin{enumerate}
\item $\bm X_0\in \Omega_0\subseteq \mathbb{R}^d$ ($\Omega_0$ compact), $\bm X_n\in\mathbb{R}^N$ for all $n=1,\ldots,L$;
\item The range of $\bm F_n$ ($n=0,\ldots,L-1$) is the hyper-cube $[-1,1]^N$.  
\end{enumerate}
We also assume that the random vectors $\{\bm \xi_0,\ldots,\bm \xi_{L-1}\}$ 
in \eqref{NNfunctional} are {\color{r}statistically independent, 
and that $\bm \xi_n$ is independent of past and current neural network 
states, i.e., $\{\bm X_0,\ldots,\bm X_{n}\}$. In these hypotheses, we 
have that $\{\bm X_0,\ldots, \bm X_L\}$ is a Markov process 
(see \ref{app:NN_characterization} for details).}
The range of $\bm X_{n+1}$ depends on the 
range of $\bm \xi_n$, as the image of each $\bm F_n$ is 
the hyper-cube $[-1,1]^N$ (condition 2. above).
Let us define\footnote{The notation $[-1,1]^N$ denotes 
a Cartesian product of $N$ one-dimensional 
domains $[-1,1]$, i.e., 
\begin{equation}
[-1,1]^N = \bigtimes_{k=1}^N [-1,1]= \underbrace{[-1,1]\times [-1,1]
\times \cdots\times [-1,1]}_{\text{$N$ times}}.
\end{equation}

}
\begin{align}
\Omega_{n+1} =& [-1,1]^N+\mathscr{R}(\bm \xi_n)\nonumber\\
= &\{\bm c\in \mathbb{R}^N\,:\, \bm c=\bm a+\bm b\quad
\bm a\in [-1,1]^N,\, \bm b\in \mathscr{R}(\bm \xi_n)\},
\label{Omega_n+1}
\end{align}
where
\begin{equation}
\mathscr{R}(\bm \xi_n)= \{\bm \xi_n(\omega)\in \mathbb{R}^N: 
\omega\in \mathcal{S}\}.
\label{Rn+1}
\end{equation}
is the range of the random vector $\bm \xi_n$. 
Clearly, the range of the random vector $\bm X_{n+1}$ is a 
subset\footnote{We emphasize that if we 
are given more information on the 
activation functions $\bm F_n$ together with 
suitable bounds on the neural network parameters 
$\bm w_n$ then we can identify a domain 
that is smaller than $\Omega_{n+1}$ which 
still contains $\mathscr{R}(\bm X_{n+1})$. 
This allows us to construct a tighter bound for 
$\lambda(\mathscr{R}(\bm X_{n+1})$ 
in Lemma \ref{lemma:Lebesgue_measure_X_Omega}, 
which depends on the activation function and on 
the parameters of the neural network.}
of $\Omega_{n+1}$, i.e., $\mathscr{R}(\bm X_{n+1})
\subseteq \Omega_{n+1}$. This implies the following 
lemma. 
\begin{lemma}
\label{lemma:Lebesgue_measure_X_Omega}
Let $\lambda(\Omega_{n+1})$ {\color{r} be} the Lebesgue measure 
of the set \eqref{Omega_n+1}. Then the Lebesgue measure 
of the range of $\bm X_{n+1}$ satisfies
\begin{equation}
\lambda(\mathscr{R}(\bm X_{n+1}))\leq \lambda(\Omega_{n+1}).
\end{equation}
\end{lemma}
\begin{proof}
The proof follows immediately from the inclusion 
$\mathscr{R}(\bm X_{n+1})
\subseteq \Omega_{n+1}$.
\end{proof}

\noindent
The $L^{\infty}$ norm of a random vector $\bm \xi$ 
is defined as the largest value of $r\geq 0$ that yields 
a nonzero probability on the event $\{\omega\in\mathcal{S}:\,
\left\|\bm \xi(\omega)\right\|_{\infty}>r \}\in\mathcal{F}$, 
i.e., 
\begin{equation}
\left\|\bm \xi \right\|_{\infty}=\sup_{r\in \mathbb{R}}\{\, 
\mathscr{P}(\{\omega\in\mathcal{S}:\,
\left\|\bm \xi(\omega)\right\|_{\infty}>r \})>0\}.
\end{equation}
This definition allows us to bound the 
Lebesgue measure of $\Omega_{n+1}$ as follows.
\begin{proposition}
\label{prop:Lebesgue_bound}
The Lebesgue measure of the set $\Omega_{n+1}$ 
defined in \eqref{Omega_n+1} can be bounded as 
\begin{equation}
\lambda(\Omega_{n+1}) \leq \left(\sqrt{N}+\left\|\bm \xi_{n} 
\right\|_{\infty}\right)^N\frac{\pi^{N/2}}{\Gamma(1+N/2)},
\label{bound1}
\end{equation}
where $N$ is the number of neurons and 
$\Gamma(\cdot)$ is the  Gamma function. 
\end{proposition}
\begin{proof}
As is well known, the length of the diagonal of 
the hypercube $[-1,1]^N$ is $\sqrt{N}$. Hence, 
$\sqrt{N}+\left\|\bm \xi_{n}\right\|_{\infty}$ is 
the radius of a ball that encloses all elements of 
$\Omega_{n+1}$. The Lebesgue measure of such  a  ball is obtained by multiplying the Lebesgue measure 
of the unit ball in $\mathbb{R}^N$, i.e., 
$\pi^{N/2}/\Gamma(1+N/2)$ by the scaling 
factor $ \left(\sqrt{N}+\left\|\bm \xi_{n} 
\right\|_{\infty}\right)^N$.

\end{proof}

\begin{lemma}
If  $\mathscr{R}(\bm \xi_{n})$ is bounded 
then $\mathscr{R}(\bm X_{n+1})$ is bounded.
\label{lemma43}
\end{lemma}
\begin{proof}
The image of the activation function $\bm F_n$ is 
a bounded set. If $\mathscr{R}(\bm \xi_{n})$ is 
bounded then $\Omega_{n+1}$  in \eqref{Omega_n+1}
is bounded. Since $\mathscr{R}(\bm X_{n+1})\subseteq \Omega_{n+1}$ 
we have that $\mathscr{R}(\bm X_{n+1})$ is bounded.

\end{proof}

\noindent
Clearly, if $\{\bm \xi_{0},\ldots, \bm \xi_{L-1}\}$ are 
i.i.d. random variables then  
there exists a domain $V=\Omega_1=\cdots=\Omega_{L}$ 
such that
\begin{equation}
\mathscr{R}(\bm \xi_{n})\subseteq \mathscr{R}(\bm X_{n+1})
\subseteq V\qquad \forall n=0,\ldots,L-1.
\end{equation}
In fact, if $\{\bm \xi_{0},\ldots, \bm \xi_{L-1}\}$ are 
i.i.d. random variables then we have 
\begin{equation}
\mathscr{R}(\bm \xi_0)=\mathscr{R}(\bm \xi_1)=\cdots=
\mathscr{R}(\bm \xi_{L-1}),
\end{equation}
which implies that all $\Omega_i$ defined in 
\eqref{Omega_n+1} are the same. If the range of each 
random vector $\bm \xi_n$ is a tensor product of one-dimensional 
domain, e.g., if the components of $\bm \xi_n$ are statistically 
independent, then $V=\Omega_1=\cdots=\Omega_{L}$
becomes particularly simple, i.e., a hypercube. 

\begin{lemma}
\label{lemma4.4}
Let $\{\bm \xi_0,\ldots,\bm \xi_{L-1}\}$ be i.i.d. random 
variables with bounded range and suppose that each 
$\bm \xi_k$ has statistically independent components with 
range $[a,b]$. Then all domains $\{\Omega_{1},\ldots,\Omega_L\}$ defined in equation \eqref{Omega_n+1} are the same, and they are 
equivalent to  
\begin{equation}
V= \bigtimes_{k=1}^N [-1+a,1+b].
\label{hypercube}
\end{equation}
$V$ includes the range of all random vectors $\bm X_{n}$ 
($n=1,\ldots,L$) and has Lebesgue measure 
\begin{equation}
\lambda(V) =  (2+b-a)^N.
\label{V2}
\end{equation}
\end{lemma}

\begin{proof}
The proof is trivial and therefore omitted.

\end{proof}

\noindent
{\bf Remark:} It is worth noticing that if 
each $\bm \xi_k$ is a uniformly 
distributed random vector with statistically 
independent components in $[-1,1]$, then for $N=10$ neurons 
the upper bound in \eqref{bound1}
is $3.98\times 10^6$ while the exact result \eqref{V2} 
gives $1.05\times 10^6$. 
Hence the estimate \eqref{bound1} is sharp in the case of 
uniform random vectors.

\subsubsection*{Boundedness of composition and transfer operators}

Lemma \ref{lemma43} states that if we perturb 
the output of the $n$-th layer of a 
neural network by a random vector $\bm \xi_n$ with finite 
range then we obtain a random vector $\bm X_{n+1}$ with finite 
range. In this hypothesis, it is straightforward to show that 
that the composition and transfer operators 
defined in \eqref{M} and \eqref{N} are bounded. 
We have seen {\color{r} in section \ref{sec:compositiontransfer}}
that these operators can be written as 
\begin{equation}
\M(n,n+1) v= \int_{\mathscr{R}(\bm X_{n+1})} v(\bm y)
p_{n+1|n}(\bm y|\bm x)d\bm y,\qquad \N(n+1,n) v=\int_{\mathscr{R}(\bm X_{n})} p_{n+1|n}(\bm x|\bm y)v(\bm y)d\bm y,
\label{MN}
\end{equation} 
where $p_{n+1|n}(\bm y|\bm x)=\rho_n(\bm y-\bm F_n(\bm x,\bm w_n))$ is the conditional transition density of $\bm X_{n+1}$ given $\bm X_n$, and $\rho_n$ is the joint PDF of the random vector $\bm \xi_n$. 
The conditional transition density $p_{n+1|n}(\bm y|\bm x)$ 
is always non-negative, i.e., 
\begin{equation}
p_{n+1|n}(\bm y|\bm x)\geq 0 \qquad \forall 
\bm y\in \mathscr{R}(\bm X_{n+1}), \quad \forall 
\bm x\in \mathscr{R}(\bm 
X_{n}).
\end{equation}
Moreover, the conditional density $p_{n+1|n}$ is 
defined on the set
\begin{equation}
\mathscr{B}_{n}=\{(\bm x,\bm y)\in \mathscr{R}
(\bm X_{n})\times \mathscr{R}(\bm X_{n+1}): \, 
(\bm y-\bm F_n(\bm x,\bm w_{n}))\in 
\mathscr{R}(\bm \xi_{n})\}.
\label{setB}
\end{equation}
Both $\mathscr{R}(\bm X_{n+1})$ and 
$\mathscr{R}(\bm X_{n})$ depend 
on $\Omega_0$ (domain of the neural network input), the 
neural network weights, and the 
noise amplitude. Thanks to Lemma \ref{lemma:Lebesgue_measure_X_Omega}, 
we have that 
\begin{equation}
\mathscr{B}_{n}\subseteq \Omega_{n}\times \Omega_{n+1}.
\end{equation}
The Lebesgue measure  of $\mathscr{B}_{n}$ can be 
calculated as follows.

\begin{lemma}
The Lebesgue measure of the set $\mathscr{B}_{n}$ 
defined in \eqref{setB} is equal to the product 
of the measure of $\lambda(\mathscr{R}(\bm X_{n}))$ 
and the measure of 
$\mathscr{R}(\bm \xi_{n})$, i.e., 
\begin{equation}
\lambda(\mathscr{B}_{n})=\lambda(\mathscr{R}(\bm X_{n}))
\lambda(\mathscr{R}(\bm \xi_{n})).
\label{measurepreservation}
\end{equation}
Moreover, $\lambda(\mathscr{B}_{n})$ is bounded by 
$\lambda(\mathscr{R}(\Omega_{n}))
\lambda(\mathscr{R}(\bm \xi_{n}))$, 
which is independent of the neural network weights. 
\end{lemma}

\begin{proof}
Let $\chi_{n}$ be the indicator function of 
the set $\mathscr{R}(\bm\xi_{n})$, 
$\bm y\in \mathscr{R}(\bm X_{n+1})$ 
and $\bm x \in \mathscr{R}(\bm X_{n})$. 
Then 
\begin{align}
\lambda(\mathscr{B}_{n}) = &
\int_{\mathscr{R}(\bm X_{n+1})}
\int_{\mathscr{R}(\bm X_{n})}
\chi_{n}(\bm y-\bm F_n(\bm x,\bm w_{n}))
d\bm x d\bm y\nonumber\\
=&\lambda(\mathscr{R}(\bm\xi_{n}))
\int_{\mathscr{R}(\bm X_{n})} d\bm x\nonumber\\
=&\lambda(\mathscr{R}(\bm X_{n}))
\lambda(\mathscr{R}(\bm \xi_{n})).
\end{align}
By using Lemma \ref{lemma:Lebesgue_measure_X_Omega}
we conclude that $\lambda(\mathscr{B}_{n})$
is bounded from above by $\lambda(\mathscr{R}
(\Omega_{n}))
\lambda(\mathscr{R}(\bm \xi_{n}))$, which is 
independent of the neural network weights. 

\end{proof}

\vs
\noindent
{\bf Remark:} The equality \eqref{measurepreservation} 
has a nice geometrical interpretation in 
two dimensions. Consider a ruler of length 
$r=\lambda(\mathscr{R}(\xi_{n}))$ 
with endpoints that can leave markings if we slide 
the ruler on a rectangular table with side lengths 
$s_b=\lambda(\mathscr{R}( X_{n+1}))$ (horizontal sides)
$s_h=\lambda(\mathscr{R}(X_{n}))$ (vertical sides). 
If we slide the ruler from the top to the bottom of the 
table, while keeping it parallel 
to the horizontal side of the table 
(see Figure \ref{fig:condPDF}) then the area 
of the domain within the markings left by the 
endpoints of the ruler is always $r\times s_h$
independently of the way we slide the ruler -- provided 
the ruler never gets out of the table and never 
inverts its vertical motion.

\begin{lemma}
If the range of $\bm \xi_{n-1}$ is a bounded subset 
of $\mathbb{R}^N$ then the transition 
density $p_{n+1|n}(\bm y|\bm x)$ is an element of
$L^1\left(\mathscr{R}(\bm 
X_{n+1})\times\mathscr{R}(\bm 
X_{n})\right)$.
\end{lemma}
\begin{proof}
Note that 
\begin{equation}
\int_{\mathscr{R}(\bm 
X_{n+1})}\int_{\mathscr{R}(\bm 
X_{n})}p_{n+1|n}(\bm y|\bm x)d\bm y d\bm x = 
\lambda\left(\mathscr{R}(\bm 
X_{n})\right)\leq \lambda(\Omega_n).
\label{L1bound}
\end{equation}
The Lebesgue measure $\lambda(\Omega_n)$ 
can be bounded as (see Proposition \ref{prop:Lebesgue_bound})
\begin{equation}
\lambda(\Omega_{n}) \leq \left(\sqrt{N}+\left\|\bm \xi_{n-1} \right\|_{\infty}\right)^N\frac{\pi^{N/2}}{\Gamma(1+N/2)}.
\end{equation}
Since the range of $\bm \xi_{n-1}$ is bounded by hypothesis 
we have that there exists a finite real number $M>0$ such 
that $\left\|\bm \xi_{n-1} \right\|_{\infty}\leq M$. This 
implies that the integral in \eqref{L1bound} is finite, i.e., 
that the transition kernel $p_{n+1|n}(\bm y|\bm x)$ 
is in $L^1\left(\mathscr{R}(\bm 
X_{n+1})\times\mathscr{R}(\bm 
X_{n})\right)$.

\end{proof}

\begin{theorem}
\label{thm:boundedPDF}
Let $C_{\bm \xi_n}(\bm x)$ be the cumulative distribution 
function $\bm \xi_{n}$. If $C_{\bm \xi_n}(\bm x)$ 
is Lipschitz continuous on $\mathscr{R}(\bm \xi_n)$ and
the partial derivatives $\partial C_{\bm \xi_n}/\partial x_k$
($k=1,\ldots, N$) are Lipschitz continuous  in 
$x_1$, $x_2$, ..., $x_N$, respectively, then  
the joint probability density function of 
$\bm \xi_{n}$ is bounded on $\mathscr{R}(\bm \xi_n)$. 
\end{theorem}
\begin{proof}
By using  Rademacher's theorem we have that if 
$C_{\bm \xi_n}(\bm x)$ is Lipschitz  on 
$\mathscr{R}(\bm \xi_n)$ then it is differentiable almost 
everywhere on $\mathscr{R}(\bm \xi_n)$ (except on a set 
with zero Lebesgue measure). Therefore the partial derivatives 
$\partial C_{\bm \xi_n}/\partial x_k$ exist almost everywhere on 
$\mathscr{R}(\bm \xi_n)$. If, in 
addition, we assume that $\partial C_{\bm \xi_n}/\partial x_k$ are
 Lipschitz continuous with respect to $x_k$ (for all $k=1,\ldots,N$) 
then by applying  \cite[Theorem 9]{Minguzzi2014} recursively
we conclude that the joint probability density function 
of $\bm \xi_n$ is bounded.

\end{proof}

\begin{lemma}
\label{thrm:boundedCondPDF}
If $\rho_n$ is bounded on $\mathscr{R}(\bm \xi_n)$ then the 
conditional PDF $p_{n+1|n}(\bm y|\bm x)=\rho_n(\bm y-\bm F_n(\bm x, \bm w_n))$ 
is bounded on $\mathscr{R}(\bm 
X_{n+1})\times\mathscr{R}(\bm 
X_{n})$.
\end{lemma}

\begin{proof}
Theorem \ref{thm:boundedPDF} states that $\rho_n$ is 
a bounded function. This implies that the conditional 
density $p_{n+1|n}(\bm y|\bm x)=
\rho_n(\bm y-\bm F_n(\bm x, \bm w_n))$ 
is bounded on $\mathscr{R}(\bm 
X_{n+1})\times\mathscr{R}(\bm 
X_{n})$.

\end{proof}

\begin{proposition}
\label{prop:L2boundeness}
Let $\mathscr{R}(\bm \xi_n)$ and 
$\mathscr{R}(\bm \xi_{n-1})$ be bounded subsets 
of $\mathbb{R}^N$. {\color{r}If $\rho_n \in L^2(\mathscr{R}(\bm \xi_n))$}
then the composition and the transfer operators defined 
in \eqref{MN} are bounded in $L^2$.
\end{proposition}

\begin{proof}
Let us first prove that $\M(n,n+1)$ is a 
bounded linear operator 
from $L^{2}(\mathscr{R}(\bm X_{n+1}))$ into 
$L^{2}(\mathscr{R}(\bm X_{n}))$. To this end, 
note that
\begin{align}
\left\|\M(n,n+1) v\right\|^2_{L^{2}(\mathscr{R}(\bm X_{n}))} 
= & \int_{\mathscr{R}(\bm X_{n})}
\left| \int_{\mathscr{R}(\bm X_{n+1})} v(\bm y)
p_{n+1|n}(\bm y|\bm x)d\bm y \right|^2d\bm x\nonumber\\
\leq & \left\|v\right\|^2_{L^{2}(\mathscr{R}(\bm X_{n+1}))} 
\underbrace{\int_{\mathscr{R}(\bm X_{n})}
\int_{\mathscr{R}(\bm X_{n+1})} 
p_{n+1|n}(\bm y|\bm x)^2 d\bm y d\bm x}_{K_n}\nonumber\\
=&K_n\left\|v\right\|^2_{L^{2}(\mathscr{R}(\bm X_{n+1}))}.
\label{s4}
\end{align}
Clearly, $K_n<\infty$ {\color{r} since 
$\rho_n\in L^2(\mathscr{R}(\bm \xi_n))$}. 
By following the same steps it 
is straightforward to show that the transfer operator 
$\N$ is a bounded linear operator, i.e.,  
\begin{equation}
\left\|\N(n+1,n)p\right\|^2_{L^{2}(\mathscr{R}(\bm X_{n+1}))}\leq K_n
\left\|p\right\|^2_{L^{2}(\mathscr{R}(\bm X_{n}))}.
\label{77}
\end{equation}
Alternatively, simply recall that $\N$ is the 
adjoint of $\M$ {\color{r} (see section \ref{sec:MandN})}, 
and the fact that the adjoint of a bounded linear operator is bounded. 

\end{proof}

\vspace{0.5cm}
\noindent
{\bf Remark:} The integrals
\begin{equation}
K_n=\int_{\mathscr{R}(\bm X_n)}
\int_{\mathscr{R}(\bm X_{n+1})}
p_{n+1|n}(\bm y|\bm x)^2d\bm yd\bm x
\label{integral}
\end{equation} 
can be computed by noting that 
\begin{equation}
p_{n+1|n}(\bm y|\bm x)=\rho_n(\bm y-\bm F_n(\bm x,\bm w_n))
\end{equation}
is essentially a {\em shift} of the PDF $\rho_n$ by a 
quantity $\bm F_n(\bm x,\bm w_n)$ that depends on 
$\bm x$ and $\bm w_n$ (see, e.g., Figure \ref{fig:condPDF}). 
Such a shift does not influence the integral 
with respect to $\bm y$, meaning that the 
integral of $p_{n+1|n}(\bm y|\bm x)$ or 
$p_{n+1|n}(\bm y|\bm x)^2$ with respect 
to $\bm y$ is the same for all $\bm x$. 
Hence, by changing variables we have that the 
integral \eqref{integral} is equivalent to 
\begin{equation}
K_n = \lambda(\mathscr{R}(\bm X_n)) 
\int_{\mathscr{R}(\bm \xi_n)}\rho_n(\bm x)^2d\bm x,
\label{Kn11}
\end{equation}
where $\lambda(\mathscr{R}(\bm X_n))$ is the Lebesgue 
measure of $\mathscr{R}(\bm X_n)$, and 
$\mathscr{R}(\bm \xi_n)$ is the range of $\bm \xi_n$. 
Note that $K_n$ depends on the neural net weights only 
through the Lebesgue measure of 
$\mathscr{R}(\bm X_n)$. Clearly, since the set 
$\Omega_{n}$ includes $\mathscr{R}(\bm X_n)$ we have by 
Lemma \ref{lemma:Lebesgue_measure_X_Omega} that 
$\lambda(\mathscr{R}(\bm X_n)) \leq \lambda(\Omega_n)$. 
This implies that
\begin{equation}
K_n\leq \lambda(\Omega_{n}) \int_{\mathscr{R}(\bm \xi_n)}
\rho_n(\bm x)^2d\bm x.
\end{equation} 
The upper bound here does not depend 
on the neural network weights. 
The following lemma summarizes all these remarks. 

\begin{proposition}
\label{lemma:operatorbounds}
{\color{r}Consider the neural network 
model \eqref{NNfunctional} and let $\mathscr{R}(\bm \xi_n)$ and 
$\mathscr{R}(\bm \xi_{n-1})$ be bounded subsets 
of $\mathbb{R}^N$. If $\rho_n \in L^2(\mathscr{R}(\bm \xi_n))$}
then the composition and the transfer operators 
defined in \eqref{MN} can be bounded as
\begin{equation}
\left\|\M(n,n+1)\right\|^2\leq K_n,\qquad 
\left\|\N(n+1,n)\right\|^2\leq K_n,
\end{equation}
where 
\begin{equation}
K_n=\lambda(\mathscr{R}(\bm X_n)) \int_{\mathscr{R}(\bm \xi_n)}\rho_n(\bm x)^2d\bm x.
\label{Kn0}
\end{equation}
Moreover, $K_n$ can be bounded as 
\begin{equation}
K_n\leq \lambda(\Omega_{n}) \int_{\mathscr{R}(\bm \xi_n)}
\rho_n(\bm x)^2d\bm x,
\label{Kn}
\end{equation}
where $\Omega_{n}$ is defined in \eqref{Omega_n+1} and 
$\rho_n$ is the PDF of $\bm \xi_n$. The upper bound 
in \eqref{Kn} does not depend on the neural network
weights and biases. 
\end{proposition}

\noindent
Under additional assumptions on the PDF $\rho_n(\bm x)$ it 
is also possible to bound the integrals on the right hand side of 
\eqref{Kn0} and \eqref{Kn}. Specifically, we have the following 
sharp bound. 

\begin{lemma}
\label{lemma:boundonPDF}
{\color{r}Let $\mathscr{R}(\bm \xi_n)$ be a compact subset 
of $\mathbb{R}^N$, $\rho_n$ continuous on 
$\mathscr{R}(\bm \xi_n)$.} Denote by  
\begin{equation}
s_n =\inf_{\bm x\in \mathscr{R}(\bm \xi_n)}  {\rho_n}(\bm x),\qquad 
S_n =\sup_{\bm x\in \mathscr{R}(\bm \xi_n)}  {\rho_n}(\bm x).
\label{defMm}
\end{equation}
If $s_n>0$ then  
\begin{equation}
\left\|\rho_n \right\|^2_{L^2(\mathscr{R}(\bm \xi_n))}
\leq \frac{1}{\lambda(\mathscr{R}(\bm \xi_n))} 
\frac{(S_n+s_n)^2}{4S_ns_n}.
\label{boundonPDF}
\end{equation}
\end{lemma}

\begin{proof}
{\color{r}Let us first notice that if $\rho_n$ is continuous 
on the compact set $\mathscr{R}(\bm \xi_n)$ then it is 
necessarily bounded, i.e., $S_n$ is finite.}
By using the definition \eqref{defMm} we have 
\begin{equation}
(\rho_n(\bm x)-s_n)(S_n-\rho_n(\bm x))\geq 0 \quad 
\text{for all}\quad \bm x\in \mathscr{R}(\bm \xi_n).
\end{equation}
This implies 
\begin{equation}
\int_{\mathscr{R}(\bm \xi_n)}\rho_n(\bm x)^2d\bm x 
\leq (S_n+s_n)- S_ns_n\lambda(\mathscr{R}(\bm \xi_n)), 
\label{j10}
\end{equation}
where we used the fact that the PDF $\rho_n$ integrates 
to one over $\mathscr{R}(\bm \xi_n)$. Next, define
\begin{equation}
R_n = \frac{1}{\lambda(\mathscr{R}(\bm \xi_n))} 
\frac{(S_n+s_n)^2}{4S_ns_n}.
\end{equation}
Clearly,  
\begin{align}
R_n \left(1-\frac{2S_ns_n}{s_n+S_n}\lambda(\mathscr{R}(\bm \xi_n))\right)^2 = 
R_n - (S_n+s_n) + S_ns_n\lambda(\mathscr{R}(\bm \xi_n))\geq 0 
\end{align}
which implies that 
\begin{equation}
(S_n+s_n) - S_ns_n\lambda(\mathscr{R}(\bm \xi_n)) \leq R_n.
\label{j9}
\end{equation}
A substitution of  \eqref{j9} into \eqref{j10} yields \eqref{boundonPDF}.

\end{proof}

\vs
\noindent
{\em An example:} Let us demonstrate the definitions and theorems 
above with a simple example. To this end, let $X_0 \in \Omega_0=[-1,1]$ and consider
\begin{equation}
X_1=\tanh(X_0+3)+\xi_0, \qquad X_2=\tanh(2X_1-1)+\xi_1,
\end{equation}
where $\xi_0$ and $\xi_1$ are uniform random variables 
with range $\mathscr{R}(\xi_0)=\mathscr{R}(\xi_1)=[-2,2]$. 
In this setting, 
\begin{align}
\mathscr{R}(X_1) = &[\tanh(2)-2, \tanh(4)+2],\nonumber\\
\mathscr{R}(X_2) = &[\tanh(2\tanh(2)-5)-2, \tanh(2\tanh(4)+3)+2].
\nonumber
\end{align}
The conditional density of $X_1$ given $X_0$ is given by
\begin{equation}
p_{1|0}(x_1|x_0) = 
\begin{cases}
\displaystyle \frac{1}{4}, & \text{if $\left|x_1-\tanh(x_0+3)\right|\leq 2$},\\
0, & \text{otherwise}.
\end{cases}
\label{CPDF}
\end{equation}
This function is plotted in Figure \ref{fig:condPDF} together 
with the domain $\mathscr{R}(X_1)\times\mathscr{R}(X_0)$ 
(interior of the rectangle delimited by dashed red lines). 
\begin{figure}
\centerline{\includegraphics[height=6cm]{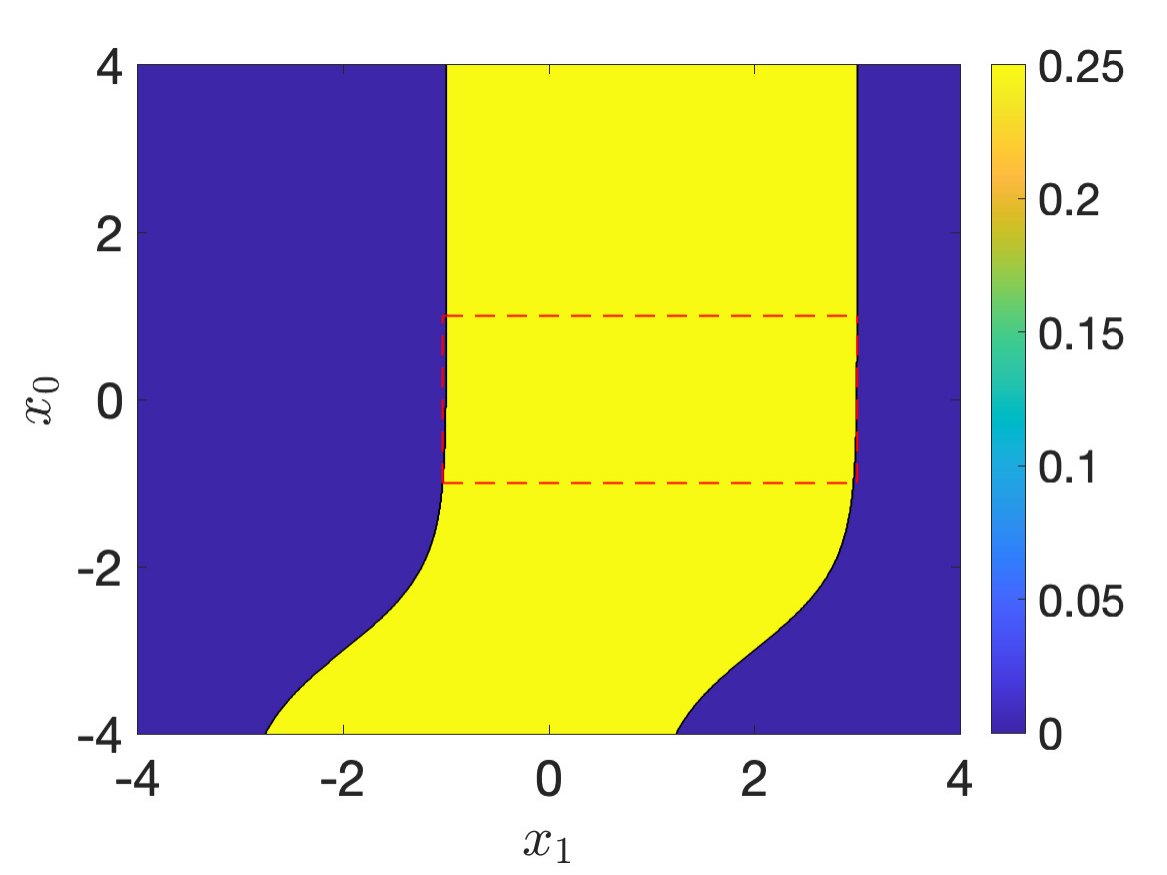}}
\caption{Conditional probability density function 
$p_{1|0}(x_1|x_0)$ defined in equation \eqref{CPDF}. 
The domain  $\mathscr{R}(X_1)\times\mathscr{R}(X_0)$ 
is the interior of the rectangle delimited by dashed 
red lines.}
\label{fig:condPDF}
\end{figure}
Clearly, the integral of the conditional PDF  \eqref{CPDF} is 
\begin{equation}
\int_{\mathscr{R}(X_0)}\int_{\mathscr{R}(X_1)}p_{1|0}(x_1|x_0)dx_1dx_0=\lambda(\mathscr{R}(X_0))=2,
\end{equation}
where $\lambda(\mathscr{R}(X_0))$ is the Lebesgue measure
of $\mathscr{R}(X_0)=[-1,1]$. The $L^2$ norm of the operators 
$\N$ and $\M$ is bounded by\footnote{For uniformly
distributed random variables we have that
\begin{equation}
\int_{\mathscr{R}(\bm \xi_n)} \rho_n(\bm x)^2 d \bm x = 
\frac{1}{\lambda(\mathscr{R}(\bm \xi_n))}. 
\end{equation}
Therefore equation \eqref{Kn0} yields 
\begin{equation}
K_n = \frac{\lambda\left(\mathscr{R}(\bm X_n)\right)}
{\lambda\left(\mathscr{R}(\bm \xi_n)\right)}\leq \frac{\lambda\left(\Omega_n\right)}
{\lambda\left(\mathscr{R}(\bm \xi_n)\right)}.
\label{Kn1}
\end{equation}
Depending on the ratio between the Lebesgue measure 
of $\mathscr{R}(\bm X_n)$ and $\mathscr{R}(\bm \xi_n)$ 
one can have $K_n$ smaller or larger than 1.
}

\begin{align}
K_0=&\int_{\mathscr{R}(X_0)}\int_{\mathscr{R}(X_1)}
p_{1|0}(x_1|x_0)^2dx_1dx_0= \frac{\lambda(\mathscr{R}(X_0))}
{\lambda(\mathscr{R}(\xi_0))}=\frac{1}{2}.
\end{align}
Hence, both operators $\N(1,0)$ and $\M(0,1)$ are 
contractions (Proposition \ref{lemma:operatorbounds}). 
On the other hand, 
\begin{align}
K_1=\frac{\lambda(\mathscr{R}(X_1))}
{\lambda(\mathscr{R}(\xi_1))} = 1+\frac{\tan(4)-\tan(2)}{4}>1.
\end{align}
Next, define $V$ as in Lemma \ref{lemma4.4}, 
i.e., $V=[-3,3]$. Clearly, both $\mathscr{R}(X_0)$ 
and $\mathscr{R}(X_1)$ are subsets of $V$. 
If we integrate the conditional PDF 
shown in Figure \ref{fig:condPDF} in $V\times V$ we obtain
\begin{equation}
\int_{V}\int_{V}
p_{1|0}(x_1|x_0)^2dx_1dx_0=
\frac{\lambda(V)}{\lambda(\mathscr{R}(\xi_0))}=
\frac{3}{2}.
\end{equation}

\subsection*{\color{r}Operator contractions induced by random noise}
In this section, we prove a result on neural 
networks models \eqref{NNfunctional} which states that it is possible to make both operators $\N$ and $\M$ in \eqref{MN}  
contractions\footnote{An linear operator is called a contraction if its operator norm is smaller than one.} if the noise is properly chosen. 
To this end, we begin with the following lemma. 
\begin{lemma}
\label{lemma:contractionsNM}
{\color{r}Let $\mathscr{R}(\bm \xi_n)$ and 
$\mathscr{R}(\bm \xi_{n-1})$ be bounded subsets 
of $\mathbb{R}^N$, $\rho_n \in L^2(\mathscr{R}(\bm \xi_n))$}.
If
\begin{equation}
\left\|\rho_n\right\|^2_{L^2(\mathscr{R}(\bm \xi_n))} 
\leq \frac{\kappa}{\lambda(\Omega_n)}\qquad 0\leq \kappa<1
\label{GNcont}
\end{equation}
then $\M(n,n+1)$ and $\N(n+1,n)$ are operator contractions. The upper bound in 
\eqref{GNcont} is independent of the neural network weights and biases.  
\end{lemma}

\begin{proof}
The proof follows immediately from equation \eqref{Kn}.

\end{proof}

\noindent
Hereafter, we specialize Lemma \ref{lemma:contractionsNM}
to neural network perturbed by uniformly distributed
random noise.

\begin{proposition}
\label{prop:impossibility_to_train}
Let $\{\bm \xi_0,\ldots,\bm \xi_{L-1}\}$ be independent 
random vectors. Suppose that the components of each 
$\bm \xi_n$ are zero-mean i.i.d. uniform random 
variables with range $[-b_n,b_n]$ ($b_n> 0$). 
If 
\begin{equation}
b_0 \geq\frac{1}{2}\left(\frac{\lambda(\Omega_0)}{\kappa}\right)^{1/N} 
\qquad \text{and}\qquad b_{n} \geq \frac{b_{n-1}+1}{\kappa^{1/N}}
\qquad n=1,\ldots, L-1, 
\label{noise_ltr}
\end{equation}
where $\Omega_0$ is the domain of the neural network input, 
$0\leq \kappa<1$, and $N$ is the number of neurons in 
each layer, then both operators $\M(n,n+1)$ and $\N(n+1,n)$ 
defined in \eqref{MN} are contractions for all $n=0,\ldots, L-1$, 
i.e., their norm can be bounded by a constant 
$K_n \leq \kappa$, independently of the weights and biases 
of the neural network. 
\end{proposition}

\begin{proof}
If $\bm \xi_n$ is uniformly distributed 
then from \eqref{Kn0} it follows  that 
\begin{equation}
K_n = \frac{\lambda\left(\mathscr{R}(\bm X_n)\right)}
{\lambda\left(\mathscr{R}(\bm \xi_n)\right)}. 
\label{Kn3}
\end{equation}
By using Lemma \ref{lemma4.4} we can bound $K_n$ as
\begin{equation}
K_n \leq \left(\frac{1+b_{n-1}}{b_n}\right)^N,
\end{equation}
where $N$ is the number of neurons in each layer of the 
neural network. Therefore, if $b_n \geq (b_{n-1}+1)/\kappa^{1/N}$ 
($n=1,\ldots,L-1$) 
we have that $K_n$ is bounded by a quantity $\kappa$ 
smaller than one. Regarding $b_0$, we notice that 
\begin{equation}
K_0= \frac{\lambda\left(\mathscr{R}(\bm X_0)\right)}
{\lambda\left(\mathscr{R}(\bm \xi_0)\right)} =
\frac{\lambda(\Omega_0)}{(2b_0)^N},
\end{equation} 
where $\Omega_0$ is the domain of the neural 
network input. Hence, if $b_0$ 
satisfies \eqref{noise_ltr} then $K_0\leq \kappa$.

\end{proof}

\noindent
One consequence of Proposition \ref{prop:impossibility_to_train} 
is that the $L^2$ norm of the neural network output decays 
with both the number of layers and the number of neurons 
if the noise amplitude from one layer to the next increases as in
\eqref{noise_ltr}. For example, if we represent the input-output 
map as a sequence of conditional expectations 
(see \eqref{composition1}), and set $u(\bm x)=\bm \alpha\cdot \bm x$ 
(linear output) then we have  
\begin{equation}
q_L(\bm x) = \M(0,1)\M(1,2)\cdots \M(L-1,L) (\bm \alpha\cdot\bm x).
\label{QL}
\end{equation}
By iterating the inequalities \eqref{noise_ltr} 
in Proposition \ref{prop:impossibility_to_train} we find that
\begin{equation}
b_{n}\geq\frac{1}{2\kappa^{n/N}}\left(\frac{\lambda(\Omega_0)}{\kappa}\right)^{1/N}+\sum_{j=1}^n\left(\frac{1}{\kappa}\right)^{j/N}
\qquad n=0,\ldots, L-1,  
\label{b0bn}
\end{equation}
In Figure \ref{fig:Gb} we plot the lower bound at the right 
hand side of \eqref{b0bn} for $\kappa=0.2$ 
and $\kappa=10^{-4}$ as a function of the number of 
neurons ($N$). 
\begin{figure}[t]
\centerline{\hspace{0.2cm}$\kappa=0.2$\hspace{6.2cm}$\kappa=0.2$}
\centerline{
\includegraphics[height=5.6cm]{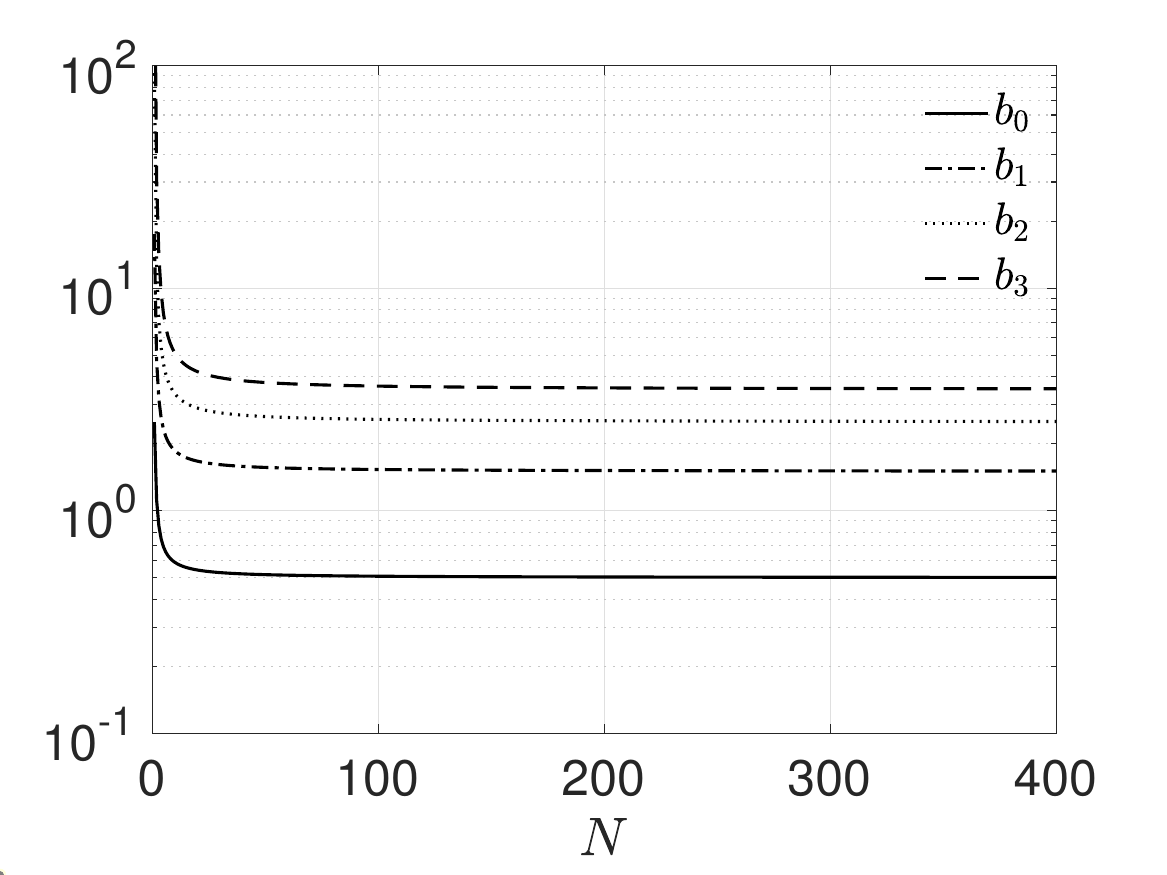}
\includegraphics[height=5.6cm]{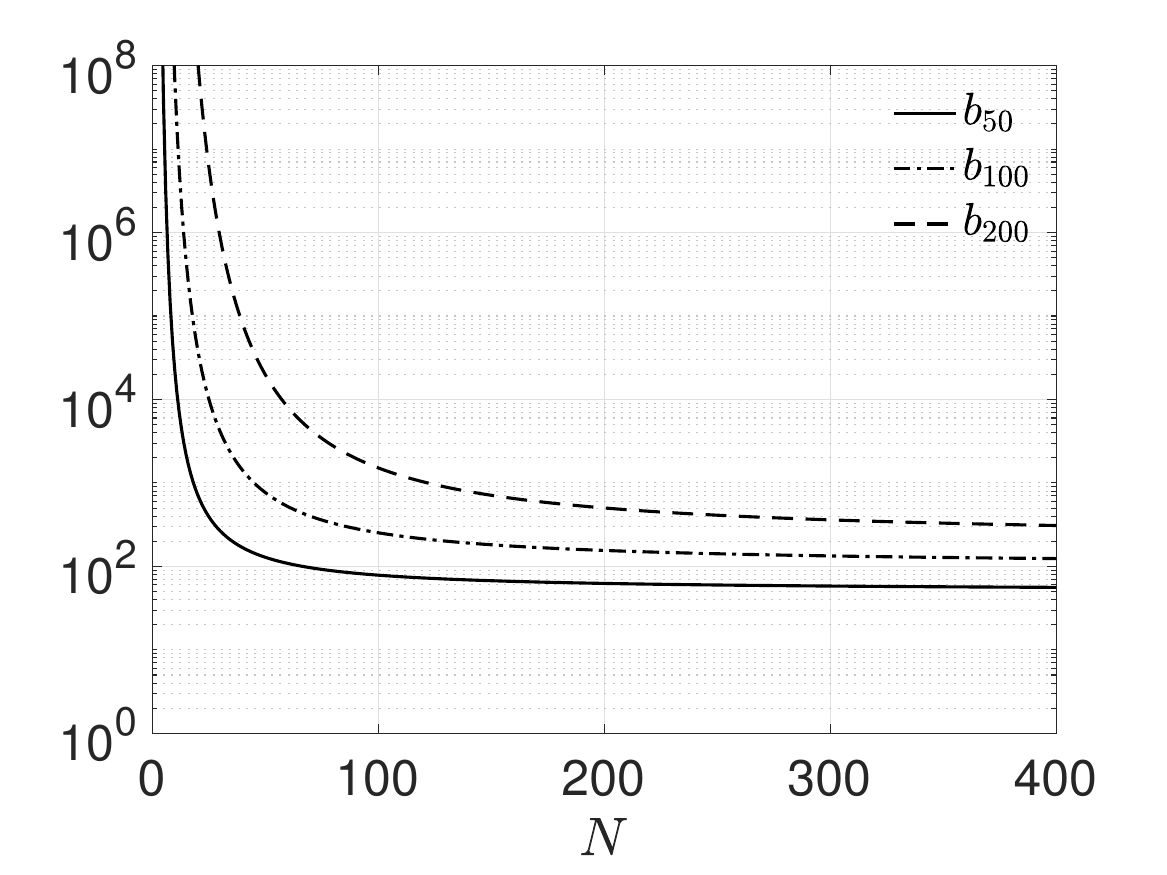}}
\centerline{\hspace{0.4cm}$\kappa=10^{-4}$\hspace{6.0cm}$
\kappa=10^{-4}$}
\centerline{
\includegraphics[height=5.6cm]{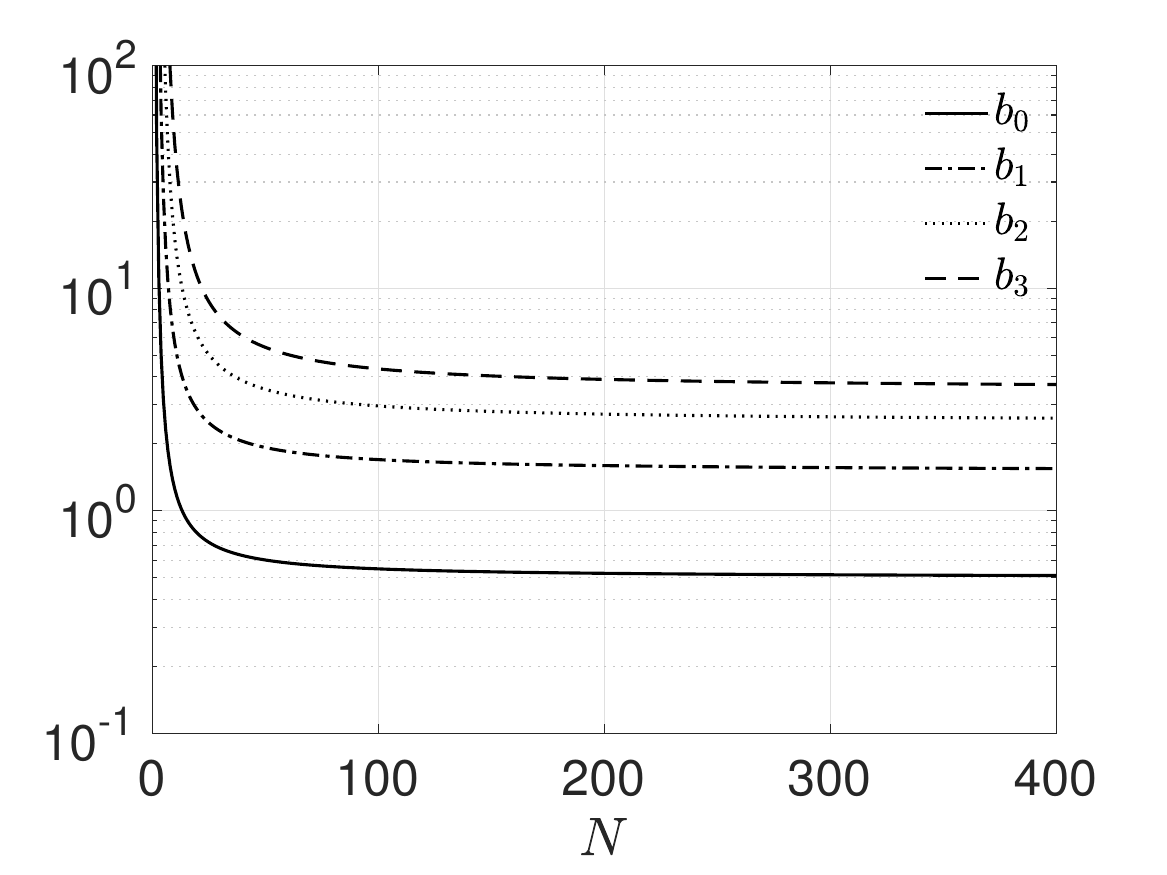}
\includegraphics[height=5.6cm]{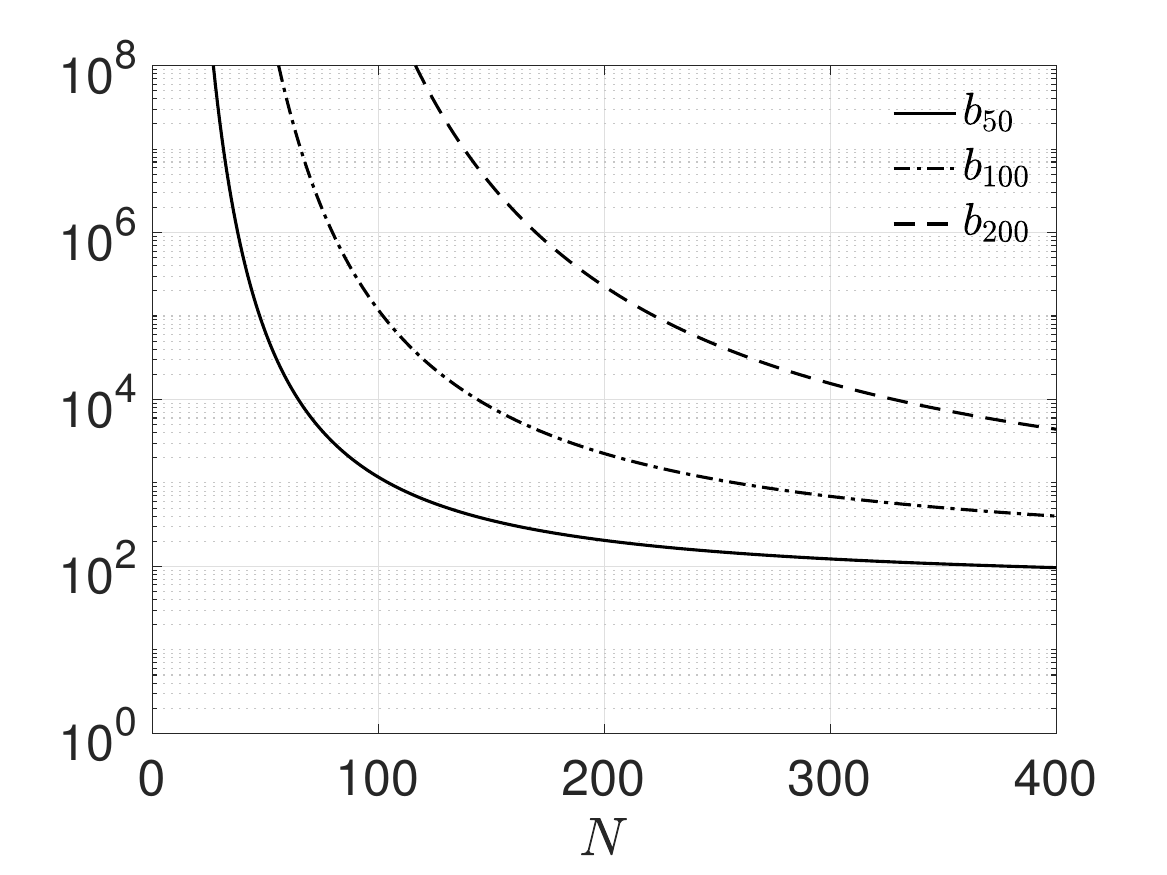}}
\caption{Lower bound on the coefficients $b_n$ 
defined in \eqref{b0bn} for 
$\lambda(\Omega_0)=1$ as a function of the number of 
neurons $N$ and the number of layers of the neural network. 
With such values of $b_n$ the operator
$\M(n,n+1)$ is a contraction satisfying 
$\left\|\M(n,n+1)\right\|^2\leq \kappa$.
Shown are results for $\kappa=0.2$ and $\kappa=10^{-4}$ 
(contraction index).}
\label{fig:Gb}
\end{figure}
With $b_n$ given in \eqref{b0bn} we have that 
the operator norms of $\M(n,n+1)$ and $\N(n+1,n)$ 
($n=0,\ldots,L-1$) are bounded exactly by 
$\kappa$ (see Lemma \ref{lemma:contractionsNM}). 
Hence, by taking the $L^2$ norm of \eqref{QL}, 
and recalling that  
$\left\|\M(n,n+1)\right\|^2\leq \kappa$ we obtain 
\begin{align}
\left\|q_L\right\|^2_{L^2(\Omega_0)}\leq& 
Z^2   \left\|\bm \alpha\right\|^2_2 \kappa^L,
\label{inequal}
\end{align}
where\footnote{In equation \eqref{inequal} we used 
the Cauchy-Schwarz inequality 
\begin{equation}
\left\| \bm \alpha\cdot\bm x\right\|^2_{L^2(\mathscr{R}(\bm X_L))}
\leq Z^2 \left\| \bm \alpha\right\|_2^2.
\end{equation}}
\begin{equation}
Z^2= \sum_{k=1}^N\int_{\mathscr{R}(\bm X_L)} x_k^2d\bm x
\qquad\text{and}
\qquad \left\|\bm \alpha\right\|^2_2=\sum_{k=1}^N \alpha_k^2.
\end{equation}
The inequality \eqref{inequal} shows that the 2-norm 
of the vector of weights $\bm \alpha$ must increase 
exponentially fast with the number of layers $L$ 
if we chose the noise amplitude as in \eqref{b0bn}. 
As shown in the following Lemma, the 
growth rate of $b_n$ that guarantees that 
both $\M$ and $\N$ are contractions is linear 
(asymptotically with the number of neurons).

\begin{lemma}
\label{lemma:asymptoticswithN}
Under the same assumptions in 
Proposition  \ref{prop:impossibility_to_train},
in the limit of an infinite number 
of neurons ($N\rightarrow\infty$), the noise 
amplitude \eqref{b0bn} satisfies 
\begin{equation}
\lim_{N\rightarrow \infty} b_n = \frac{1}{2}+n,
\end{equation}
independently of the contraction factor 
$\kappa$ and the domain $\Omega_0$. 
This means that for a finite number of neurons
the noise amplitude $b_n$ that 
guarantees that $\|\M(n,n+1)\|\leq\kappa$ 
is bounded from below ($\kappa<1$) 
or from above ($\kappa>1$) by a function 
that increases linearly with the number 
of layers.
\end{lemma}

\begin{proof}
The proof follows by taking the limit 
of \eqref{b0bn} for $N\rightarrow \infty$.
\end{proof}


{\color{r}
\section{Markovian neural networks}
\label{app:NN_characterization}

Consider the neural network model \eqref{discrete_dyn}, hereafter 
rewritten for convenience
\begin{equation}
\bm X_{n+1}=\bm H_n\left(\bm X_n,\bm w_n,\bm \xi_n\right)\quad n=0,\ldots,L-1.
\label{MP}
\end{equation}
In this Appendix we show that if the random vectors 
$\{\bm \xi_0,\ldots,\bm \xi_{L-1}\}$ are statistically 
independent, and if $\bm \xi_n$ is independent of 
past and current states, i.e., $\{\bm X_0,\ldots,\bm X_{n}\}$,
then \eqref{MP} defines a Markov process\footnote{Note that 
$\bm X_n$ depends on $\bm \xi_{n-1}$ via 
the recursion \eqref{MP}.}. To this end, we first 
notice that the full statistical information of 
the neural network is represented by the 
joint probability density function of 
$\{\bm X_L,\ldots, \bm X_0,\bm \xi_{L-1}, \ldots,  \bm \xi_0\}$.
Let us denote by 
$p\left(\bm x_L,\ldots, \bm x_0, \bm \xi_{L-1},\ldots,\bm \xi_0\right)$ such joint 
density function. By using well-known identities 
for conditional PDFs we can write 
\begin{align}
p\left(\bm x_L,\ldots, \bm x_0, \bm \xi_{L-1},\ldots,
\bm \xi_0\right) =&
p\left(\bm x_L|\bm x_{L-1},\ldots, \bm x_0, \bm \xi_{L-1},\ldots,\bm \xi_0\right)\times \nonumber\\
&p\left(\bm x_{L-1}|\bm x_{L-2},\ldots, \bm x_0, \bm \xi_{L-1},\ldots,\bm \xi_0\right)\times\cdots\nonumber\\
&p\left(\bm x_{0}|, \bm \xi_{L-1},\ldots,\bm \xi_0\right)p\left(\bm \xi_{L-1},\ldots,\bm \xi_0\right).
\label{prob_id}
\end{align}  
Clearly, from equation \eqref{MP} it follows that
\begin{equation}
p\left(\bm x_{n+1}|\bm x_{n},\ldots, \bm x_0, \bm \xi_{L-1},\ldots,\bm \xi_0\right)= p\left(\bm x_{n+1}|\bm x_{n}, \bm \xi_{n}\right). 
\end{equation}
This allows us to rewrite \eqref{prob_id} as 
\begin{align}
p\left(\bm x_L,\ldots, \bm x_0, \bm \xi_{L-1},\ldots,
\bm \xi_0\right) =&
p\left(\bm x_L|\bm x_{L-1},\bm \xi_{L-1}\right)p\left(\bm x_{L-1}|\bm x_{L-2},\bm \xi_{L-2}\right)\cdots p\left(\bm x_{1}|\bm x_{0},\bm \xi_{0}\right)p\left(\bm \xi_{L-1},\ldots,\bm \xi_0\right).
\label{prob_id1}
\end{align}  
If $\{\bm \xi_{L_1},\ldots,\bm \xi_0\}$ are 
statistically independent, and if $\bm \xi_k$ is independent 
of $\{\bm X_0,\ldots, \bm X_k\}$ then 
\begin{align}
p\left(\bm \xi_{L-1},\ldots,\bm \xi_0\right)= &p\left(\bm \xi_{L-1}\right)\cdots p\left(\bm \xi_0\right)\nonumber\\
= &p\left(\bm \xi_{L-1}|\bm x_{L-1}\right)\cdots p\left(\bm \xi_0|\bm x_{0}\right).
\label{noisepk}
\end{align}
Substituting \eqref{noisepk} into \eqref{prob_id1} and 
integrating over $\{\bm \xi_{L-1},\ldots,\bm \xi_0\}$ yields 
\begin{align}
p\left(\bm x_L,\ldots, \bm x_0\right) =&\underbrace{\left(\int
p\left(\bm x_L|\bm x_{L-1},\bm \xi_{L-1}\right) 
p\left(\bm \xi_{L-1}|\bm x_{L-1}\right)d\bm \xi_{L-1}\right)}_{p\left(\bm x_L|\bm x_{L-1}\right)}\times \cdots\nonumber\\
&\underbrace{\left(\int
p\left(\bm x_1|\bm x_{0},\bm \xi_{0}\right) 
p\left(\bm \xi_{0}|\bm x_{0}\right)d\bm \xi_{0}\right)}_{p\left(\bm x_1|\bm x_{0}\right)} p(\bm x_0)\nonumber\\
=&p\left(\bm x_L|\bm x_{L-1}\right)p\left(\bm x_{L-1}|\bm x_{L-2}\right)\cdots p\left(\bm x_1|\bm x_0\right)p(\bm x_0),
\label{prob_id2}
\end{align} 
which clearly represents the joint PDF of a Markov process. 
Note that the Markovian property of the process $\{\bm X_n\}$
representing the neural network states relies heavily on the fact that the joint PDF of the random vectors $\{\bm \xi_{L-1},\ldots,\bm \xi_0\}$  can be factorized as a product of conditional densities as in \eqref{noisepk}, i.e., that 
the random vectors are statistically independent, and also that $\bm \xi_n$ 
is independent of past and current states, i.e., $\{\bm X_0,\ldots,\bm X_{n}\}$.}

\bibliographystyle{plain}
\bibliography{nnet}

\end{document}